\documentclass{article}

\usepackage[preprint,nonatbib]{neurips_2026}

\usepackage[utf8]{inputenc} 
\usepackage[T1]{fontenc}    
\usepackage{hyperref}       
\usepackage{url}            
\usepackage{booktabs}       
\usepackage{amsfonts}       
\usepackage{nicefrac}       
\usepackage{microtype}      
\usepackage{multirow}
\usepackage[table,xcdraw]{xcolor}
\usepackage{graphicx}
\usepackage[normalem]{ulem} 
\useunder{\uline}{\ul}{}    
\usepackage{subcaption}
\usepackage{booktabs} 
\usepackage{bm}
\usepackage{subcaption}
\usepackage{enumitem}
\usepackage{algorithm}
\usepackage{algorithmic}
\usepackage{amsmath}
\usepackage{amssymb}
\usepackage{mathtools}
\usepackage{amsthm}

\usepackage{wrapfig}
\usepackage[capitalize,noabbrev]{cleveref}

\theoremstyle{plain}
\newtheorem{theorem}{Theorem}[section]

\newtheorem{lemma}[theorem]{Lemma}

\theoremstyle{definition}

\newtheorem{assumption}[theorem]{Assumption}
\theoremstyle{remark}
\newtheorem{remark}[theorem]{Remark}
\usepackage{todonotes}
\usepackage{etoc}

\title{FGRPO: Federated GRPO with Adaptive Aggregation on Non-IID Data}

\author{
  \textbf{Pengyu Chen}\textsuperscript{1}, 
  \textbf{Shaowei Li}\textsuperscript{1}, 
  \textbf{Kai Wang}\textsuperscript{2}, 
  \textbf{Yunsheng Yuan}\textsuperscript{1} \\
  \textbf{Kai Han}\textsuperscript{3}, 
  \textbf{Jun Luo}\textsuperscript{4}, 
  \textbf{Feng Li}\textsuperscript{1} \\
  \vspace{0.2cm}
  \normalfont 
  \textsuperscript{1}School of Computer Science and Technology, Shandong University, Qingdao China \\
  \textsuperscript{2}School of Mathematical Science, Peking University, China \\
  \textsuperscript{3}School of Computer Science and Artificial Intelligence, \\Shanghai University of Finance and Economics, Shanghai, China \\
  \textsuperscript{4}College of Computing and Data Science, Nanyang Technological University, Singapore \\
  \vspace{0.1cm}
  Email: \texttt{\{202435194, 202420868\}@mail.sdu.edu.cn, wangkaisd@stu.pku.edu.cn,}\\
  \texttt{ysyuan1028@mail.sdu.edu.cn, hankai@mail.shufe.edu.cn, }\\
  \texttt{junluo@ntu.edu.sg, fli@sdu.edu.cn}
}

\begin{document}

\maketitle

\begin{abstract}
  Recent advances in language models have established reinforcement learning as the primary paradigm for eliciting self-correction and long-chain reasoning. While \emph{group relative policy optimization} (GRPO) offers superior scalability by eliminating the critic network, deploying it on a central infrastructure entails collecting a large volume of data from distributed owners, which poses significant privacy risks. To address these concerns, we introduce \emph{federated GRPO} (FGRPO), a framework designed to decentralize the fine-tuning of reasoning models across heterogeneous data owners. To effectively mitigate the instability caused by divergent reward scales across heterogeneous tasks, FGRPO incorporates an adaptive aggregation mechanism based on relative performance gain. By characterizing each client's improvement relative to its personalized historical baseline, the framework dynamically prioritizes effective learning trajectories regardless of local task difficulty. FGRPO ensures robust convergence on non-IID data while preserving data privacy.
\end{abstract}

\section{Introduction} \label{sec:intro}
  The paradigm of \textit{large language models} (LLMs) has recently been redefined by the emergence of advanced reasoning frameworks, such as OpenAI-o1~\cite{jaech2024openai} and DeepSeek-R1~\cite{GuoYZSWZXZMB-Nature25}, as well as instruction-tuned multimodal models (e.g., Qwen2.5-VL-3B-Instruct~\cite{bai2025qwen2}). These models exhibit strong reasoning capabilities across diverse tasks, including mathematical theorem proving, competitive programming, and multi-step scientific synthesis. A key driver behind these advances is the strategic use of \emph{reinforcement learning} (RL), which has emerged as a central paradigm for eliciting self-correction and robust long-chain reasoning in LLMs~\cite{ZhangZHSLJFTJL-arXiv25}.
  Existing RL methods, such as proximal policy optimization (PPO)~\cite{SchulmanWDRK-arXiv17} and REINFORCE-style variants such as leave-one-out (RLOO)~\cite{AhmadianCGFKPUS-ACL24}, estimate advantages from absolute reward signals, typically using value-function critics or variance-reduction baselines. In contrast, \textit{group relative policy optimization} (GRPO)~\cite{ShaoWZXSBZZLW-arXiv24} adopts a different paradigm by leveraging intra-group comparisons to compute relative advantages, thereby avoiding the need for a separate critic network. Recent works~\cite{ZhengLLCYGDLMY-arXiv25,ChenCWY-arXiv25,ZhangZ-arXiv25,YuZZYZYDFLL-arXiv25} further improve the robustness and efficiency of GRPO.

  Meanwhile, data privacy remains a critical concern in training large-scale reasoning models, as high-quality reasoning traces often contain sensitive intellectual property. Examples include expert reasoning processes in code generation and mathematical proving, proprietary logic in commercial systems, and regulated data in healthcare and finance. Aggregating such data from multiple owners into a centralized third-party infrastructure creates significant risks of privacy leakage and legal non-compliance. To this end, \emph{federated learning} (FL) has emerged as an effective framework for distributed model training, in which a learning task is collaboratively solved by a federation of participating data owners (or clients)~\cite{McMahanMRHA-AISTATS17,LiSZSTS-MLSys20,KarimireddyKMRS-ICML20}. Each FL client computes a local update to the current global model maintained by a central server, and the FL server updates the global model by aggregating the received local updates. In this process, clients communicate only their local updates to the central server and do not share their local data with others for the sake of privacy preservation.

  Although a substantial body of work has explored \emph{federated reinforcement learning} (FedRL)~\cite{FanMDJTL-NIPS21,JinPYWZ-AISTATS22,KhodadadianSJM-ICML22,ZhangWMA-ICLR24,WooSJC-ICML24,FangWG-WWW25,JiangWZBTF-AAMAS25,WangHZMA-ICML24,XiongWJL-ICLR25,LanHHAB-ICLR25}, applying FL to LLM fine-tuning has only recently emerged in the form of federated \emph{reinforcement learning with human feedback} (RLHF)~\cite{FanTOWO-AAMAS25}, where human feedback is utilized to define client-specific rewards. In this setting, adopting traditional RL methods such as Actor–Critic~\cite{KondaT-NIPS99} and PPO~\cite{SchulmanWDRK-arXiv17} leads to federated variants relying on value (critic) networks and absolute reward signals for advantage estimation. In contrast, as mentioned above, GRPO eliminates the critic and derives advantages from relative comparisons among multiple sampled outputs; this design renders advantages inherently local. Particularly, under non-IID data, where each client’s dataset comprises tasks of varying difficulty, these advantages become incomparable across clients: modest gains on simpler tasks may yield gradients comparable in magnitude to substantial improvements on more complex tasks, leading to gradient interference during aggregation. While prior FedRL approaches address data heterogeneity through shared value estimation, e.g., by aligning gradients via global Q-function estimation~\cite{YangCWCC-NIPS24} or enforcing a shared critic through reversed actor–critic updates~\cite{XieS-arXiv25}, such mechanisms are not applicable for critic-free GRPO.  

  In this paper, we propose FGRPO, an FL framework that leverages GRPO to fine-tune reasoning-capable LLMs under non-IID data. To compensate for the absence of a global or shared critic, FGRPO introduces \emph{relative performance gain} (RPG) as a decentralized surrogate, shifting aggregation from absolute performance to relative learning progress. Specifically, each client measures its improvement relative to an \emph{exponential moving average} (EMA) baseline and normalizes it by local reward volatility, yielding a progress-aware signal that restores cross-client comparability while suppressing noisy or stagnant updates. The server then aggregates local models using an RPG-based weighting scheme that prioritizes clients demonstrating consistent and meaningful progress relative to their task difficulty. In this way, FGRPO amplifies informative learning trajectories and enables robust optimization under non-IID data. Our main contributions are summarized as follows:
  \begin{itemize}
    \item We propose {FGRPO}, a framework extending GRPO to FL settings, by incorporating an RPG-based adaptive aggregation mechanism to mitigate the instability caused by divergent reward scales under non-IID data.
    \item We provide a rigorous theoretical analysis establishing the non-convex convergence of FGRPO. Specifically, we quantify how FGRPO suppresses the convergence error floor, offering a theoretical guarantee for the resilience against data heterogeneity.
    \item We conduct extensive experiments on standard benchmarks with various models. The results show that FGRPO consistently outperforms state-of-the-art baselines, demonstrating superior scalability and robustness across varying client populations and non-IID settings.
  \end{itemize}

  The remainder of this paper is organized as follows. Sec.~\ref{sec:sys} introduces the system model and preliminaries. The detailed design of our FGRPO framework and the corresponding analysis are provided in Sec.~\ref{sec:fgrpo}. We report extensive experiment results in Sec.~\ref{sec:exp} and survey the related literature in Sec.~\ref{sec:relwork}. Finally, Sec.~\ref{sec:conclusion} concludes the paper.

\section{System Model and Preliminaries} \label{sec:sys}
  \subsection{Federated Learning} \label{ssec:fl}
    Consider a set of $N$ clients $\mathcal{N}=\{1,2,\cdots,N\}$ and a central server. Each client $i\in\mathcal{N}$ is associated with a fixed local dataset $\mathcal{D}_i$. Let ${D}_i$ denote the number of data samples in $\mathcal{D}_i$. The clients are coordinated by the central server to achieve a global learning objective:
    \begin{equation}
      \min_{\theta \in \mathbb{R}^d} ~ F(\theta) \triangleq  \sum^N_{i=1} \omega_i \mathbb{E}_{\xi \sim \mathcal{D}_i} \left[f_i(\theta; \xi)\right]
    \end{equation}
    where $\theta \in \mathbb{R}^d$ denotes the parameters of the global model, $f_i(\theta; \xi)$ denotes the local loss of client $i$ evaluated at model $\theta$ on a sample $\xi$ drawn from $\mathcal{D}_i$, and $\omega_i \geq 0$ is the weight such that $\sum^N_{i=1} \omega_i = 1$. 
    %

    At the beginning of each communication round $t$, the central server broadcasts the current global model parameters $\theta^{[t]}$ to the participating clients. Each client $i$ subsequently computes a local model $\theta^{[t]}_i$ by optimizing the objective on its private dataset. These local models are then transmitted back to the server. Typically following the conventions established in~\cite{McMahanMRHA-AISTATS17,LiSZSTS-MLSys20}, the server aggregates them to derive the updated global model $\theta^{[t+1]}$ via weighted averaging such that $\theta^{[t+1]} = \sum^N_{i=1} \omega_i \theta^{[t]}_i$.

  \subsection{GRPO} \label{ssec:grpo}
    GRPO~\cite{ShaoWZXSBZZLW-arXiv24,GuoYZSWZXZMB-Nature25} departs from traditional actor-critic reinforcement learning algorithms. We consider a language model acting as a policy $\pi_{\theta}$, which maps an input prompt $q$ to a sequence of tokens $o$. This policy is parameterized by $\theta$. In conventional frameworks such as PPO~\cite{SchulmanWDRK-arXiv17}, a separate value network is required to estimate a baseline for advantage computation. This dual-network architecture increases memory consumption and computational complexity during training. In contrast, GRPO computes advantages by sampling a group of $K$ outputs $\{o_1, o_2, \dots, o_K\}$ for each input prompt $q$. For each output sample $o_k$, reward $r_k$ is assigned by a reward function $R(q, o_k)$, which may consist of rule-based verifiers, neural reward models, or a combination thereof. The advantage $A_k$ for each output $o_k$ is then calculated based on its relative performance within the group:
    \begin{equation} \label{eq:grpo_advantage}
      A_k = \frac{r_k - \mathrm{mean}(r_1, \dots, r_K)}{\mathrm{std}(r_1, \dots, r_K)}
    \end{equation}
    By utilizing the intra-group baseline, GRPO eliminates the need for a dedicated critic network. For each token position $\ell = 1,2,\dots,|o_{k}|$ within a specific output $o_k$, the surrogate objective is
    \begin{equation}
      \Phi_{k,\ell}(\theta) = \min \left\{ \phi_{k,\ell}(\theta)A_k, \text{clip} ( \phi_{k,\ell}(\theta), 1-c, 1+c ) A_k \right\}
    \end{equation}
    where $\phi_{k,\ell}(\theta) = \frac{\pi_\theta (o_{k,\ell} \mid o_{k,<\ell}, q)}{\pi_{\theta_{\text{old}}} (o_{k,\ell} \mid o_{k,<\ell}, q)}$
    %
    %
    represents the probability ratio between the current and old policies, and $c$ is a clipping hyperparameter. $\pi_\theta$ is optimized by maximizing the following objective function:
    \begin{align} \label{eq:obj_grpo}
      \mathcal{J}_{\text{GRPO}}(\theta) = \mathbb{E}_{q \sim \mathcal{D}, \{o_k\}^K_{k=1}\sim\pi_{\theta_{\text{old}}}(\cdot\mid q) } \left[ \frac{1}{K} \sum^K_{k=1} \left( \frac{1}{|o_k|} \sum^{|o_k|}_{\ell=1} \Phi_{k,\ell}(\theta) - \beta D_{\text{KL}}(\pi_\theta || \pi_{\text{ref}}) \right) \right]
    \end{align}
    where $\pi_{\mathrm{ref}}$ denotes a reference policy, $D_{\mathrm{KL}}(\pi_\theta\|\pi_{\mathrm{ref}})$ is the Kullback–Leibler (KL) divergence between the current policy $\pi_\theta$ and a reference policy $\pi_{\mathrm{ref}}$, and $\beta$ is a hyper-parameter controlling how strongly the optimization penalizes divergence from $\pi_{\mathrm{ref}}$. 

     Designing a federated GRPO framework on non-IID data is particularly challenging because GRPO is critic-free and derives advantages through intra-group normalization, making the resulting learning signals inherently local and not directly comparable across clients. In heterogeneous settings, clients may face prompts with very different difficulty levels and reward distributions, so the same normalized improvement can correspond to very different absolute progress across clients. Without a global or shared critic to anchor these signals, standard federated averaging becomes unreliable and may aggregate misaligned updates, leading to gradient mismatch and unstable optimization. Consequently, the core challenge is to restore cross-client comparability while remaining fully critic-free.

\section{FGRPO} \label{sec:fgrpo}

  \subsection{Overview} \label{ssec:overview}
    \textbf{Algorithm}~\ref{alg:fgrpo} presents the iterative execution of our FGRPO framework, orchestrating collaboration between a central server and $N$ distributed clients over $T$ communication rounds. In each round $t$, the process begins with the server broadcasting the current global parameters $\theta^{[t]}$ to all participating clients. On the client side (Lines 4-18), each participant performs $E$ steps of local update using its private dataset $\mathcal{D}_i$. Crucially, this involves sampling a group of $K$ outputs per prompt to compute intra-group advantages $A_k(q)$ without a critic network, updating the local policy via the GRPO objective, and simultaneously tracking a local return ($\bar{R}_i^{[t]}$) to quantify reasoning progress. Once these local updates are uploaded, the server executes the adaptive global phase (Lines 21-23). It calculates aggregation weights $\omega_i^{[t]}$ based on each client's RPG value, which is a metric derived from the client's improvement over its own historical baseline. Finally, the server applies an Adam-style adaptive update to the weighted gradients, ensuring the global model $\theta^{[t+1]}$ converges robustly despite the divergent reward scales inherent in heterogeneous reasoning tasks.
    \begin{algorithm}[t!]
    \caption{FGRPO}
    \label{alg:fgrpo}
    \begin{algorithmic}[1]
    \item[] {\bfseries Input:} Initialized global model $\theta^{[0]}$ and the first and second moments $m^{[-1]}=0, ~v^{[-1]}=0$, the maximum number of rounds $T$, the maximum number of local updates per round $E$.
    \item[] {\bfseries Output:} Final global model $\theta^{[T]}$.
    %
    %
    %
    \FOR{$t=0$ {\bfseries to} $T-1$}
      \STATE \textbf{Server} broadcasts global parameters $\theta^{[t]}$ to clients $\mathcal{N}$;
      \FOR{\textbf{each client} $i\in\mathcal{N}$ \textbf{in parallel}}
        \STATE $\theta^{[t,0]}_i = \theta^{[t]}$
        \FOR{$e=0$ {\bfseries to} $E-1$}
        \STATE Sample mini-batch of $B$ prompts $\mathcal{Q} \subseteq \mathcal{D}_i$;
        \FOR{each prompt $q \in \mathcal{Q}$}
          \STATE Generate $K$ outputs $\{o_{k}\}^K_{k=1} \sim \pi_{\theta^{[t,e]}_{i}}(\cdot|q)$;
          \STATE Obtain rewards $R(q, o_k)$, $\forall k$;
          %
          %
          \STATE Compute advantages $A_k(q)$ using Eq.~\eqref{eq:grpo_advantage}
        \ENDFOR
        %
        %
        \STATE $g^{[t,e]}_{i} = \nabla_{\theta_i} \mathcal{J}_{\mathrm{GRPO}}^{(i)}(\theta^{[t,e]}_{i})$ using Eq.~\eqref{eq:obj_grpo}
        \STATE $\theta^{[t,e+1]}_i = \theta^{[t,e]}_i + \alpha_i g^{[t,e]}_{i}$
        \STATE $\bar{R}^{[t,e]}_i = \frac{1}{BK}\sum_{q\in\mathcal{Q}}\sum^K_{k=1}R(q,o_k)$
      \ENDFOR
      \STATE $\bar{R}^{[t]}_{i} = \frac{1}{W} \sum^{E-1}_{e=E-W}\bar{R}^{[t,e]}_i$
      \STATE $\theta^{[t]}_i = \theta^{[t,E]}_i$
      \STATE Upload $\theta^{[t]}_i$ and $\bar{R}^{[t]}_{i}$ to the server
    \ENDFOR
    %
    %
    \STATE \textbf{Server side:}
    \STATE Calculate weight $\omega^{[t]}_i$ using Eq.~\eqref{eq:weight}
    \STATE Calculate moments $m^{[t]}$ and $v^{[t]}$ using Eq.~\eqref{eq:model_diff}--\eqref{eq:momentum}
    \STATE Update global model $\theta^{[t+1]} = \theta^{[t]} + \alpha \frac{m^{[t]}}{\sqrt{v^{[t]}} + \epsilon}$
  \ENDFOR
  \end{algorithmic}
  \end{algorithm}

  \subsection{Algorithm Design} \label{ssec:algdesign}
    \subsubsection{Local Updating at Client Side}
      In each communication round $t$, the central server broadcasts the current global model parameters $\theta^{[t]}$ to the clients. The goal of each client $i$ is to compute a local update that maximizes the GRPO objective based on its local data $\mathcal{D}_i$. To ensure training stability and minimize communication frequency, clients perform $E$ local update steps per round. 
      
      As detailed in \textbf{Algorithm}~\ref{alg:fgrpo}, each client initializes its local model as $\theta^{[t,0]}_i = \theta^{[t]}$. For each step $e\in\{0,1,\cdots,E-1\}$, client $i$ samples a mini-batch of $B$ prompts $\mathcal{Q}$ from the local data $\mathcal{D}_i$. For each prompt $q \in \mathcal{Q}$, client $i$ generates a group of $K$ outputs $\{o_1, \dots, o_K\}$ using the current policy $\pi_{\theta^{[t,e]}_i}$. After assigning rewards $\{r_1, \dots, r_K\}$ via the reward function $R(q, o_k)$, the client computes local advantage $A_k(q)$ for each prompt input $q$ and corresponding output $o_k$ using Eq.~\eqref{eq:grpo_advantage}, based on which, client $i$ calculates a stochastic gradient $g^{[t,e]}_i$ using Eq.~\eqref{eq:obj_grpo}, and then updates its local model:
      \begin{equation}
        \theta^{[t,e+1]}_i = \theta^{[t,e]}_i + \alpha_i g^{[t,e]}_i.
      \end{equation}
      where $\alpha_i$ is the local learning rate. 
      %
      %
      Simultaneously, the client tracks the average step reward, computed as $\bar{R}^{[t,e]}_i = \frac{1}{BK} \sum_{q\in\mathcal{Q}} \sum^K_{k=1} R(q,o_k)$. Given that rewards in earlier local steps often exhibit significant variance as the model begins to deviate from the global parameters, client $i$ computes a local performance indicator, $\bar{R}^{[t]}_i$, by averaging the rewards across the last $W$ local steps:
      \begin{equation}
        \bar{R}^{[t]}_i = \frac{1}{W} \sum^{E-1}_{e=E-W} \bar{R}^{[t,e]}_i.
      \end{equation}
      $\bar{R}^{[t]}_i$ represents the refined reasoning performance of client $i$ within round $t$. Finally, client $i$ uploads its updated local model $\theta^{[t]}_i = \theta^{[t,E]}_i$ and the reward estimate $\bar{R}^{[t]}_i$ to the server.

    \subsubsection{Model Aggregation at Server Side}
    %

      A key challenge in aggregating local models is the large variance in reward distributions across the clients with heterogeneous data. Since each client normalizes its advantages using a local baseline, a ``high'' advantage on one client may correspond to a lower absolute reasoning quality than a ``low'' advantage on another. As a result, simple averaging methods such as FedAvg~\cite{McMahanMRHA-AISTATS17} can be ineffective, as they treat all updates equally and may amplify conflicting gradients under non-IID data, where heterogeneity naturally induces gradient misalignment and unstable convergence. To address this issue, we propose an adaptive RPG-based aggregation scheme that weights each client’s contribution according to its \emph{relative performance gain} (RPG), measured against its own historical baseline. By prioritizing clients with consistent and meaningful progress, the server emphasizes high-quality learning trajectories while down-weighting noisy or stagnant updates, leading to more stable and effective global optimization.

      \textbf{Baseline Tracking}: The server maintains an \emph{exponential moving average} (EMA) baseline $\varphi^{[t]}_{i}$ for each client $i$:
      \begin{equation} \label{eq:ema}
        \varphi^{[t]}_{i} = \lambda_{\rm base} \bar{R}^{[t]}_i + (1-\lambda_{\rm base}) \varphi^{[t-1]}_{i},
      \end{equation}
      where $\lambda_{\rm base}\in(0,1]$ is a smoothing coefficient. This baseline serves as a ``personalized anchor'', representing the reward performance of client $i$ based on its recent history. The RPG is defined as the normalized improvement over this baseline:
      %
      %
      \begin{equation} \label{eq:rpg}
        h^{[t]}_i = \left(\bar{R}^{[t]}_i - \varphi^{[t-1]}_{i}\right) \big/ \left(\sigma^{[t]}_i + \varepsilon\right),
      \end{equation}
      where $\sigma^{[t]}_i = \mathrm{clip}\big(\iota \sigma^{[t-1]}_i + (1-\iota)\big|\bar{R}^{[t]}_i-\varphi^{[t-1]}_{i}\big|, \sigma_{\rm min}, \sigma_{\rm max} \big)$ is a clipped estimate of the reward fluctuation with $0<\iota\leq1$. Here, $\sigma_{\rm min}$ and $\sigma_{\rm max}$ are hyperparameters representing the lower and upper bounds of reward volatility. 
      %
      %
      A higher RPG indicates a more significant and consistent improvement in the client's reasoning capability.
      %
      %
      
      \textbf{Adaptive Weighting}: The server calculates the aggregation weight $\omega^{[t]}_i$ using a Boltzmann distribution with a dynamic temperature $\tau^{[t]}$:
      \begin{equation} \label{eq:weight}
        \omega^{[t]}_i = {\exp(h^{[t]}_i / \tau^{[t]})} \big/ {\sum^N_{i'=1} \exp(h^{[t]}_{i'} / \tau^{[t]})}.
      \end{equation}
      We employ an exponential annealing schedule for $\tau^{[t]}$:
      \begin{equation}
        \tau^{[t]} = \tau_{\rm min} + (\tau_{\rm max}-\tau_{\rm min})\exp(-\lambda_{\rm anneal} \cdot t).
      \end{equation}
      In early stages, a higher temperature smooths disparities to preserve optimization diversity; as training progresses, the temperature decreases to amplify the influence of the most ``reliable'' clients.

      \textbf{Global Update}: The server calculates a global pseudo-gradient according to the RPG-based weights:
      \begin{equation} \label{eq:model_diff}
        \Delta^{[t]}=\sum^N_{i=1}\omega^{[t]}_i(\theta^{[t]}_i- \theta^{[t]}),
      \end{equation}
      Subsequently, it updates the first and second moments
      \begin{equation} \label{eq:momentum}
        m^{[t]} = \beta_1 m^{[t-1]} + (1-\beta_1)\Delta^{[t]}, 
        ~\text{and}~ v^{[t]} = \beta_2 v^{[t-1]} + (1-\beta_2)(\Delta^{[t]})^2,
      \end{equation}
      where $\beta_1>0$ and $\beta_2>0$ are decay rates, and finally performs an adaptive global update:
      \begin{equation} \label{eq:model_update}
        \theta^{[t+1]} = \theta^{[t]} + \alpha \frac{m^{[t]}}{\sqrt{v^{[t]}} + \epsilon},
      \end{equation}
      where $\alpha$ is the server-side learning rate, and $\epsilon > 0$ is a small constant for numerical stability. Note that the squaring $(\cdot)^2$ and square root $\sqrt{~\cdot~}$ operations are applied element-wise.


  \subsection{Convergence Analysis} \label{ssec:convergence}
  %


    We first establish the necessary assumptions that have been widely adopted in the convergence analysis of federated optimization~\cite{McMahanMRHA-AISTATS17,LiSZSTS-MLSys20,ZhangWMA-ICLR24,FanMDJTL-NIPS21,ReddiCZGRKKM-ICLR21}.

    \begin{assumption}[$L$-Lipschitz Smoothness]
    \label{ass:main_lsmooth}
      Each local objective function $F_i$ is $L$-smooth for any $i \in \{1, \ldots, N\}$. Specifically, there exists constant $L \geq 0$ such that
      \begin{equation} \label{eq:lsmooth}
        \|\nabla F_i(\theta) - \nabla F_i(\theta')\| \le L \|\theta - \theta'\|, ~\forall \theta, \theta' \in \mathbb{R}^d.
      \end{equation}
      Consequently, the global objective function $F$ is also $L$-smooth.
    \end{assumption}
    \begin{assumption}[Unbiased Gradient and Bounded Variance]
    \label{ass:main_various}
      The stochastic gradient $g_i^{[t,e]}$ is an unbiased estimator of the true gradient $\nabla F_i(\theta_i^{[t,e]})$. We define the gradient noise as $\xi_i^{[t,e]} = g_i^{[t,e]} - \nabla F_i(\theta_i^{[t,e]})$. There exists a constant $\sigma \geq 0$ such that the noise satisfies the \emph{martingale difference sequence} (MDS) property:
      \begin{equation} \label{eq:various}
       \mathbb{E}\left[\xi_i^{[t,e]} \bigg| \mathcal{F}_{t,e}\right] = 0, \quad \mathbb{E}\left[\left\|\xi_i^{[t,e]}\right\|^2 \bigg| \mathcal{F}_{t,e}\right] \le \sigma^2,
      \end{equation}
      where the filtration $\mathcal{F}_{t,e}$ represents the information available up to local step $e$ of round $t$, including the current model state and all randomness revealed in previous local updates.
      %
      %
      Furthermore, the noise from different clients is uncorrelated given the same filtration.      
    \end{assumption}
    \begin{assumption}[Bounded Second Moment]
    \label{ass:main_secmoment}
      The second moment of the stochastic gradients is uniformly bounded. There exists a constant $G \geq 0$ such that:
      \begin{equation} \label{eq:secmoment}
        \mathbb{E}\left[\|g_i^{[t,e]}\|^2\right] \le G^2, \quad \forall i, t, e.
      \end{equation}
    \end{assumption}
    \begin{assumption}[Bounded Data Heterogeneity]
    \label{ass:main_datahet}
      There exists a constant $\kappa \geq 0$ such that for any $\theta \in \mathbb{R}^d$, the gradient heterogeneity across clients is bounded:
      \begin{equation} \label{eq:datahet}
        \frac{1}{N} \sum_{i=1}^N \|\nabla F_i(\theta) - \nabla F(\theta)\|^2 \le \kappa^2.
      \end{equation}
    \end{assumption}
    \begin{assumption}[Bounded RPG Score]
    \label{ass:main_boundrpg}
      The Relative Performance Gain (RPG) scores $h_i^{[t]}$ are uniformly bounded. There exists a constant $h_{\rm max} \geq 0$ such that:
      \begin{equation} \label{eq:boundrpg}
        \left| h_i^{[t]} \right| \le h_{\rm max}, \quad \forall i, t.
      \end{equation}
      %
    \end{assumption}
    \begin{assumption}[Bounded Adaptive Preconditioner]
    \label{ass:main_precondition}
      The eigenvalues of the $d\times d$ Adam-style preconditioner matrix $H^{[t]} = \mathrm{diag}((\sqrt{v^{[t]}} + \epsilon)^{-1})$ are bounded. Specifically, there exist positive constants $c_{\rm min}$ and $c_{\rm max}$ such that:
      \begin{equation}
        c_{\rm min} I \preceq H^{[t]} \preceq c_{\rm max} I, \quad \forall t.
      \end{equation}
      %
    \end{assumption}

    Based on the above assumptions, we establish the following convergence guarantee for FGRPO.
    \begin{theorem}[Non-Convex Convergence of FGRPO] \label{thm:main_convergence}
      Under \textbf{Assumptions}~\ref{ass:main_lsmooth}--\ref{ass:main_precondition}, when learning rates $\alpha = 1/T^{1/4}$ and $\alpha_i = 1/(ELT^{1/4}),~\forall i$, we have
      %
      %
      %
      \begin{align} 
        &\frac{1}{T} \sum_{t=0}^{T-1} \mathbb{E}\left[\left\|\nabla F(\theta^{[t]})\right\|^2\right] \nonumber\\
        \le& \underbrace{\frac{2L(F(\theta^{[0]}) - F^*)}{c_{\min}\sqrt{T}}}_{\text{Optimization term}~ \mathcal{O}(1/\sqrt{T})} + \underbrace{\mathcal{O}\!\left(\frac{1}{\sqrt{T}}\right)}_{\text{High-order decay term} } + \underbrace{\frac{4 c_{\max}^2}{c_{\min}^2}\exp\!\left(\frac{2 h_{\max}}{\tau_{\min}}\right)\left(\kappa^2 + \frac{\sigma^2}{E}\right)}_{\text{Irreducible error floor}~\mathcal{O}(1) }
      \end{align}
      where $F^* = \min_\theta F(\theta)$ denotes the global optimum.
      %
      %
    \end{theorem}


    The above theorem establishes a non-convex convergence guarantee for FGRPO, showing that the expected gradient norm decays at a rate of $O(1/\sqrt{T})$, matching the optimal rate for stochastic optimization~\cite{ShaoWZXSBZZLW-arXiv24,PangJ-arXiv25}. Beyond this asymptotic behavior, the bound reveals a non-vanishing error floor, consistent with traditional FedRL methods even with global or shared critic~\cite{YangCWCC-NIPS24,XieS-arXiv25}, which is governed by data heterogeneity ($\kappa^2$) and stochastic gradient noise ($\sigma^2/E$). Moreover, the bound explicitly characterizes the effect of RPG-based aggregation through $h_{\max}/\tau_{\min}$. In particular, the error floor is scaled by $\exp(2h_{\max}/\tau_{\min})$, indicating that larger values concentrate aggregation on a subset of clients and amplify heterogeneity, while smaller values promote more balanced updates and improved stability. 
    %
    %
    In a nutshell, $h_{\max}/\tau_{\min}$ governs how this irreducible heterogeneity is expressed, revealing a trade-off between selective aggregation and stable convergence. The detailed proof is provided in Appendix~\ref{sec:app_analysis}.

\section{Experiments} \label{sec:exp}
  \subsection{Implementation} \label{ssec:implement}
    To improve computational and communication efficiency, we incorporate Low-Rank Adaptation (LoRA)~\cite{HuSWALWWC-ICLR22} into FGRPO. LoRA is a parameter-efficient fine-tuning method that freezes the pre-trained model weights and injects trainable low-rank matrices, thereby substantially reducing the number of trainable parameters and GPU memory requirements. Specifically, for each client $i \in \mathcal{N}$ in \textbf{Algorithm}~\ref{alg:fgrpo}, we keep the backbone weights fixed and optimize only local adapter parameters $\theta_i=\{\mathbf{A}_i,\mathbf{B}_i\}$ throughout all communication rounds. Accordingly, the local gradient $g_i^{[t,e]}$ is calculated only with respect to these low-rank adapters. This design replaces full-model transmission with the communication of compact adapter matrices, so the server aggregates low-rank updates directly. As a result, FGRPO substantially reduces computation and communication cost while ensuring the effectiveness of policy optimization. More implementation details are given in Appendix~\ref{sec:app_lora}.

  \subsection{Experiment setup} \label{ssec:expset}
    We select \textbf{Qwen2.5-VL-3B-Instruct}, \textbf{Qwen2.5-VL-7B-Instruct}~\cite{bai2025qwen2}, \textbf{Qwen3-VL-4B-Instruct}~\cite{bai2025qwen3} and \textbf{Llama-3.2-11B-Vision-Instruct}~\cite{grattafiori2024llama} to evaluate FGRPO across different scales, ranging from resource-efficient edge-level models to high-performance reasoning models. 
    We evaluate FGRPO on the \textbf{OpenR1}~\cite{openr1_verified} and \textbf{GEOQA}~\cite{chen2025r1v} benchmarks. OpenR1 is a large-scale reasoning dataset with verified reasoning traces, including mathematical reasoning and other multi-step tasks, while GEOQA is a multimodal geometric question answering benchmark for numerical reasoning over both textual descriptions and visual diagrams. For OpenR1, we simulate data heterogeneity by partitioning samples into three difficulty tiers (\textit{simple}, \textit{medium}, and \textit{hard}) according to reasoning trace length, and then allocating them to clients using a Dirichlet distribution $\mathsf{Dir}(\mu)$, where smaller $\mu$ induces more severe non-IID partitions. For GEOQA, we instead model domain heterogeneity by partitioning samples according to geometric primitives (\textit{points}, \textit{lines}, \textit{circles}, and \textit{polygons}) under the same Dirichlet-based strategy, such that different clients specialize in distinct geometric concepts. We set $\mu=0.05$ for both datasets to create highly heterogeneous client distributions. Each dataset is split into 80\% for training and 20\% for testing.

    We compare FGRPO with three representative FL baselines. \textbf{FedAvg} performs standard data-volume-weighted model aggregation~\cite{McMahanMRHA-AISTATS17}. \textbf{FedProx} extends this framework by introducing a proximal regularization term in the local objective to mitigate client drift under non-IID data~\cite{LiSZSTS-MLSys20}. \textbf{SCAFFOLD} further addresses the data heterogeneity by using control variates to reduce the discrepancy between local and global update directions~\cite{KarimireddyKMRS-ICML20}. Although these methods were not originally designed for GRPO, their mechanisms for handling non-IID data can be naturally adapted to our FGRPO framework. To compute the reward for reinforcement learning, we evaluate the reasoning process by comparing the predicted answer against the ground-truth reference. A binary reward of 1 is assigned if the prediction matches the reference, and 0 otherwise. Consequently, we define the accuracy of a fine-tuned model as the aggregate proportion of these correct predictions across the test set. We implement these different algorithms on a cluster equipped with ten NVIDIA RTX Pro 6000 GPUs. 
    %
    %
    More details about the experiment settings can be found in the appendix (see Appendix~\ref{sec:app_expsetting}).
    %

  \subsection{Experiment Results} \label{ssec:expresults}
    Due to space limitations, we report results only for a five-client FL system with $\mu=0.05$, $\lambda_{\rm base} = 0.8$, $\tau_{\rm min} = 1.5$, $\tau_{\rm max} = 2.5$, $\sigma_{\rm min} = 0.05$, $\sigma_{\rm max} = 0.2$, and $\lambda_{\rm anneal} = 0.1$. 
    Additional experimental results, including the effects of varying the number of clients and data heterogeneity levels, extensions to different GRPO variants with RPG ablations, hyperparameter analysis, and resource consumption, are provided in Appendix~\ref{sec:appexp}.

    Table~\ref{tab:main_results} shows the model accuracy of the different algorithms. On Open-R1, FGRPO obtains the highest total accuracy across all four backbone models. For Qwen2.5-3B, FGRPO improves the total accuracy from the strongest baseline of 38.84\% to 41.86\%, yielding a gain of 3.02\%. For Qwen2.5-7B, FGRPO achieves 47.68\%, outperforming the best baseline FedProx by 1.84\%. The advantage remains consistent on larger and cross-family backbones. On Qwen3-4B, FGRPO reaches 43.16\% total accuracy, surpassing the strongest baseline FedProx by 1.46\%. This improvement is particularly pronounced on the challenging \emph{hard} split, where FGRPO achieves 18.80\%, substantially higher than FedProx at 14.37\%. This indicates that the proposed RPG-based aggregation is especially beneficial for difficult reasoning examples, where clients may exhibit heterogeneous local optimization progress and reward scales. On Llama-3.2-11B, FGRPO also achieves the best total accuracy of 41.88\%, outperforming the strongest baseline SCAFFOLD by 0.72\%. More importantly, FGRPO again shows clear advantages on the \emph{hard} split, obtaining 35.33\% compared with 33.29\% from the best baseline. These results suggest that FGRPO is not only effective for Qwen-based models, but can also generalize to a distinct vision-language model family. The advantage of FGRPO is even more evident on the more challenging \emph{hard} split of Open-R1. Compared with the strongest baseline, FGRPO improves the Hard accuracy by 5.27\%, 4.02\%, 4.43\%, and 2.04\% on Qwen2.5-3B, Qwen2.5-7B, Qwen3-4B, and Llama-3.2-11B, respectively.

    On GEOQA, FGRPO also consistently achieves the best total performance across all four models. For Qwen2.5-3B, FGRPO reaches 40.76\% total accuracy, surpassing the best baseline FedProx by 1.48\%. On Qwen2.5-7B, FGRPO further improves the total score to 51.36\%, outperforming FedProx by 2.41\%. The same trend holds on Qwen3-4B, where FGRPO obtains the highest total accuracy of 53.93\%, improving over FedProx by 1.00\%. Although FedAvg and SCAFFOLD perform competitively on some individual geometric categories, FGRPO achieves the best results on \emph{lines} and \emph{circles}, indicating stronger robustness on diverse geometric reasoning tasks. On Llama-3.2-11B, FGRPO achieves 28.42\% total accuracy, surpassing FedAvg by 1.60\%. Notably, FGRPO obtains the best performance across all GEOQA subcategories, including \emph{points}, \emph{lines}, \emph{circles}, and \emph{polygons}. This further confirms that the benefits of FGRPO are not limited to specific models or datasets, but remain effective across different backbone architectures and reasoning task structures.
    \begin{table*}[t]
    \centering
    \caption{Test accuracy (\%) of different methods on the Open-R1 and GEOQA benchmarks. Results are reported as mean$\pm$std. \textbf{Bold} values indicate the best performance among decentralized methods, while \underline{underlined} values denote the second-best results.}
    \label{tab:main_results}
    \small
    \setlength{\tabcolsep}{2.5pt}
    \renewcommand{\arraystretch}{1.12}
    \resizebox{\textwidth}{!}{
    \begin{tabular}{llccccccccc}
    \toprule
    \multicolumn{2}{c}{\textbf{Setup}} 
    & \multicolumn{4}{c}{\textbf{Open-R1}} 
    & \multicolumn{5}{c}{\textbf{GEOQA}} \\
    \cmidrule(lr){1-2} \cmidrule(lr){3-6} \cmidrule(lr){7-11}
    \textbf{Model} & \textbf{Method} 
    & \textbf{Simple} & \textbf{Medium} & \textbf{Hard} & \textbf{Total}
    & \textbf{Points} & \textbf{Lines} & \textbf{Circles} & \textbf{Polygons} & \textbf{Total} \\
    \midrule
    \multirow{4}{*}{Qwen2.5-3B}
    & FedAvg   
    & \underline{47.63$\pm$1.15} & 40.24$\pm$1.45 & \underline{28.08$\pm$3.75} & 38.64$\pm$1.09 
    & 41.43$\pm$4.94 & 36.89$\pm$2.69 & 40.93$\pm$2.74 & 31.78$\pm$0.93 & 35.59$\pm$1.31 \\
    & FedProx  
    & 41.80$\pm$1.33 & 42.04$\pm$0.67 & 24.55$\pm$1.60 & 36.12$\pm$0.86 
    & \underline{47.62$\pm$6.52} & 40.99$\pm$3.17 & \underline{44.89$\pm$1.15} & \underline{35.09$\pm$0.70} & \underline{39.28$\pm$1.14} \\
    & SCAFFOLD 
    & 46.37$\pm$2.43 & \underline{43.00$\pm$4.09} & 27.19$\pm$0.98 & \underline{38.84$\pm$1.68} 
    & \textbf{51.90$\pm$3.53} & \underline{41.37$\pm$1.68} & 40.79$\pm$1.06 & 34.98$\pm$0.66 & 38.08$\pm$0.85 \\
    & FGRPO   
    & \textbf{48.11$\pm$0.96} & \textbf{44.14$\pm$3.87} & \textbf{33.35$\pm$2.19} & \textbf{41.86$\pm$1.21} 
    & \underline{47.62$\pm$4.12} & \textbf{45.34$\pm$2.64} & \textbf{46.65$\pm$1.38} & \textbf{35.88$\pm$1.40} & \textbf{40.76$\pm$0.74} \\
    \midrule
    \multirow{4}{*}{Qwen2.5-7B}
    & FedAvg   
    & \textbf{60.90$\pm$0.98} & 43.72$\pm$0.89 & \underline{32.75$\pm$2.38} & 45.78$\pm$0.70 
    & \underline{58.57$\pm$6.43} & 50.06$\pm$2.35 & \underline{55.33$\pm$2.06} & 44.20$\pm$1.20 & 48.87$\pm$1.32 \\
    & FedProx  
    & \underline{58.56$\pm$1.96} & 46.25$\pm$1.20 & \underline{32.75$\pm$1.83} & \underline{45.84$\pm$1.19} 
    & 56.67$\pm$2.61 & 51.68$\pm$3.21 & 52.47$\pm$1.58 & \textbf{45.84$\pm$1.68} & \underline{48.95$\pm$1.59} \\
    & SCAFFOLD 
    & 55.98$\pm$2.20 & \underline{48.17$\pm$1.48} & 32.69$\pm$3.16 & 45.60$\pm$1.10 
    & 57.62$\pm$7.22 & \underline{52.17$\pm$3.26} & 51.94$\pm$1.48 & 44.99$\pm$1.89 & 48.42$\pm$1.75 \\
    & FGRPO   
    & 55.44$\pm$1.68 & \textbf{50.87$\pm$1.43} & \textbf{36.77$\pm$2.20} & \textbf{47.68$\pm$0.79} 
    & \textbf{60.00$\pm$4.58} & \textbf{55.40$\pm$1.55} & \textbf{59.16$\pm$1.29} & \underline{45.34$\pm$2.17} & \textbf{51.36$\pm$0.90} \\
    \midrule
    \multirow{4}{*}{Qwen3-4B}
    & FedAvg   
    & 62.82$\pm$3.49 & \textbf{49.37$\pm$2.38} & 8.38$\pm$3.16 & 40.16$\pm$0.69 
    & 53.33$\pm$7.06 & 45.34$\pm$2.37 & 53.79$\pm$1.57 & \textbf{51.15$\pm$2.12} & 51.40$\pm$1.86 \\
    & FedProx  
    & \textbf{65.83$\pm$1.61} & 44.98$\pm$1.92 & \underline{14.37$\pm$3.47} & \underline{41.70$\pm$0.55} 
    & \underline{56.19$\pm$4.33} & 52.17$\pm$1.46 & \underline{57.31$\pm$0.78} & 50.28$\pm$1.64 & \underline{52.93$\pm$0.83} \\
    & SCAFFOLD 
    & 63.06$\pm$1.64 & 45.05$\pm$0.74 & 9.21$\pm$1.88 & 39.08$\pm$0.67 
    & \textbf{60.48$\pm$2.71} & \underline{52.42$\pm$2.90} & 54.67$\pm$1.85 & 49.75$\pm$1.36 & 51.95$\pm$1.27 \\
    & FGRPO   
    & \underline{65.47$\pm$1.18} & \underline{45.29$\pm$2.51} & \textbf{18.80$\pm$3.61} & \textbf{43.16$\pm$1.50} 
    & 53.81$\pm$3.98 & \textbf{55.40$\pm$0.81} & \textbf{58.59$\pm$1.94} & \underline{50.83$\pm$1.28} & \textbf{53.93$\pm$1.37} \\
    \midrule
    \multirow{4}{*}{Llama-3.2-11B}
    & FedAvg   
    & \textbf{49.25$\pm$2.03} & 40.84$\pm$1.52 & 31.80$\pm$1.43 & 40.62$\pm$0.46 
    & 33.81$\pm$5.43 & 30.06$\pm$5.28 & 27.49$\pm$2.17 & \underline{25.34$\pm$1.17} & \underline{26.82$\pm$1.33} \\
    & FedProx  
    & 48.83$\pm$1.25 & 40.24$\pm$2.12 & 32.04$\pm$2.03 & 40.36$\pm$0.72 
    & 32.86$\pm$5.43 & \underline{30.68$\pm$3.12} & 27.49$\pm$0.98 & 24.49$\pm$0.75 & 26.41$\pm$0.66 \\
    & SCAFFOLD 
    & 48.17$\pm$1.17 & \textbf{42.04$\pm$2.16} & \underline{33.29$\pm$2.23} & \underline{41.16$\pm$1.64} 
    & \underline{35.24$\pm$4.26} & 29.81$\pm$2.85 & \underline{29.38$\pm$0.96} & 24.04$\pm$0.84 & 26.75$\pm$0.80 \\
    & FGRPO   
    & \underline{49.07$\pm$1.87} & \underline{41.26$\pm$1.51} & \textbf{35.33$\pm$3.10} & \textbf{41.88$\pm$0.89} 
    & \textbf{36.19$\pm$1.06} & \textbf{31.18$\pm$3.02} & \textbf{30.13$\pm$1.38} & \textbf{26.37$\pm$1.64} & \textbf{28.42$\pm$1.00} \\
    \bottomrule
    \end{tabular}
    }
\end{table*}

    The training trajectories in Fig.~\ref{fig:acc_convergence} further validate the effectiveness of FGRPO during federated optimization. For both Open-R1 and GEOQA, FGRPO generally reaches higher final accuracy and maintains a stronger upward trend over communication rounds compared with the other baselines. In particular, the advantage becomes more pronounced in the later training rounds, indicating that FGRPO can better accumulate useful policy improvements rather than being dominated by clients with larger absolute reward scales. The reward curves in Fig.~\ref{fig:reward_convergence_comparison_with_baseline} show a consistent pattern: FGRPO achieves more favorable reward trajectories across local update steps, which aligns with its superior downstream accuracy.
    \begin{figure*}[t!]
    \centering
      \parbox{.245\textwidth}{\center\includegraphics[width=.245\textwidth]{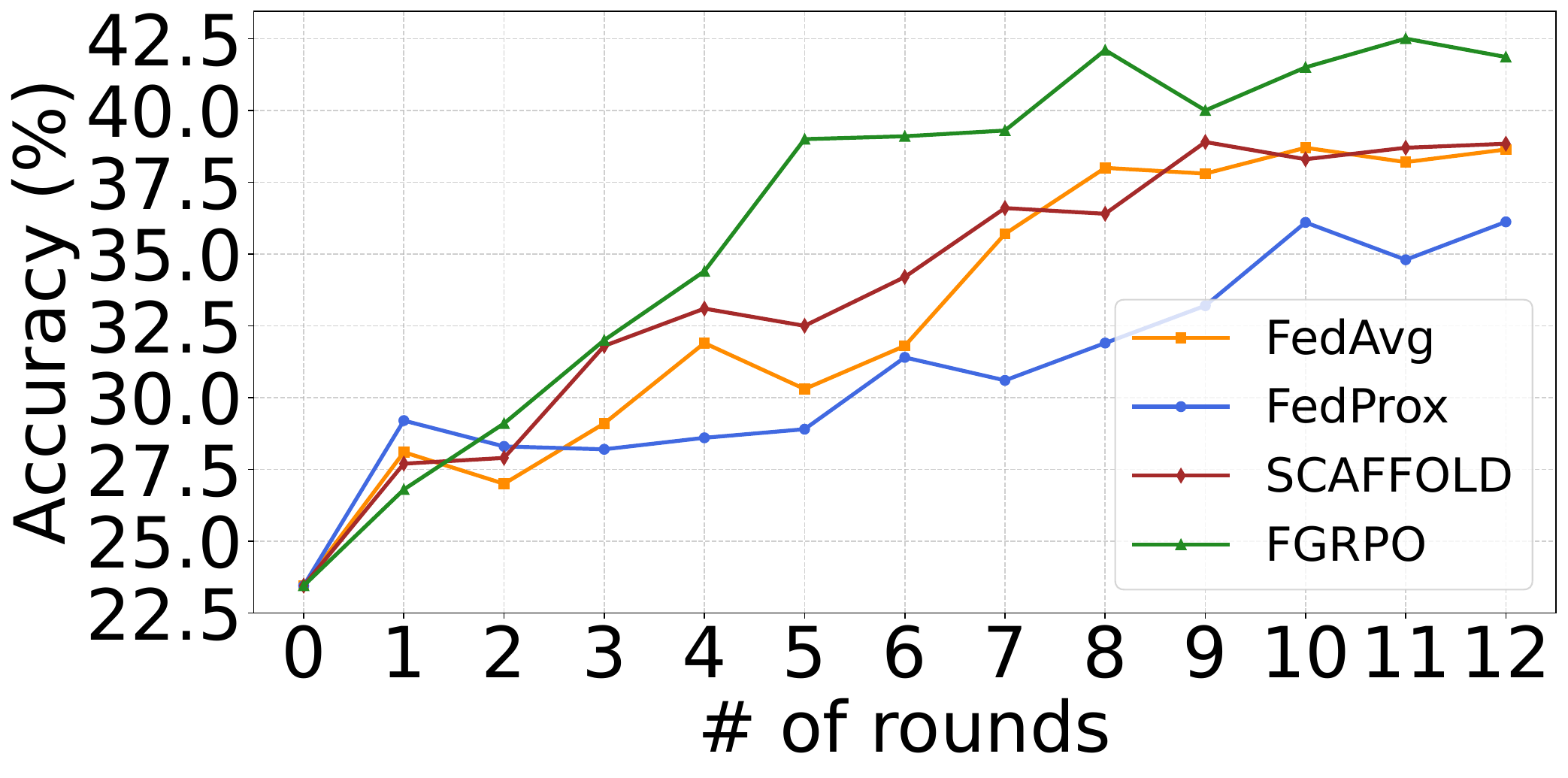}}
      \parbox{.245\textwidth}{\center\includegraphics[width=.245\textwidth]{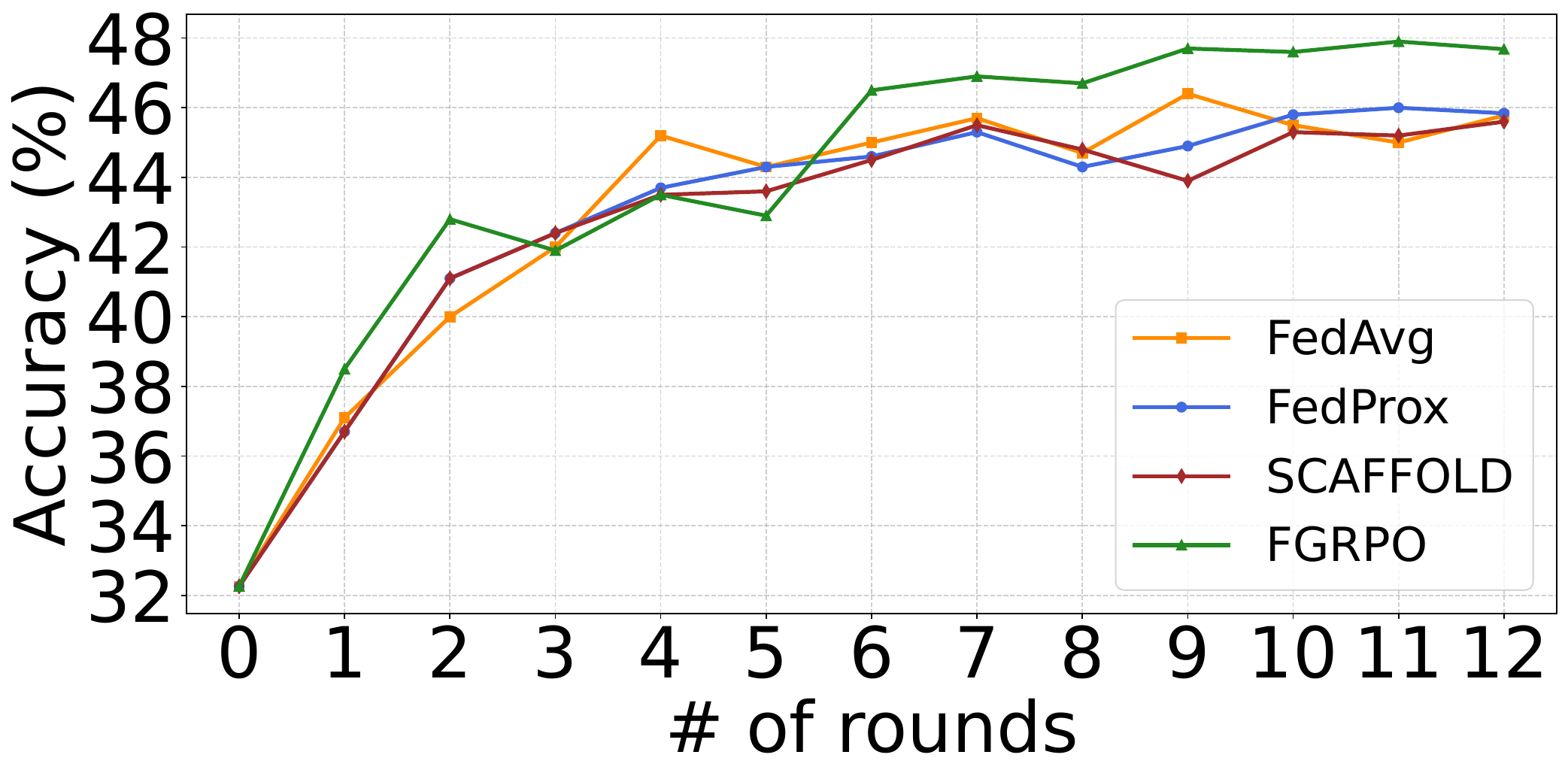}}
      \parbox{.245\textwidth}{\center\includegraphics[width=.245\textwidth]{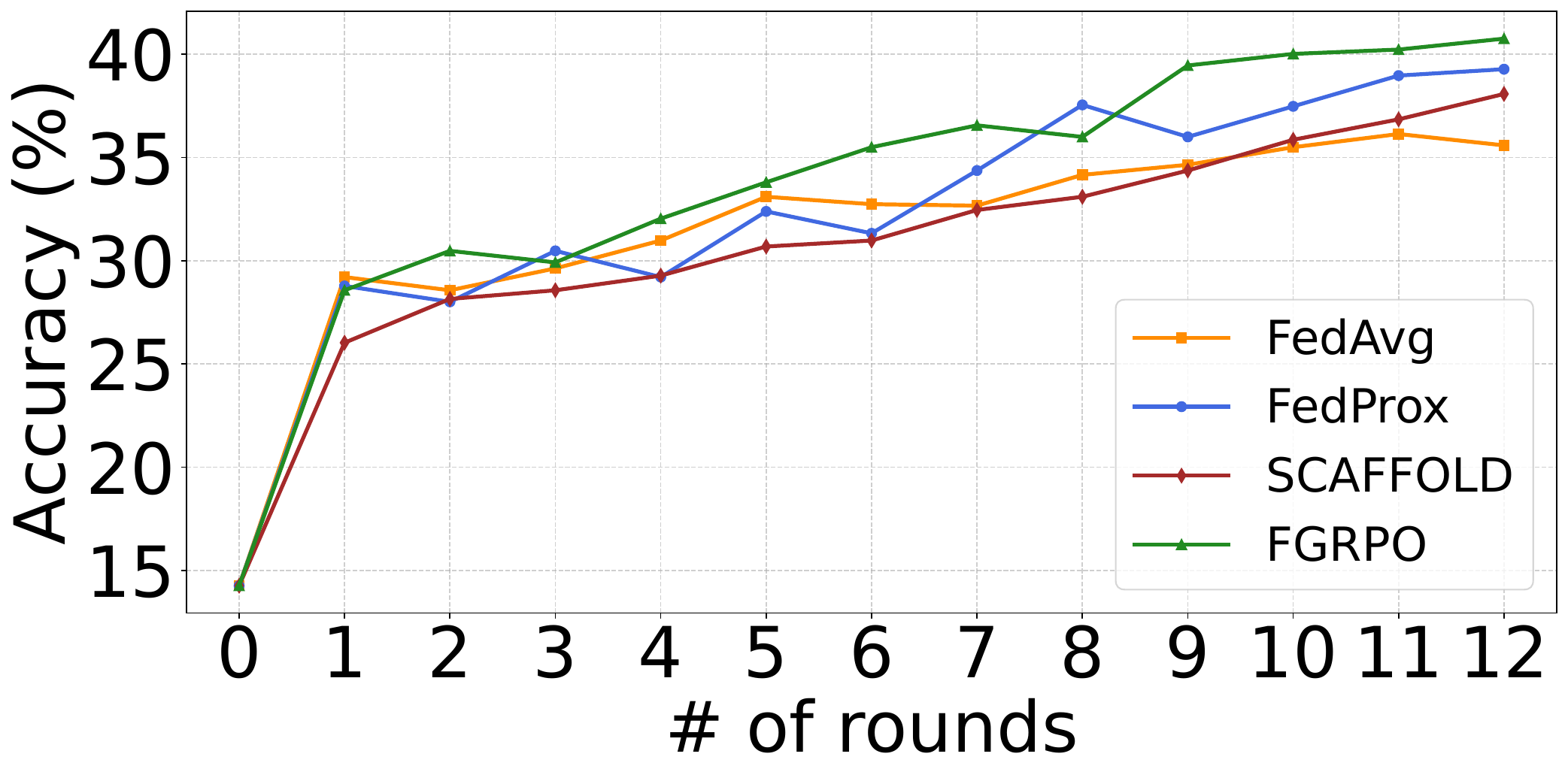}}
      \parbox{.245\textwidth}{\center\includegraphics[width=.245\textwidth]{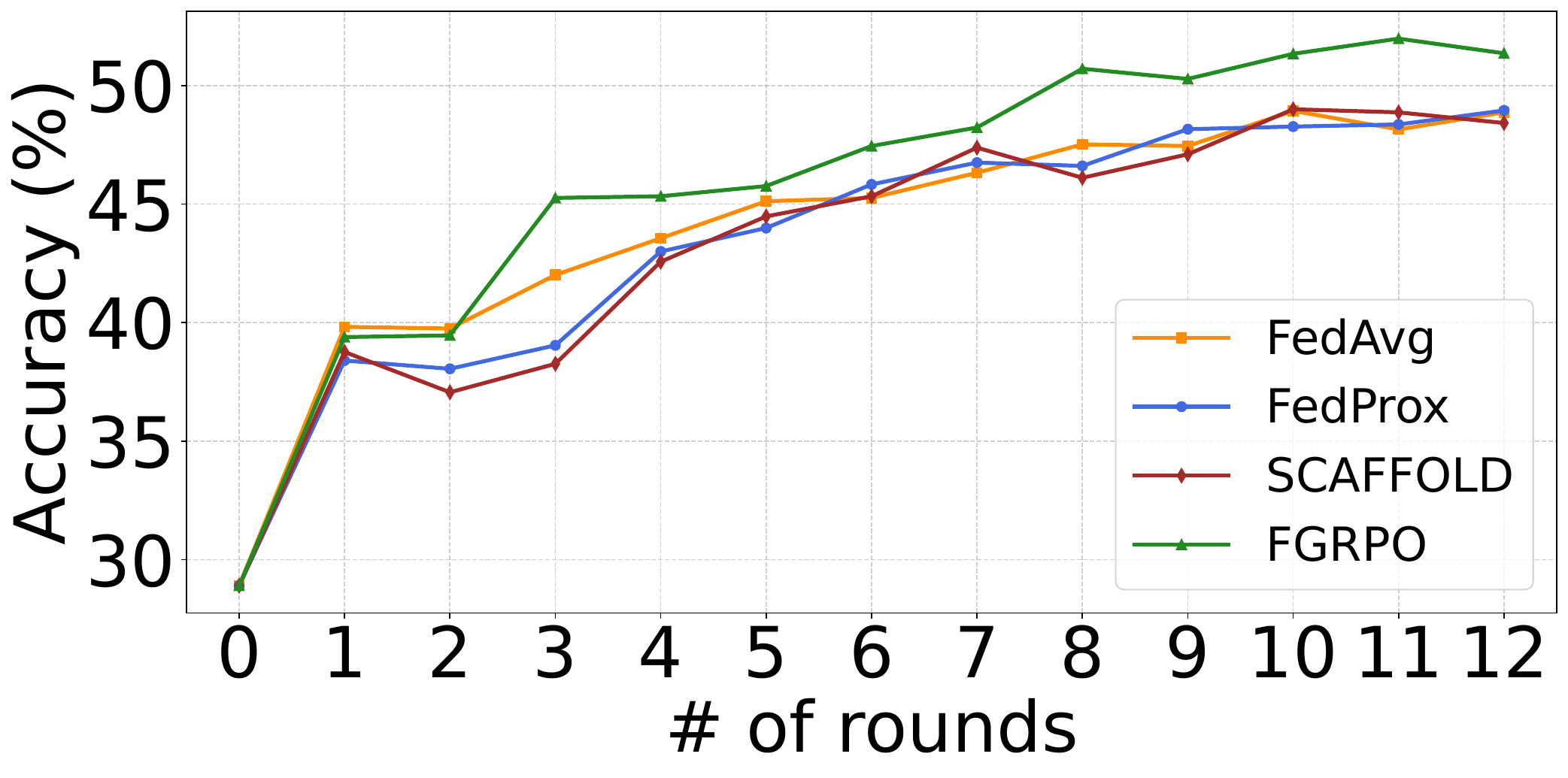}}
      \parbox{.245\textwidth}{\center\scriptsize(a) 3B model on OpenR1}
      \parbox{.245\textwidth}{\center\scriptsize(b) 7B model on OpenR1}
      \parbox{.245\textwidth}{\center\scriptsize(c) 3B model on GEOQA}
      \parbox{.245\textwidth}{\center\scriptsize(d) 7B model on GEOQA}
    \caption{Convergence performance of different methods in terms of model accuracy with Qwen2.5-3B and Qwen2.5-7B models.}    
    \label{fig:acc_convergence}
    \end{figure*}
    \begin{figure*}[t!]
    \centering
      \parbox{.245\textwidth}{\center\includegraphics[width=.245\textwidth]{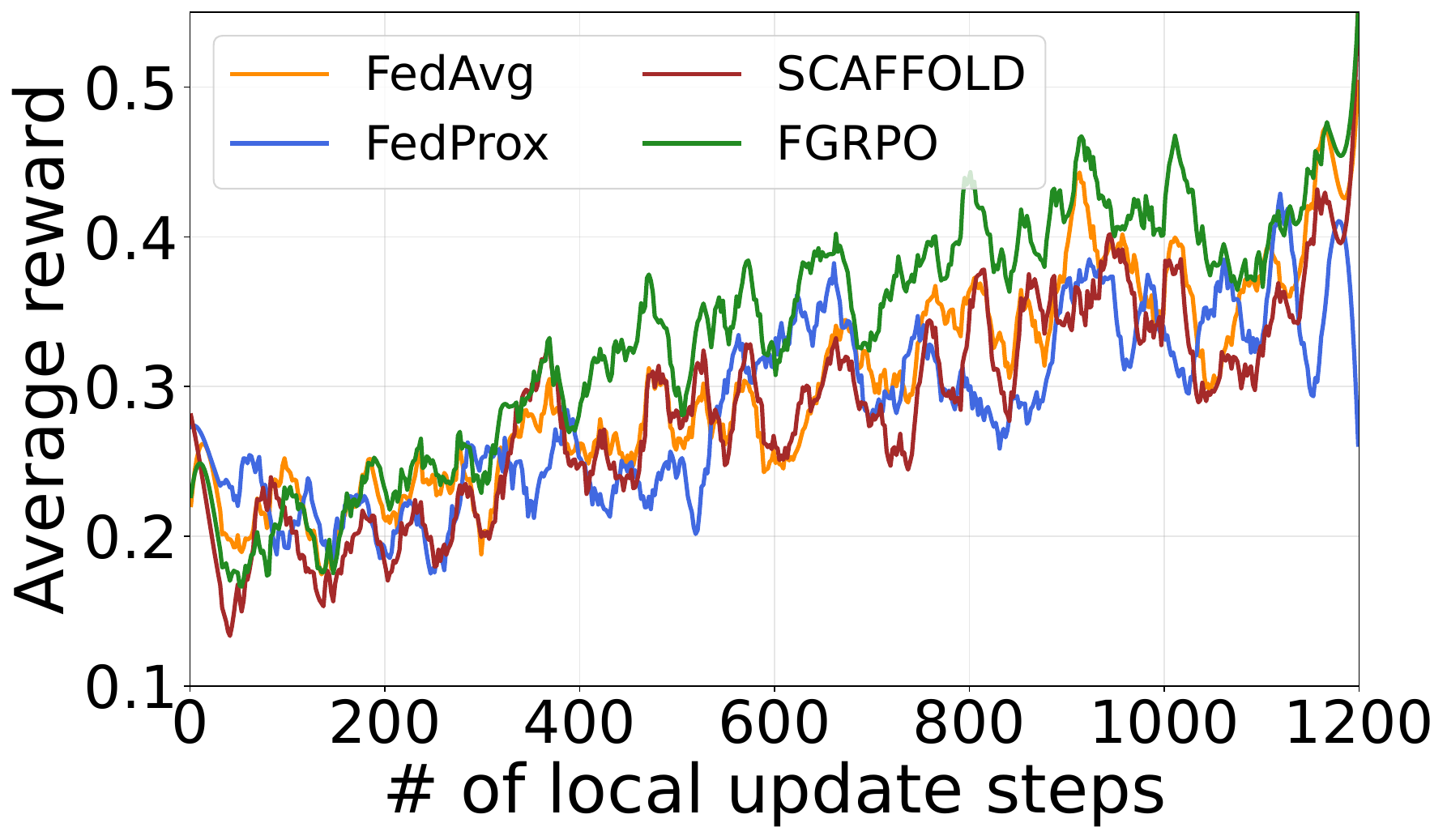}}
      \parbox{.245\textwidth}{\center\includegraphics[width=.245\textwidth]{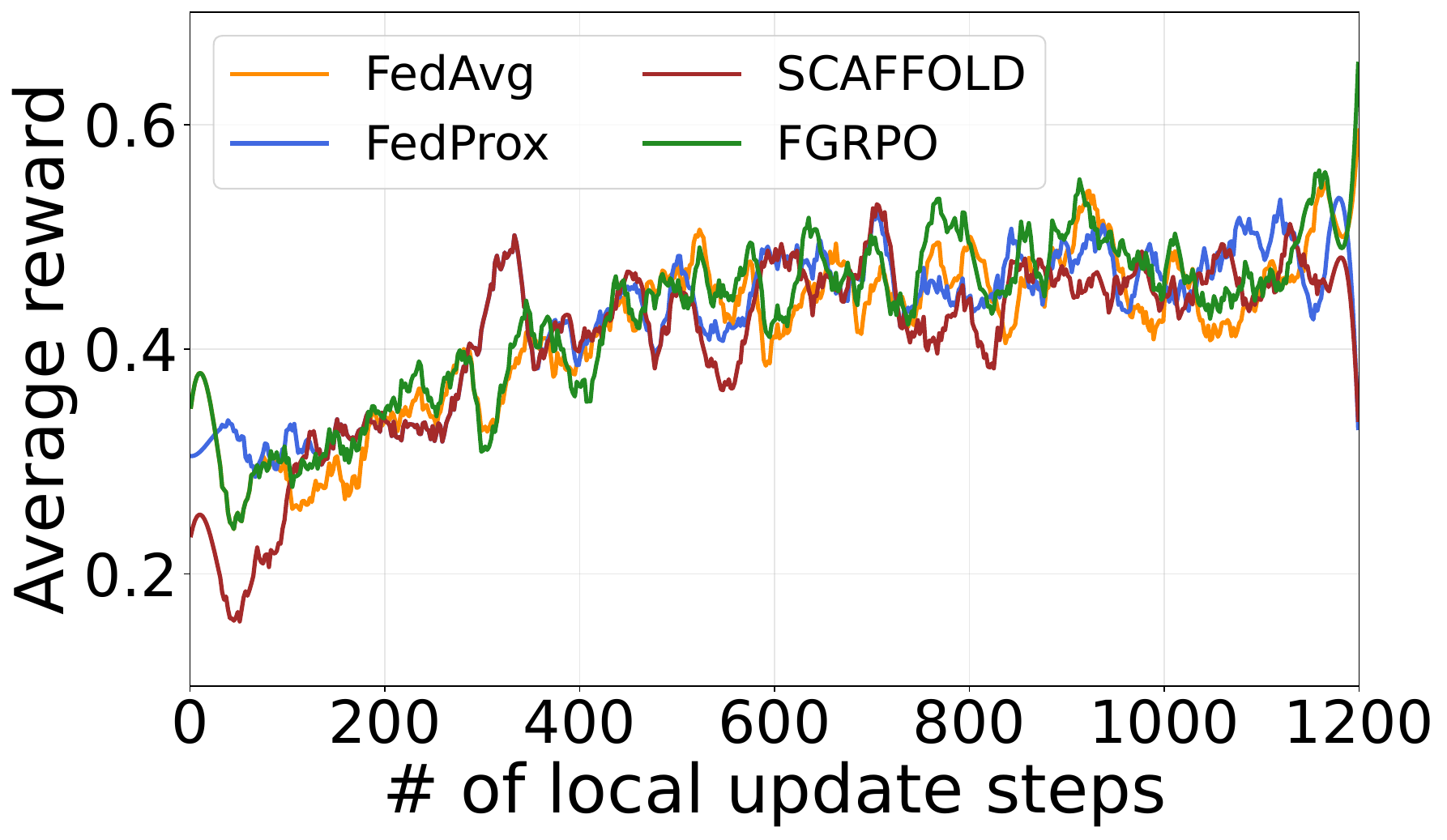}}
      \parbox{.245\textwidth}{\center\includegraphics[width=.245\textwidth]{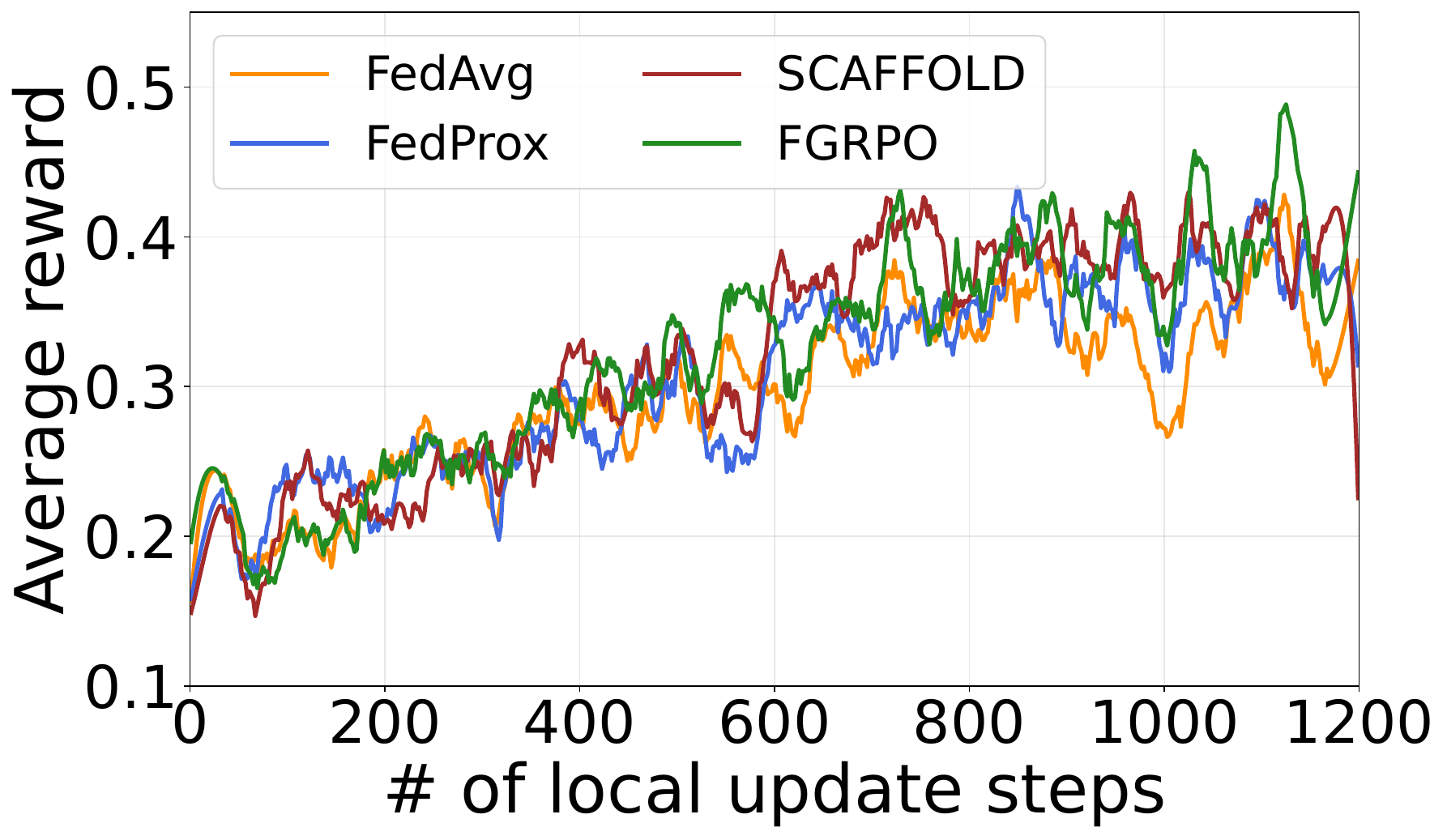}}
      \parbox{.245\textwidth}{\center\includegraphics[width=.245\textwidth]{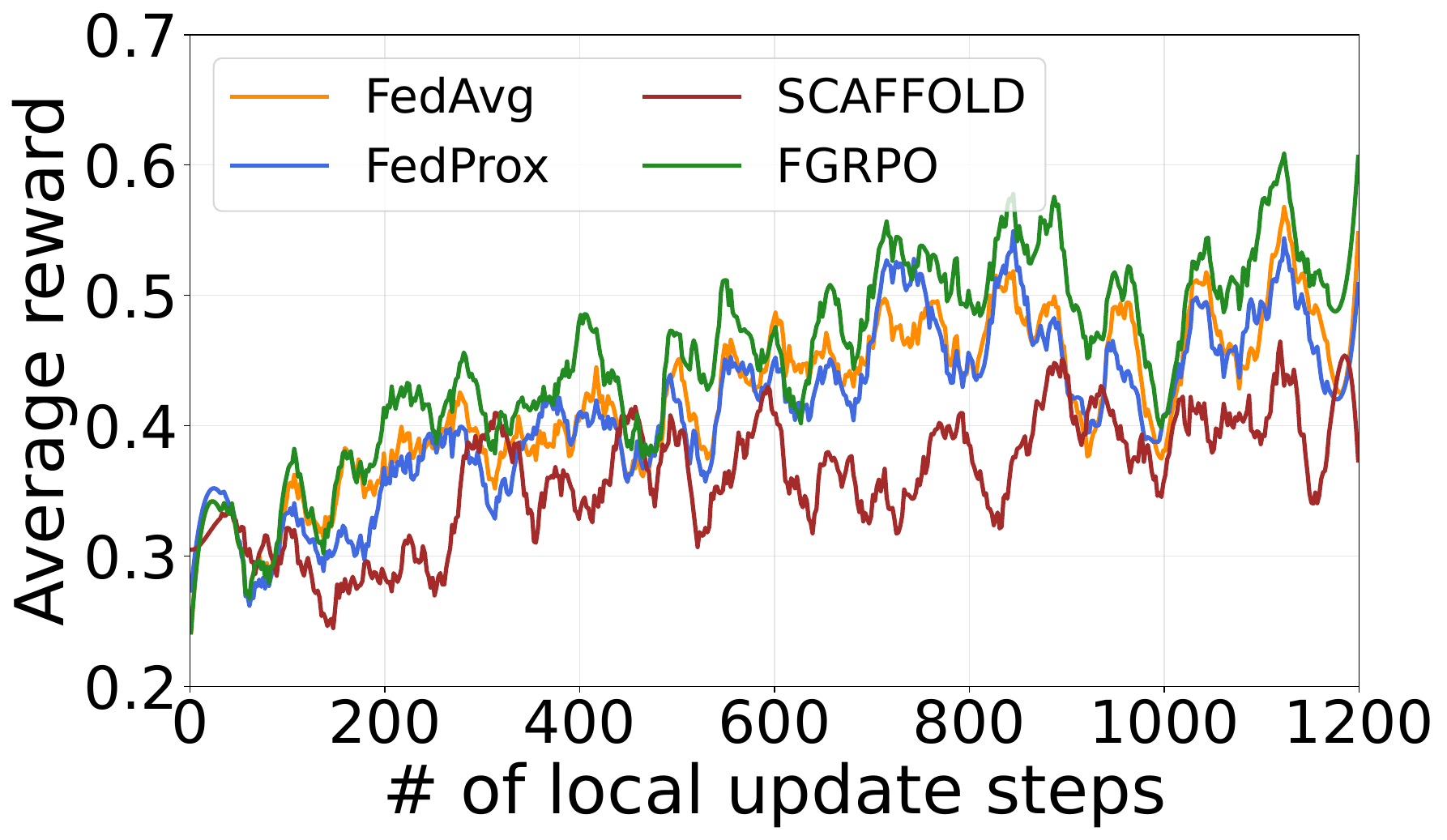}}
      \parbox{.245\textwidth}{\center\scriptsize(a) 3B model on OpenR1}
      \parbox{.245\textwidth}{\center\scriptsize(b) 7B model on OpenR1}
      \parbox{.245\textwidth}{\center\scriptsize(c) 3B model on GEOQA}
      \parbox{.245\textwidth}{\center\scriptsize(d) 7B model on GEOQA}
      \caption{Convergence performance of different algorithms in terms of average reward with Qwen2.5-3B and Qwen2.5-7B models.}
    \label{fig:reward_convergence_comparison_with_baseline}
    \end{figure*}

\section{Brief Literature Survey} \label{sec:relwork}
%
  %
  FedRL synergizes collaborative learning with sequential decision-making to enhance sample efficiency while preserving privacy. Recent theoretical advances have established rigorous convergence guarantees under data heterogeneity \cite{ZhangWMA-ICLR24} and Markovian sampling complexities \cite{KhodadadianSJM-ICML22}, with extensions to offline settings identifying sufficient conditions for global optimality \cite{WooSJC-ICML24}. Beyond theory, research has addressed practical system constraints, including robustness against Byzantine failures \cite{FanMDJTL-NIPS21, FangWG-WWW25} and asynchronous communication delays \cite{LanHHAB-ICLR25}. Furthermore, managing heterogeneity remains a central focus, where methods utilizing momentum-based aggregation \cite{WangHZMA-ICML24}, policy distillation \cite{JiangWZBTF-AAMAS25}, and shared representation learning \cite{XiongWJL-ICLR25} have been proposed to stabilize updates across diverse environments.

  %
  GRPO eliminates the critic network to scale reasoning model fine-tuning via group-level reward normalization \cite{ShaoWZXSBZZLW-arXiv24,GuoYZSWZXZMB-Nature25}. To mitigate the high computational costs and lack of sensitivity in the original framework, recent works have introduced efficiency optimizations such as completion pruning \cite{LinLXJ-arXiv23} and semantic entropy integration \cite{ChenCWY-arXiv25}. Concurrently, precision-enhancing techniques like difficulty-aware advantage reweighting \cite{ZhangZ-arXiv25} and dynamic sampling systems \cite{YuZZYZYDFLL-arXiv25} have been developed to stabilize long reasoning. Whereas these advancements implicitly assume centralized data availability, our proposed framework extends GRPO to distributed settings, enabling the collaborative training of reasoning LLMs across non-IID data. 
  A comprehensive survey is provided in Appendix~\ref{sec:app_survey}.

\section{Conclusion} \label{sec:conclusion}
  In this paper, we introduce \emph{federated group relative policy optimization} (FGRPO), a novel framework designed to enable the federated fine-tuning of reasoning-capable \emph{large language models} (LLMs). To address the critical challenge of divergent reward scales, where data heterogeneity across tasks can disproportionately destabilize the global model's reasoning trajectories, FGRPO leverages an adaptive aggregation mechanism based on \emph{relative performance gain} (RPG). We substantiate FGRPO through rigorous theoretical analysis and extensive empirical validation. Theoretically, we prove that FGRPO matches the asymptotic convergence rate of $\mathcal{O}(1/\sqrt{T})$ achieved by state-of-the-art FedRL algorithms, while uniquely utilizing local updates as an implicit variance stabilizer against stochastic noise. These insights are corroborated by our extensive experiments, which demonstrate that FGRPO achieves superior reasoning performance compared to state-of-the-art methods and maintains robust convergence even under highly heterogeneous data distributions.

    %
    
    %



\bibliographystyle{plain}
\bibliography{ref}



\clearpage
\setcounter{tocdepth}{2}
\tableofcontents
\clearpage

\appendix

\section*{Appendix}
\section{Proof of \textbf{Theorem}~\ref{thm:main_convergence}}
\label{sec:app_analysis}
%
  

  In this section, we present a comprehensive convergence analysis of FGRPO. We begin by a detailed discussion of the theoretical assumptions substantiating our analysis in Sec.~\ref{ssec:app_assumption}. We then provide the rigorous derivation in Sec.~\ref{ssec:detailed_proof}.

  \subsection{Assumptions} \label{ssec:app_assumption}
    We first establish the necessary assumptions which have been widely adopted in the convergence analysis of federated optimization~\cite{McMahanMRHA-AISTATS17,LiSZSTS-MLSys20,ZhangWMA-ICLR24,FanMDJTL-NIPS21,ReddiCZGRKKM-ICLR21}.
    \begin{itemize}
    \item \textbf{Assumption~\ref{ass:main_lsmooth} ($L$-Lipschitz Smoothness).} The objective function is smooth, meaning its gradient does not change arbitrarily fast. This standard assumption permits quadratic upper bounds on the loss, ensuring that gradient steps reliably reduce the objective value.
    \item \textbf{Assumption~\ref{ass:main_various} (Unbiased Gradient and Bounded Variance).} We assume local stochastic gradients are unbiased estimators of the true gradient with bounded noise. This ensures that while individual updates are noisy, the optimization process trends correctly on average without being overwhelmed by variance.
    \item \textbf{Assumption~\ref{ass:main_secmoment} (Bounded Second Moment).} The expected squared norm of the stochastic gradients is uniformly bounded. This prevents the accumulated momentum terms in the Adam optimizer from exploding, ensuring numerical stability during updates.
    \item \textbf{Assumption~\ref{ass:main_datahet} (Bounded Data Heterogeneity).} While data is non-IID, the divergence between local and global gradients is capped by constant $\kappa$. This guarantees that clients share a coherent global objective, making aggregation meaningful despite local distribution shifts.
    \item \textbf{Assumption~\ref{ass:main_boundrpg} (Bounded RPG Score).} The RPG scores are bounded, a property enforced by the algorithm's clipping mechanism. This prevents the aggregation weights from degenerating into a ``winner-take-all'' scenario, preserving the diversity of client contributions.
    \item \textbf{Assumption~\ref{ass:main_precondition} (Bounded Adaptive Preconditioner).} The eigenvalues of the Adam preconditioner matrix are bounded away from zero and infinity. This ensures the adaptive learning rates remain well-conditioned, preventing vanishing or exploding steps common in adaptive optimization.
    \end{itemize}

  \subsection{Proof Details} \label{ssec:detailed_proof}
    To facilitate the convergence analysis of FGRPO, we first establish a series of lemmas in Sec.~\ref{sssec:app_lemmas} to reveal the stepwise descent properties required for the final convergence theorem. \textbf{Lemma}~\ref{le:boundweights} first establishes that the RPG-based aggregation weights are uniformly bounded by $C_\omega = \exp(2h_{\rm max}/\tau_{\min})$, providing a key control over adaptive aggregation. Building on this, \textbf{Lemma}~\ref{le:one-step-descent} derives the fundamental descent inequality based on smoothness, which serves as the backbone of the analysis. \textbf{Lemma}~\ref{le:local-drift-bound} then quantifies the deviation introduced by multi-step local updates, capturing the client-side optimization error. Leveraging the bounded weights from \textbf{Lemma}~\ref{le:boundweights}, \textbf{Lemma}~\ref{le:weighted-variance-hete-control} further bounds the weighted gradient mismatch by $C_\omega \kappa^2$, where $\kappa$ quantifies the level of data heterogeneity. Finally, \textbf{Lemma}~\ref{le:momentum-bounds} bounds the second moments of the momentum and update steps by quantities proportional to $C_\omega$, combining the effects of bounded weights and local updates to control the overall update magnitude. Together, these lemmas bound all sources of error in the gradient descent.

    We then divide into the main proof in Sec.~\ref{sssec:app_mainproof}. The proof starts from a smoothness-based descent inequality for the global objective, with the update penalty controlled by the bounded second moment of the Adam update. It then decomposes the weighted round gradient into the current global gradient plus four error terms: staleness, gradient mismatch, local drift, and stochastic noise. By bounding these terms separately and combining them into a unified aggregate-error bound, the proof substitutes the result back into the descent inequality and telescopes over communication rounds, yielding the final non-convex convergence rate with a heterogeneity- and noise-dependent error floor.

    \subsubsection{Auxiliary Lemmas} \label{sssec:app_lemmas}
      \begin{lemma}[Bounded Weights] \label{le:boundweights}
        Suppose \textbf{Assumption}~\ref{ass:main_boundrpg} holds, then the aggregation weights $\omega_i^{[t]}$ possess a uniform upper bound. Specifically, there exists a constant $C_\omega = \exp(2h_{\rm max}/\tau_{\min})$ such that $0 \le \omega_i^{[t]} \le {C_\omega}/{N}, \forall i,t$.
      \end{lemma}
      \begin{proof}
        Given $\omega^{[t]}_i = \frac{\exp(h^{[t]}_i / \tau^{[t]})}{\sum^N_{j=1} \exp\left(h^{[t]}_{j} / \tau^{[t]}\right)}$ (see Eq.~\eqref{eq:weight} and \textbf{Assumption}~\ref{ass:main_boundrpg}, it follows that $h_j^{[t]} - h_i^{[t]} \ge -2h_{\rm max}$, combining which with $\tau^{[t]}\ge \tau_{\min}$, we have $\sum_{j=1}^N \exp((h_j^{[t]} - h_i^{[t]})/\tau^{[t]}) \ge N \exp(-2h_{\rm max}/\tau_{\min})$, and thus $\omega_i^{[t]} \le \frac{1}{N}\exp\left(\frac{2h_{\rm max}}{\tau_{\min}}\right)$.
      \end{proof}

      \begin{lemma}[One-step Descent] \label{le:one-step-descent} 
        Under \textbf{Assumption}~\ref{ass:main_lsmooth}, we have
        \begin{equation} \label{eq:one_step_descent}
          \mathbb{E}\left[ F(\theta^{[t+1]}) \right] \le \mathbb{E}\left[F(\theta^{[t]})\right] + \mathbb{E}\left[\left\langle \nabla F(\theta^{[t]}), \theta^{[t+1]} - \theta^{[t]} \right\rangle\right] + \frac{L}{2} \mathbb{E}\left[\left\|\theta^{[t+1]} - \theta^{[t]}\right\|^2\right], ~\forall t.
        \end{equation}
      \end{lemma}
      \begin{proof}
        According to \textbf{Assumption}~\ref{ass:main_lsmooth}, we have 
        \begin{equation*}
          F(\theta^{[t+1]}) \le F(\theta^{[t]}) + \left\langle \nabla F(\theta^{[t]}), \theta^{[t+1]} - \theta^{[t]} \right\rangle + \frac{L}{2} \left\|\theta^{[t+1]} - \theta^{[t]}\right\|^2.
        \end{equation*}
        Taking the expectation with respect to the randomness at round $t$, we complete the proof.
      \end{proof}

      \begin{lemma}[Local Drift Bound] \label{le:local-drift-bound}
        Let $\alpha_i = \hat\alpha, ~\forall i\in\mathcal{N}$ for simplicity. Under \textbf{Assumption}~\ref{ass:main_secmoment}, the local model drift satisfies:
        \begin{equation} \label{eq:local_drift_bound}
          \mathbb{E}\left[\left\|\theta_i^{[t,e]}- \theta^{[t]}\right\|^2 \bigg| \mathcal{F}_t \right]
          \le \hat\alpha^2 E \sum_{k=0}^{e-1}\mathbb{E}\left[\left\|g_i^{[t,k]}\right\|^2 \bigg| \mathcal{F}_t\right]
          \le \hat\alpha^2 E^2 G^2, \quad\forall i,t,e
        \end{equation}    
      \end{lemma}
      \begin{proof}
        The local model at step $e$ is given by $\theta_i^{[t,e]} = \theta^{[t]} - \hat\alpha \sum_{k=0}^{e-1} g_i^{[t,k]}$. Therefore, the drift can be calculated as $\theta_i^{[t,e]} - \theta^{[t]} = - \hat\alpha \sum_{k=0}^{e-1} g_i^{[t,k]}$, and its squared norm is
        \begin{equation*}
          \left\|\theta_i^{[t,e]} - \theta^{[t]}\right\|^2 = {\hat\alpha}^2 \left\| \sum_{k=0}^{e-1} g_i^{[t,k]} \right\|^2.
        \end{equation*}
        By applying Jensen's inequality (or Cauchy-Schwarz inequality), we have
        \begin{equation*}
          \left\| \sum_{k=0}^{e-1} g_i^{[t,k]} \right\|^2 \le e \sum_{k=0}^{e-1} \left\| g_i^{[t,k]} \right\|^2.
        \end{equation*}
        Taking the expectation conditioned on $\mathcal{F}_t$, we have
        \begin{equation*}
          \mathbb{E}\left[\left\| \sum_{k=0}^{e-1} g_i^{[t,k]} \right\|^2 \Bigg\vert \mathcal{F}_t\right] \le e \sum_{k=0}^{e-1} \mathbb{E}\big[\|g_i^{[t,k]}\|^2 \mid \mathcal{F}_t\big].
        \end{equation*}
        Since $e \le E$, we apply \textbf{Assumption}~\ref{ass:main_secmoment}, and obtain
        \begin{equation*}
          \mathbb{E}\left[ \left\| \theta_i^{[t,e]}-\theta^{[t]} \right\|^2 \bigg| \mathcal{F}_t\right] \le {\hat\alpha}^2 e \cdot e G^2 \le {\hat\alpha}^2 E^2 G^2,
        \end{equation*}
        which completes the proof.
      \end{proof}

      \begin{lemma}[Weighted Variance \& Heterogeneity Control] \label{le:weighted-variance-hete-control}
        Under \textbf{Assumption}~\ref{ass:main_datahet} and \textbf{Lemma}~\ref{le:boundweights}, the weighted heterogeneity term is bounded by:
        \begin{equation}
          \sum_{i=1}^N \omega_i^{[t]} \left\|\nabla F_i(\theta) - \nabla F(\theta) \right\|^2 \le C_\omega \kappa^2.
        \end{equation}
      \end{lemma}
      \begin{proof}
        The result follows directly from $\omega_i^{[t]} \le {C_\omega}/{N}$ and the heterogeneity assumption $\frac{1}{N} \sum_{i=1}^N \|\nabla F_i - \nabla F\|^2 \le \kappa^2$.
      \end{proof}

      \begin{lemma}[Momentum \& Update Bounds] \label{le:momentum-bounds}
        Under \textbf{Assumption}~\ref{ass:main_secmoment} and \textbf{Lemma}~\ref{le:boundweights}, the second moments of the Adam momentum term $m^{[t]}$ and the update step $z^{[t]} = \theta^{[t+1]} - \theta^{[t]}$ are uniformly bounded such that
        \begin{equation}
          \mathbb{E} \left[ \left\|m^{[t]}\right\|^2 \right] \le M_m^2, \quad \mathbb{E} \left[ \left\|z^{[t]}\right\|^2 \right] \le \alpha^2 c_{\max}^2 M_m^2,\quad\forall t
        \end{equation}
        where $M_m^2 = \hat\alpha^2 C_\omega  E^2 G^2$.
      \end{lemma}
      \begin{proof}
        We first give the bound on the weighted gradient aggregation $\Delta^{[t]} = \sum_{i=1}^N \omega_i^{[t]} (-\hat\alpha \sum_{e=0}^{E-1} g_i^{[t,e]})$. By applying Jensen's inequality, we have
        \begin{equation*}
          \left\|\Delta^{[t]}\right\|^2 \le \sum_{i=1}^N \omega_i^{[t]} \hat\alpha^2 E \sum_{e=0}^{E-1} \left\|g_i^{[t,e]}\right\|^2.
        \end{equation*}
        Furthermore, since $\omega_i^{[t]} \le C_\omega/N$ (see \textbf{Lemma}~\ref{le:boundweights}) and $\mathbb{E}\left[ \left\| g_i^{[t,e]} \right\|^2 \right] \le G^2$ for any $i,t,e$ (see \textbf{Assumption}~\ref{ass:main_secmoment}),
        \begin{equation*}
          \mathbb{E}\left[ \left\| \Delta^{[t]} \right\|^2 \right] \le \frac{1}{N} \sum_{i=1}^N  \hat\alpha^2 C_\omega E^2 G^2 = \hat\alpha^2 C_\omega E^2 G^2.
        \end{equation*}
        According to Jensen's inequality (convex combination), we have $m^{[t]} = (1-\beta_1) \sum_{k=0}^t \beta_1^{t-k} \Delta^{[k]}$ and thus
        \begin{equation*}
          \mathbb{E}\left[ \left\|m^{[t]}\right\|^2 \right]\le \sum_{k=0}^t (1-\beta_1)\beta_1^{t-k} \mathbb{E} \left[\|\Delta^{[k]}\|^2\right] \le \left( \sum_{k=0}^t (1-\beta_1)\beta_1^{t-k} \right) \hat\alpha^2 C_\omega E^2 G^2 \le M_m^2.
        \end{equation*}
        The bound on $z^{[t]}$ directly follows the update rule $\|z^{[t]}\| \le \alpha c_{\rm max} \|m^{[t]}\|$. 
      \end{proof}

    \subsubsection{Main Proof} \label{sssec:app_mainproof}
      \paragraph{i) Global descent inequality.}
      Under \textbf{Assumption}~\ref{ass:main_lsmooth}, for the global update step $z^{[t]} = \theta^{[t+1]} - \theta^{[t]} = \alpha H^{[t]} m^{[t]}$, we expand the objective function at $\theta^{[t]}$:
      \begin{equation}
        F(\theta^{[t+1]}) \le F(\theta^{[t]}) + \left\langle \nabla F(\theta^{[t]}), z^{[t]} \right\rangle + \frac{L}{2} \left\|z^{[t]}\right\|^2.
      \end{equation}
      By taking the expectation on both sides (conditioned on $\mathcal{F}_t$), we have
      \begin{equation} \label{eq:descent_inequality_error}
        \mathbb{E}[F(\theta^{[t+1]})] \le \mathbb{E}\left[ F(\theta^{[t]}) \right] + \underbrace{\alpha \mathbb{E}\left[\left\langle \nabla F(\theta^{[t]}), H^{[t]} m^{[t]} \right\rangle\right]}_{\mathbb{E}[\mathfrak{T}_1]} + \underbrace{\frac{L}{2} \mathbb{E}\left[\left\|z^{[t]}\right\|^2\right]}_{\mathbb{E}[\mathfrak{T}_2]}.
      \end{equation}

      We leverage the result from \textbf{Lemma}~\ref{le:momentum-bounds} to calculate an upper bound on the penalty term:
      \begin{equation} \label{eq:upbd_penalty}
        \mathbb{E}[\mathfrak{T}_2] =\frac{L}{2} \mathbb{E} \left\| \alpha H^{[t]} m^{[t]} \right\|^2 \le \frac{L \alpha^2 c^2_{\rm max}}{2} M_m^2,
      \end{equation}
      where $M_m^2 = C_\omega \hat\alpha^2 E^2 G^2$.

      To analyze the primary term $\mathfrak{T}_1$, we decompose the momentum $m^{[t]}$ into a weighted sum of historical updates. Define
      \begin{equation} \label{eq:avg_grad_round}
        \bar{g}^{[k]} = \frac{1}{E} \sum_{i=1}^N \omega_i^{[k]} \sum_{e=0}^{E-1} g_i^{[k,e]}
      \end{equation}
      as the average gradient in round $k$. Following the Adam update rule $m^{[t]} = (1-\beta_1) \sum_{k=0}^t \beta_1^{t-k} \Delta^{[k]}$ where $\Delta^{[k]} = - \hat\alpha E \bar{g}^{[k]}$, we obtain:
      \begin{equation} \label{eq:T1}
        \mathfrak{T}_1 = - \alpha \hat\alpha E (1-\beta_1) \sum_{k=0}^t \beta_1^{t-k} \left\langle \nabla F(\theta^{[t]}), H^{[t]} \bar{g}^{[k]} \right\rangle.
      \end{equation}

      We decompose the stochastic gradient $g_i^{[k,e]}$ as $g_i^{[k,e]} = \nabla F_i(\theta^{[k]}) + u_i^{[k,e]} + \xi_i^{[k,e]}$, where:
      \begin{itemize}
        %
        \item $u_i^{[k,e]} = \nabla F_i(\theta_i^{[k,e]}) - \nabla F_i(\theta^{[k]})$ denotes the gradient drift on client $i$ at local step $e$ in round $k$, capturing the deviation between gradients evaluated at the locally updated model $\theta_i^{[k,e]}$ and the global model $\theta^{[k]}$.
        %
        %
        \item $\xi_i^{[k,e]} = g_i^{[k,e]} - \nabla F_i(\theta_i^{[k,e]})$ denotes the deviation between the stochastic gradient $g_i^{[k,e]}$ and the true gradient $\nabla F_i(\theta_i^{[k,e]})$ at the local model $\theta_i^{[k,e]}$, indicating the gradient noise for client $i$ at local step $e$ in round $k$. It satisfies the MDS property as specified in \textbf{Assumption}~\ref{ass:main_various}.
      \end{itemize}
      We decompose $\bar{g}^{[k]}$ around the global gradient at round $t$ as the sum of the true gradient and several error terms:
      %
      %
      %
      \begin{align} \label{eq:decomp_grad_est}
        \bar{g}^{[k]} = \nabla F(\theta^{[t]}) + 
        \left(\nabla F(\theta^{[k]}) - \nabla F(\theta^{[t]})\right)
        + \left( \sum_{i=1}^N \omega_i^{[k]} \nabla F_i(\theta^{[k]}) - \nabla F(\theta^{[k]}) \right)
        + E^{-1} U^{[k]}
        + E^{-1} \Xi^{[k]}
      \end{align}
      where 
      \begin{itemize}
        \item \textbf{Staleness error:} $\delta^{[k,t]} = \nabla F(\theta^{[k]}) - \nabla F(\theta^{[t]})$ measures the discrepancy between gradients at past and current iterates, reflecting the staleness introduced by momentum accumulation.
        \item \textbf{Gradient mismatch:} $b^{[k]} = \sum_{i=1}^N \omega_i^{[k]} \nabla F_i(\theta^{[k]}) - \nabla F(\theta^{[k]})$ captures the mismatch between the weighted local gradients and the global gradient, arising from data heterogeneity and adaptive aggregation weights, and is bounded by $C_\omega \kappa^2$.
        \item \textbf{Local Drift:} Let $U^{[k]} = \sum_{i=1}^N \omega_i^{[k]} \sum_{e=0}^{E-1} u_i^{[k,e]}$. Then $E^{-1} U^{[k]}$ represents the average gradient deviation induced by multi-step local updates across clients.
        \item \textbf{Stochastic noise:} Let $\Xi^{[k]} = \sum_{i=1}^N \omega_i^{[k]} \sum_{e=0}^{E-1} \xi_i^{[k,e]}$. Then, $E^{-1}\Xi^{[k]}$ represents the aggregated stochastic gradient noise induced by sampling across clients and local updates.
      \end{itemize}
      %
      %
      %
      Defining the aggregate error term as 
      \begin{equation} \label{eq:agg_error}
        \mathfrak{E}^{[k]} := \delta^{[k, t]} + b^{[k]} + E^{-1} U^{[k]} + E^{-1} \Xi^{[k]},
      \end{equation}
      it follows that $\bar{g}^{[k]} = \nabla F(\theta^{[t]}) + \mathfrak{E}^{[k]}$. Through leveraging $H^{[t]} \succeq c_{\rm min} I$ (see \textbf{Assumption}~\ref{ass:main_precondition}) and Young's Inequality~\cite{young1912classes}, we have
      \begin{align}
        \left\langle \nabla F(\theta^{[t]}), H^{[t]} \bar{g}^{[k]} \right\rangle &= \left\langle \nabla F(\theta^{[t]}), H^{[t]} \nabla F(\theta^{[t]}) \right\rangle + \left\langle \nabla F(\theta^{[t]}), H^{[t]} \mathfrak{E}^{[k]} \right\rangle \nonumber\\
        &\ge c_{\rm min} \left\|\nabla F(\theta^{[t]})\right\|^2 - c_{\rm max} \left\|\nabla F(\theta^{[t]})\right\| \left\|\mathfrak{E}^{[k]}\right\| \nonumber\\
        &\ge \frac{c_{\rm min}}{2} \left\|\nabla F(\theta^{[t]})\right\|^2 - \frac{c^2_{\rm max}}{2c_{\rm min}} \left\|\mathfrak{E}^{[k]}\right\|^2.
      \end{align}
      by combining which with Eq.~\eqref{eq:T1}, we have
      \begin{equation} \label{eq:upper_bound_on_expT1}
        \mathbb{E}[\mathfrak{T}_1] \le - \frac{\alpha \hat\alpha  c_{\rm min} E (1-\beta_1^{t+1})}{2} \mathbb{E}\left[\left\|\nabla F(\theta^{[t]})\right\|^2\right] +  \frac{\alpha \hat\alpha c^2_{\rm max} E}{2c_{\rm min}} (1-\beta_1) \sum_{k=0}^t \beta_1^{t-k} \mathbb{E}\left[\left\|\mathfrak{E}^{[k]}\right\|^2\right].
      \end{equation}
      For sufficiently large $t$, we approximate $1-\beta_1^{t+1} \approx 1$.

      \paragraph{ii) Bounding the aggregated error $\mathfrak{E}$.}
      We now derive specific constant upper bounds for the four components of the total expected squared error $\mathbb{E}\left[\|\mathfrak{E}^{[k]}\|^2\right]$.
      %
      %
      Leveraging the inequality $(\sum_{j=1}^4 x_j)^2 \le 4 \sum_{j=1}^4 x_j^2$, we obtain:
      \begin{equation} \label{eq:exaggregate_error}
        \mathbb{E}\left[\left\|\mathfrak{E}^{[k]}\right\|^2\right] \le 4 \left( \mathbb{E}\left[\left\|\delta^{[k,t]}\right\|^2\right] + \mathbb{E}\left[\left\|b^{[k]}\right\|^2\right] + E^{-2}\mathbb{E}\left[\left\|U^{[k]}\right\|^2\right] + E^{-2}\mathbb{E}\left[\left\|\Xi^{[k]}\right\|^2 \right]\right).
      \end{equation}
      The terms are derived individually as follows:
     
      \textbf{Bounding staleness error:} By the $L$-smoothness shown in \textbf{Assumption}~\ref{ass:main_lsmooth}, we have $\|\delta^{[k,t]}\|^2 = \|\nabla F(\theta^{[k]}) - \nabla F(\theta^{[t]})\|^2 \le L^2 \|\theta^{[k]} - \theta^{[t]}\|^2$. Furthermore, we apply Jensen's inequality to the cumulative updates and thus have
      \begin{equation}
        \left\|\theta^{[t]} - \theta^{[k]}\right\|^2 = \left\|\sum_{j=k}^{t-1} z^{[j]}\right\|^2 = \left\|\sum_{j=k}^{t-1} \alpha H^{[j]} m^{[j]}\right\|^2 \le (t-k) \sum_{j=k}^{t-1} \alpha^2 \left\|H^{[j]} m^{[j]}\right\|^2.
      \end{equation}
      Finally, from \textbf{Lemma}~\ref{le:momentum-bounds} and its corollaries ($\mathbb{E}[\|m^{[j]}\|^2] \le M_m^2$ and $\|H^{[j]}\| \le c_{\rm max}$), we have
      \begin{equation} \label{eq:his_error}
        \mathbb{E}\left[\left\|\delta^{[k,t]}\right\|^2\right] \le L^2 (t-k) \sum_{j=k}^{t-1} \alpha^2 c^2_{\rm max} M_m^2 = L^2 \alpha^2 c^2_{\rm max} M_m^2 (t-k)^2.
      \end{equation}

      \textbf{Bounding gradient mismatch:} Recall $b^{[k]} = \sum_{i=1}^N \omega_i^{[k]} (\nabla F_i(\theta^{[k]}) - \nabla F(\theta^{[k]}))$. Using Jensen's inequality and \textbf{Lemma}~\ref{le:weighted-variance-hete-control}, we have
      \begin{equation}
        \mathbb{E}\left[\left\|b^{[k]}\right\|^2\right] \le \sum_{i=1}^N \omega_i^{[k]} \mathbb{E} \left[\left\|\nabla F_i(\theta^{[k]}) - \nabla F(\theta^{[k]})\right\|^2\right].
      \end{equation}
      Furthermore, given $0 \le \omega_i^{[k]} \le C_\omega/N$ (see \textbf{Lemma}~\ref{le:boundweights}) and $\frac{1}{N}\sum_{i=1}^N \|\nabla F_i(\theta^{[k]}) - \nabla F(\theta^{[k]})\|^2 \le \kappa^2$ (see \textbf{Assumption}~\ref{ass:main_datahet}), we obtain
      \begin{equation} \label{eq:het_bias}
       \mathbb{E}\left[\left\|b^{[k]}\right\|^2\right] \le \frac{C_\omega}{N} \cdot N \kappa^2 = C_\omega \kappa^2.
      \end{equation}

      \textbf{Bounding local drift:} Recall $U^{[k]} = \sum_{i=1}^N \omega_i^{[k]} \sum_{e=0}^{E-1} u_i^{[k,e]}$. By Jensen's inequality, we obtain
      \begin{equation}
        \|U^{[k]}\|^2 \le \left(\sum_i \omega_i^{[k]}\right) \sum_i \omega_i^{[k]} \left\|\sum_e u_i^{[k,e]}\right\|^2 \le \sum_i \omega_i^{[k]} \left( E \sum_e \left\|u_i^{[k,e]}\right\|^2 \right).
      \end{equation}
      Leveraging $L$-smoothness (see \textbf{Assumption}~\ref{ass:main_lsmooth}) and $\mathbb{E}\|\theta_i^{[k,e]} - \theta^{[k]}\|^2 \le \hat\alpha^2 e^2 G^2$ (see \textbf{Lemma}~\ref{le:local-drift-bound}), we obtain
      \begin{equation}
        \mathbb{E}\left[\left\|u_i^{[k,e]}\right\|^2\right] = \mathbb{E}\left[\left\|\nabla F_i(\theta_i^{[k,e]}) - \nabla F_i(\theta^{[k]})\right\|^2\right] \le L^2 \mathbb{E}\left[\left\|\theta_i^{[k,e]} - \theta^{[k]}\right\|^2\right] \le  \hat\alpha^2 e^2 L^2 G^2.
      \end{equation}
      Since $\sum_{e=0}^{E-1} e^2 \le E^3/3 \le E^3$, it follows that $\sum_e \|u_i^{[k,e]}\|^2 \le L^2 \hat\alpha^2 E^3 G^2$. Substituting it into $\mathbb{E}\|U^{[k]}\|^2$, we get
      \begin{equation}  \label{eq:local_drift}
       E^{-2} \mathbb{E} \left[\left\|U^{[k]}\right\|^2\right] \le E^{-2} \cdot \frac{C_\omega}{N} \sum_{i=1}^N ( E \cdot L^2 \hat\alpha^2 E^3 G^2 ) = \hat\alpha^2 C_\omega L^2 E^2 G^2.
      \end{equation}

      \textbf{Bounding stochastic noise:} Recall $\Xi^{[k]} = \sum_{i=1}^N \omega_i^{[k]} \sum_{e=0}^{E-1} \xi^{[k,e]}_i$. We first apply Jensen's inequality to the weighted sum over clients:
      \begin{equation}
        \left\| \Xi^{[k]} \right\|^2 \le \sum_{i=1}^N \omega_i^{[k]} \left\| \sum_{e=0}^{E-1} \xi_i^{[k,e]} \right\|^2.
      \end{equation}
      By the MDS property shown in \textbf{Assumption}~\ref{ass:main_various}, the noise terms $\xi_i^{[k,e]}$ for different local steps $e$ are uncorrelated given the filtration $\mathcal{F}_{k,e}$. Therefore, 
      \begin{equation}
        \mathbb{E} \left[ \left\| \sum_{e=0}^{E-1} \xi_i^{[k,e]} \right\|^2 \right] = \sum^{E-1}_{e=0} \mathbb{E} \left[ \left\| \xi_i^{[k,e]} \right\|^2 \right] \le E \sigma^2.
      \end{equation}
      Furthermore, by considering the weight bound $\omega_i^{[k]} \le C_\omega / N$ from \textbf{Lemma}~\ref{le:boundweights}, we have
      \begin{equation}
        \mathbb{E} \left[\left\| \Xi^{[k]} \right\|^2\right] \le \sum_{i=1}^N \frac{C_\omega}{N} \cdot E \sigma^2 = C_\omega E \sigma^2.
      \end{equation}
      and thus
      \begin{equation} \label{eq:stochastic_noise}
        E^{-2} \mathbb{E}\left[\left\| \Xi^{[k]} \right\|^2\right] \le E^{-2} \cdot C_\omega E \sigma^2 = \frac{C_\omega \sigma^2}{E}.
      \end{equation}

      Substituting Eqs. \eqref{eq:his_error}, \eqref{eq:het_bias}, \eqref{eq:local_drift}, and \eqref{eq:stochastic_noise} into Eq.~\eqref{eq:exaggregate_error}, we have
      \begin{equation} \label{eq:bdonXi}
        \mathbb{E}\left[\left\|\mathfrak{E}^{[k]}\right\|^2\right] \le 4 \left( C_{0} \alpha^2 (t-k)^2 + C_{1} \right).
      \end{equation}
      where
      \begin{equation}
        C_{0} = L^2 c^2_{\rm max} M_m^2 =  \hat\alpha^2 c^2_{\rm max} L^2 E^2 G^2 \exp(2h_{\rm max} / \tau_{\rm min})
      \end{equation}
      and
      \begin{equation}
        C_{1} = C_\omega \left(\kappa^2 + \hat\alpha^2 L^2 E^2 G^2 + \frac{\sigma^2}{E}\right) = \left(\kappa^2 + \hat\alpha^2 L^2 E^2 G^2 + \frac{\sigma^2}{E}\right) \exp(2h_{\rm max} / \tau_{\rm min}),
      \end{equation}


      \paragraph{iii) Final Synthesis and Convergence Bound.}
      Substituting the aggregated error bound (Eq.~\eqref{eq:bdonXi}) back into the expression of $\mathbb{E}[\mathfrak{T}_1]$~\eqref{eq:upper_bound_on_expT1}, we obtain:      
      \begin{equation} \label{eq:bound_T1}
        \mathbb{E}[\mathfrak{T}_1] \le - \frac{\alpha \hat\alpha c_{\rm min} E}{2} \mathbb{E}\left[\left\|\nabla F(\theta^{[t]})\right\|^2\right] + \alpha \hat\alpha E \frac{2c^2_{\rm max}}{c_{\rm min}} (1-\beta_1) \sum_{k=0}^t \beta_1^{t-k} \left( C_{\rm hist} \alpha^2 (t-k)^2 + C_{\rm err} \right).
      \end{equation}
      where the approximation $(1-\beta_1^{t+1}) \approx 1$ for sufficiently large $t$ is adopted. Furthermore, substituting Eq.~\eqref{eq:bound_T1} and Eq.~\eqref{eq:upbd_penalty} into the single-step descent inequality in Eq.~\eqref{eq:descent_inequality_error}, we obtain
      \begin{equation}
        \frac{\alpha \hat\alpha  c_{\rm min} E}{2} \mathbb{E} \left[\left\|\nabla F(\theta^{[t]})\right\|^2\right] \le \mathbb{E} \left[ F(\theta^{[t]}) \right] - \mathbb{E} \left[ F(\theta^{[t+1]}) \right] + \frac{L \alpha^2 c^2_{\rm max}}{2} M_m^2 + \mathfrak{R}^{[t]},
      \end{equation}
      where 
      \begin{equation}
        \mathfrak{R}^{[t]} =  \alpha \hat\alpha E \frac{2c^2_{\rm max}}{c_{\rm min}} (1-\beta_1) \sum_{k=0}^t \beta_1^{t-k} \left( C_{\rm hist} \alpha^2 (t-k)^2 + C_{\rm err} \right).
      \end{equation}
      Summing over $t = 0, \ldots, T-1$, we obtain
      \begin{align} \label{eq:sumgradient}
        &\frac{\alpha \hat\alpha  c_{\rm min} E}{2} \sum^{T-1}_{t=0} \mathbb{E} \left[\left\|\nabla F(\theta^{[t]})\right\|^2\right] \nonumber\\ 
        \le& \sum^{T-1}_{t=0} \left( \mathbb{E} \left[ F(\theta^{[t]}) \right] - \mathbb{E} \left[ F(\theta^{[t+1]}) \right] \right) + \sum^{T-1}_{t=0} \frac{L \alpha^2 c^2_{\rm max}}{2} M_m^2 + \sum^{T-1}_{t=0} \mathfrak{R}^{[t]} \nonumber\\
        \le& \mathbb{E}[F(\theta^{[0]})] - \mathbb{E}[F(\theta^{[T]})] + \frac{\alpha^2 c^2_{\rm max} L T M^2_m}{2} + \sum^{T-1}_{t=0} \mathfrak{R}^{[t]}  \nonumber\\
        \le& F(\theta^{[0]}) - F^* + \frac{\alpha^2 c^2_{\rm max} L T M^2_m}{2} + \sum^{T-1}_{t=0} \mathfrak{R}^{[t]}.
      \end{align}

      Therein,
      \begin{align}
        \sum^{T-1}_{t=0} \mathfrak{R}^{[t]} =& \sum^{T-1}_{t=0} \left(   \frac{2 \alpha \hat\alpha c^2_{\rm max} (1-\beta_1) E}{c_{\rm min}}   \sum_{k=0}^t \beta_1^{t-k} \left( C_{\rm hist} \alpha^2 (t-k)^2 + C_{\rm err} \right)\right)  \nonumber\\
        =& \frac{2 \alpha \hat\alpha c^2_{\rm max} (1-\beta_1) E}{c_{\rm min}}  \sum^{T-1}_{t=0} \left( \sum_{k=0}^t \beta_1^{t-k} \left( C_{\rm hist} \alpha^2 (t-k)^2 + C_{\rm err} \right)\right)  \nonumber\\
        %
        %
        %
        =& \frac{2 \alpha^3 \hat\alpha c^2_{\rm max} (1-\beta_1) E C_{\rm hist}}{c_{\rm min}}  \sum^{T-1}_{t=0} \sum^t_{k=0} \beta_1^{t-k} (t-k)^2 \nonumber\\
        & + \frac{2 \alpha \hat\alpha c^2_{\rm max} (1-\beta_1) E C_{\rm err}}{c_{\rm min}} \sum^{T-1}_{t=0} \sum^t_{k=0} \beta_1^{t-k} 
        %
      \end{align}
      Since 
      \begin{align*}
        \sum^{T-1}_{t=0} \sum^t_{k=0} \beta_1^{t-k} (t-k)^2 
        = \sum^{T-1}_{t=0} \sum^t_{j=0} \beta^j_1 j^2  
        \leq \sum^{T-1}_{t=0} \sum^\infty_{j=0} \beta^j_1 j^2
        = \frac{T\beta_1 (1+\beta_1)}{(1-\beta_1)^3}
      \end{align*}
      and
      \begin{align*}
        \sum^{T-1}_{t=0} \sum^t_{k=0} \beta_1^{t-k}  
        =\sum^{T-1}_{k=0} \sum^{T-1}_{t=k} \beta^{t-k}_1
        =\sum^{T-1}_{k=0} \sum^{T-k-1}_{j=0} \beta^j_1
        = \frac{T}{1-\beta_1},
      \end{align*}
      we have
      \begin{align} \label{eq:error_r}
        \sum^{T-1}_{t=0} \mathfrak{R}^{[t]} =& \sum^{T-1}_{t=0} \left(   \frac{2 \alpha \hat\alpha c^2_{\rm max} (1-\beta_1) E}{c_{\rm min}}   \sum_{k=0}^t \beta_1^{t-k} \left( C_{\rm hist} \alpha^2 (t-k)^2 + C_{\rm err} \right)\right)  \nonumber\\
        \le& \frac{2 \alpha^3 \hat\alpha c^2_{\rm max} (1-\beta_1) E C_{\rm hist}}{c_{\rm min}} \cdot \frac{T\beta_1(1+\beta_1)}{(1-\beta_1)^3} + \frac{2 \alpha \hat\alpha c^2_{\rm max} (1-\beta_1) E C_{\rm err}}{c_{\rm min}} \cdot \frac{T}{1-\beta_1} \nonumber\\
        =& \frac{2 \alpha^3 \hat\alpha c^2_{\rm max} \beta_1(1+\beta_1) T E C_{\rm hist}}{c_{\rm min} (1-\beta_1)^2} + \frac{2 \alpha \hat\alpha c^2_{\rm max} T E C_{\rm err}}{c_{\rm min}} 
      \end{align}
      By substituting \eqref{eq:error_r} into \eqref{eq:sumgradient}, we have

      \begin{align} \label{eq:convergence222}
      \frac{1}{T} \sum_{t=0}^{T-1} \mathbb{E}\left[\left\|\nabla F(\theta^{[t]})\right\|^2\right] 
      \le& \frac{2(F(\theta^{[0]}) - F^*)}{\alpha \hat\alpha c_{\rm min} E T}  
      + \frac{4 c_{\rm max}^2 C_{\rm err}}{c_{\rm min}^2}  \nonumber\\
      & + \frac{4 c_{\rm max}^2 \alpha^2 \beta_1(1+\beta_1) C_{\rm hist}}{c_{\rm min}^2 (1-\beta_1)^2} 
      + \frac{L \alpha c_{\rm max}^2 M_m^2}{\hat\alpha c_{\rm min} E} \end{align}
      %
      %
      When $\alpha=1/{T^{\frac{1}{4}}}$ and $\hat\alpha ={1}/{(LET^{\frac{1}{4}})}$ (i.e., the clients and the server all adopt a decaying learning rate), by plugging the following symbols into \eqref{eq:convergence222}
      \begin{equation}
      \begin{cases}
        C_\omega = \exp\!\left(\frac{2 h_{\max}}{\tau_{\min}}\right), \\
        C_{0} = L^2 c_{\max}^2 M_m^2 = \frac{c_{\max}^2 G^2}{\sqrt{T}} \exp\!\left(\frac{2 h_{\max}}{\tau_{\min}}\right), \\
        C_{1} = C_\omega\!\left(\kappa^2 + \hat{\alpha}^2 L^2 E^2 G^2 + \frac{\sigma^2}{E}\right)= \exp\!\left(\frac{2 h_{\max}}{\tau_{\min}}\right)\left(\kappa^2 + \frac{G^2}{\sqrt{T}} + \frac{\sigma^2}{E}\right) \\
        M_m^2 = \hat{\alpha}^2 C_\omega E^2 G^2 = \frac{G^2}{L^2 \sqrt{T}} \exp\!\left(\frac{2 h_{\max}}{\tau_{\min}}\right), \\
      \end{cases}
      \end{equation}
      %
      %
    we have
    \begin{align}
      & \frac{1}{T} \sum_{t=0}^{T-1} \mathbb{E}\!\left[\left\|\nabla F(\theta^{[t]})\right\|^2\right] \nonumber\\
      \le& \frac{2(F(\theta^{[0]}) - F^*)}{\alpha \hat{\alpha} c_{\min} E T} + \frac{4 c_{\max}^2 C_{\mathrm{err}}}{c_{\min}^2} + \frac{4 c_{\max}^2 \alpha^2 \beta_1(1+\beta_1) C_{\mathrm{hist}}}{c_{\min}^2 (1-\beta_1)^2} + \frac{L \alpha c_{\max}^2 M_m^2}{\hat{\alpha} c_{\min} E} \nonumber\\
      =&\frac{2L(F(\theta^{[0]}) - F^*)}{c_{\min}\sqrt{T}} + \frac{4 c_{\max}^2}{c_{\min}^2}\exp\!\left(\frac{2 h_{\max}}{\tau_{\min}}\right)\left(\kappa^2 + \frac{G^2}{\sqrt{T}} + \frac{\sigma^2}{E}\right) \nonumber\\
      &+ \frac{4 c_{\max}^4 \beta_1(1+\beta_1) G^2}{c_{\min}^2 (1-\beta_1)^2}\frac{1}{T}\exp\!\left(\frac{2 h_{\max}}{\tau_{\min}}\right) + \frac{c_{\max}^2 G^2}{c_{\min}\sqrt{T}}\exp\!\left(\frac{2 h_{\max}}{\tau_{\min}}\right) \nonumber\\
      =& \frac{4 c_{\max}^2}{c_{\min}^2}\exp\!\left(\frac{2 h_{\max}}{\tau_{\min}}\right)\left(\kappa^2 + \frac{\sigma^2}{E}\right) + \frac{2L(F(\theta^{[0]}) - F^*)}{c_{\min}\sqrt{T}} \nonumber\\
      &+ \left(\frac{4 c_{\max}^2 G^2}{c_{\min}^2}\exp\!\left(\frac{2 h_{\max}}{\tau_{\min}}\right) + \frac{c_{\max}^2 G^2}{c_{\min}}\exp\!\left(\frac{2 h_{\max}}{\tau_{\min}}\right)\right)\frac{1}{\sqrt{T}} \nonumber\nonumber\\
      &+ \frac{4 c_{\max}^4 \beta_1(1+\beta_1) G^2}{c_{\min}^2 (1-\beta_1)^2}\frac{1}{T}\exp\!\left(\frac{2 h_{\max}}{\tau_{\min}}\right) \nonumber\\
      %
      %
      =& \frac{2L(F(\theta^{[0]}) - F^*)}{c_{\min}\sqrt{T}} + \frac{4 c_{\max}^2}{c_{\min}^2}\exp\!\left(\frac{2 h_{\max}}{\tau_{\min}}\right)\left(\kappa^2 + \frac{\sigma^2}{E}\right) + \mathcal{O}\!\left(\frac{1}{\sqrt{T}}\right)
    \end{align} 
    which completes the proof of \textbf{Theorem}~\ref{thm:main_convergence}.

    \begin{remark}
      The role of the upper bound on the RPG values, $h_{\max}$, is captured entirely through the quantity $C_\omega=\exp(2h_{\max}/\tau_{\min})$ introduced in \textbf{Lemma}~\ref{le:boundweights}. Because the RPG-based aggregation weights follow a softmax rule (see Eq.~\eqref{eq:weight}), a larger $h_{\max}$ increases the worst-case spread of the logits and therefore allows the weights to become more concentrated on a small subset of clients, whereas a smaller $h_{\max}$ leads to more balanced aggregation; this is consistent with the standard behavior of temperature-scaled softmax distributions, where a larger logit spread yields a sharper probability distribution. This effect propagates throughout the proof. In particular, \textbf{Lemma}~\ref{le:weighted-variance-hete-control} amplifies the weighted heterogeneity term from $\kappa^2$ to $C_\omega \kappa^2$, showing that $h_{\max}$ does not change the intrinsic heterogeneity level itself, but magnifies its impact under adaptive weighting. The same factor also appears in \textbf{Lemma}~\ref{le:momentum-bounds} through $M_m^2=\hat\alpha^2 C_\omega E^2 G^2$, implying that a larger $h_{\max}$ increases the worst-case magnitude of both the momentum and the update steps. These dependencies further enter the main proof through the bounds on gradient mismatch, local drift, and stochastic noise in Eqs.~\eqref{eq:het_bias}, \eqref{eq:local_drift}, and \eqref{eq:stochastic_noise}, and are collected in the aggregate error bound \eqref{eq:bdonXi} through the constants $C_0$ and $C_1$. Consequently, in the final convergence bound, the irreducible error floor is multiplied by $\exp(2h_{\max}/\tau_{\min})$, which means that a larger $h_{\max}$ leads to sharper weighting and a higher error floor, while a smaller $h_{\max}$ yields more conservative weighting and improved stability.
      %
    \end{remark}

\section{Implementation of LoRA under FGRPO framework} \label{sec:app_lora}
  Low-Rank Adaptation (LoRA)~\cite{HuSWALWWC-ICLR22} is a parameter-efficient fine-tuning method that adapts large pre-trained models without updating the full set of model parameters. For a linear layer with frozen pre-trained weight $\mathbf{W}_0$, LoRA introduces a trainable low-rank update $\Delta\mathbf{W}$ and uses the effective weight $\mathbf{W}=\mathbf{W}_0+\Delta\mathbf{W}$ during the forward pass. The update is parameterized as $\Delta\mathbf{W}=\frac{\alpha}{r}\mathbf{B}\mathbf{A}$, where $\mathbf{A}$ and $\mathbf{B}$ are trainable low-rank matrices, $r$ is the LoRA rank, and $\alpha$ is the scaling factor. Since only $\mathbf{A}$ and $\mathbf{B}$ are optimized, LoRA substantially reduces the number of trainable parameters, GPU memory usage, and communication cost, while keeping the backbone model frozen and incurring no additional inference latency after merging the adapters into the base weights.

  We implement LoRA under our FGRPO framework by restricting both local optimization and server-client communication to the adapter parameters. Specifically, LoRA modules are inserted into the target linear projection layers of each decoder layer in the language-model backbone, including the self-attention projections (\texttt{q\_proj}, \texttt{k\_proj}, \texttt{v\_proj}, and \texttt{o\_proj}) and the MLP projections (\texttt{gate\_proj}, \texttt{up\_proj}, and \texttt{down\_proj}). We set the LoRA rank to $r=32$ and the scaling factor to $\alpha=128$. At the beginning of each communication round $t$, the server broadcasts only the global adapter parameters $\theta^{[t]}=\{\mathbf{A}^{[t]},\mathbf{B}^{[t]}\}$ to the clients, while the backbone weights $\mathbf{W}_0$ remain fixed throughout training. Each client then performs local GRPO updates on its private data, where gradients are computed only with respect to the LoRA adapters. After local training, client $i$ uploads its updated adapters $\theta^{[t]}_i=\{\mathbf{A}^{[t]}_i,\mathbf{B}^{[t]}_i\}$ together with its round-level reward statistic $\bar{R}^{[t]}_i$. The server applies our RPG-based adaptive weighting mechanism to these uploaded adapters and aggregates them to form the next global adapter. 
  %

  We report the per-round communication overhead of the LoRA-based implementation in Table~\ref{tab:communication_overhead}. Since LoRA freezes the backbone and updates only low-rank adapter parameters, it substantially reduces the number of transmitted parameters. The communication overhead remains moderate across model scales, ranging from 228.45 MB to 400.08 MB. Specifically, the overhead is 228.45 MB, 252.07 MB, 308.06 MB, and 400.08 MB for the 3B, 4B, 7B, and 11B models, respectively. These results show that even when scaling to larger backbones, FGRPO communicates only a compact set of trainable adapter parameters, making federated reasoning-policy optimization practical under limited communication resources.
  \begin{table}[t]
  \centering
  \caption{Per-round communication overhead of FGRPO for each client across different backbone models.}
  \label{tab:communication_overhead}
  \small
  \setlength{\tabcolsep}{7pt}
  \begin{tabular}{lcccc}
  \toprule
  \textbf{Model} & \textbf{Qwen2.5-3B} & \textbf{Qwen3-4B} & \textbf{Qwen2.5-7B} & \textbf{Llama-3.2-11B} \\
  \midrule
  \textbf{Communication overhead (MB)} 
  & 228.45 & 252.07 & 308.06 & 400.08 \\
  \bottomrule
  \end{tabular}
  \end{table}

\section{More Information about Experiment Settings} \label{sec:app_expsetting}
  We conduct our extensive experiments on the following two datasets:
\begin{itemize}
    \item \textbf{Multimodal-Open-R1-8k-Verified}~\cite{openr1_verified} consists of 8,000+ high-quality samples focusing on diverse mathematical disciplines, including algebra, probability, and functional analysis. It requires models to perform multi-step deduction grounded in various visual contexts such as statistical charts and functional plots. In our framework, we adopt an outcome-based verification protocol to evaluate the model's ability to synthesize coherent, long-chain reasoning paths across broad logical domains, focusing on the accuracy of the final answer.
    \item \textbf{GEOQA-8k}~\cite{chen2025r1v} is designed for specialized geometric reasoning and comprises 8,000+ problems pairing 2D geometric diagrams with natural language descriptions. It emphasizes spatial reasoning and geometric constraints, requiring models to derive precise numerical solutions through joint spatial-logical synthesis. Following the same outcome-based verification protocol, the framework is assessed on its capacity to arrive at correct answers based on axiomatic theorems without intermediate process-based supervision.
  \end{itemize}
  %


  For the OpenR1 dataset, we stratify samples into three difficulty tiers (\emph{simple}, \emph{medium}, and \emph{hard}) based on equal tertiles of reasoning trace length. We use trace length as a metric to evaluate task complexity, as longer traces typically involve more logical transitions and self-correction steps. To simulate data heterogeneity, we allocate these tiers to clients using a Dirichlet distribution parameterized by $\mu$, so that different clients are exposed to varying levels of reasoning difficulty. 
  %
  %
  We examine different levels of data heterogeneity by varying $\mu \in \{0.05, 1.0\}$, corresponding to \emph{highly} and \emph{moderately} non-IID regimes, respectively, along with a \emph{uniform} baseline with $\mu \rightarrow \infty$.

  For the GEOQA dataset, we model domain heterogeneity by partitioning samples according to topological complexity and geometric primitives (e.g., \emph{points}, \emph{lines}, \emph{circles}, and \emph{polygons}). Using the same Dirichlet-based allocation, we assign clients to specialize in disjoint visual concepts, such as circular reasoning versus polygonal construction. Unlike the complexity-driven imbalance in OpenR1, this setup evaluates the model’s ability to aggregate diverse feature representations and generalize across structurally distinct geometric subdomains without incurring catastrophic forgetting.

  In this paper, we compare FGRPO with the following three representative FL baselines:
  \begin{itemize}
    \item \textbf{FedAvg}~\cite{McMahanMRHA-AISTATS17} is the classical FL algorithm, which performs iterative model averaging across decentralized clients. In each communication round, each client performs local GRPO optimization on its private dataset, and the server aggregates the resulting local models via a weighted average. Specifically, the aggregation weights are proportional to the local data volume, i.e., $\omega_i = \frac{D_i}{\sum_{i'=1}^{N} D_{i'}}$, where $D_i$ denotes the size of client $i$'s dataset. This procedure enables collaborative model training without sharing raw data, while naturally emphasizing clients with larger datasets.
    %
    %
    %
    \item \textbf{FedProx}~\cite{LiSZSTS-MLSys20} is a heterogeneity-aware extension of FedAvg designed to improve federated optimization under non-IID data. In the FGRPO setting, we adapt FedProx by adding a proximal regularization term $\|\theta_i-\theta^{[t]}\|^2_2$ to each client's local GRPO objective. 
    %
    %
    %
    %
    %
    %
    \item \textbf{SCAFFOLD}~\cite{KarimireddyKMRS-ICML20} is a variance-reduction-based federated optimization method that explicitly addresses client drift under non-IID data distributions. Specifically, SCAFFOLD equips each client $i$ with control variates that are incorporated into its local gradient $g^{[t,e]}_i$, thereby correcting the discrepancy between local and global update directions. These control variates estimate the deviation between the client-specific gradient and the global objective, enabling each local update to better align with the global optimization trajectory and significantly reducing the variance induced by heterogeneous data.
    %
    %
  %
  \end{itemize}

\section{Supplementary Experiment Results} \label{sec:appexp}
  We first present complementary comparison results in Sec.\ref{ssec:addcomp}. We then analyze the impact of the number of clients in Sec.\ref{ssec:client}, followed by the effect of data heterogeneity in Sec.\ref{ssec:datahet}. Next, we examine the extension of our framework to GRPO-family variants with RPG ablations in Sec.\ref{ssec:extension}, and conduct hyperparameter sensitivity analysis in Sec.\ref{ssec:sensitivity}. Finally, we report the resource consumption of FGRPO in Sec.\ref{ssec:resconsumption}.

  \subsection{Additional Comparison Results} 
  \label{ssec:addcomp}
    \begin{figure*}[t!]
    \centering
      \parbox{.245\textwidth}{\center\includegraphics[width=.245\textwidth]{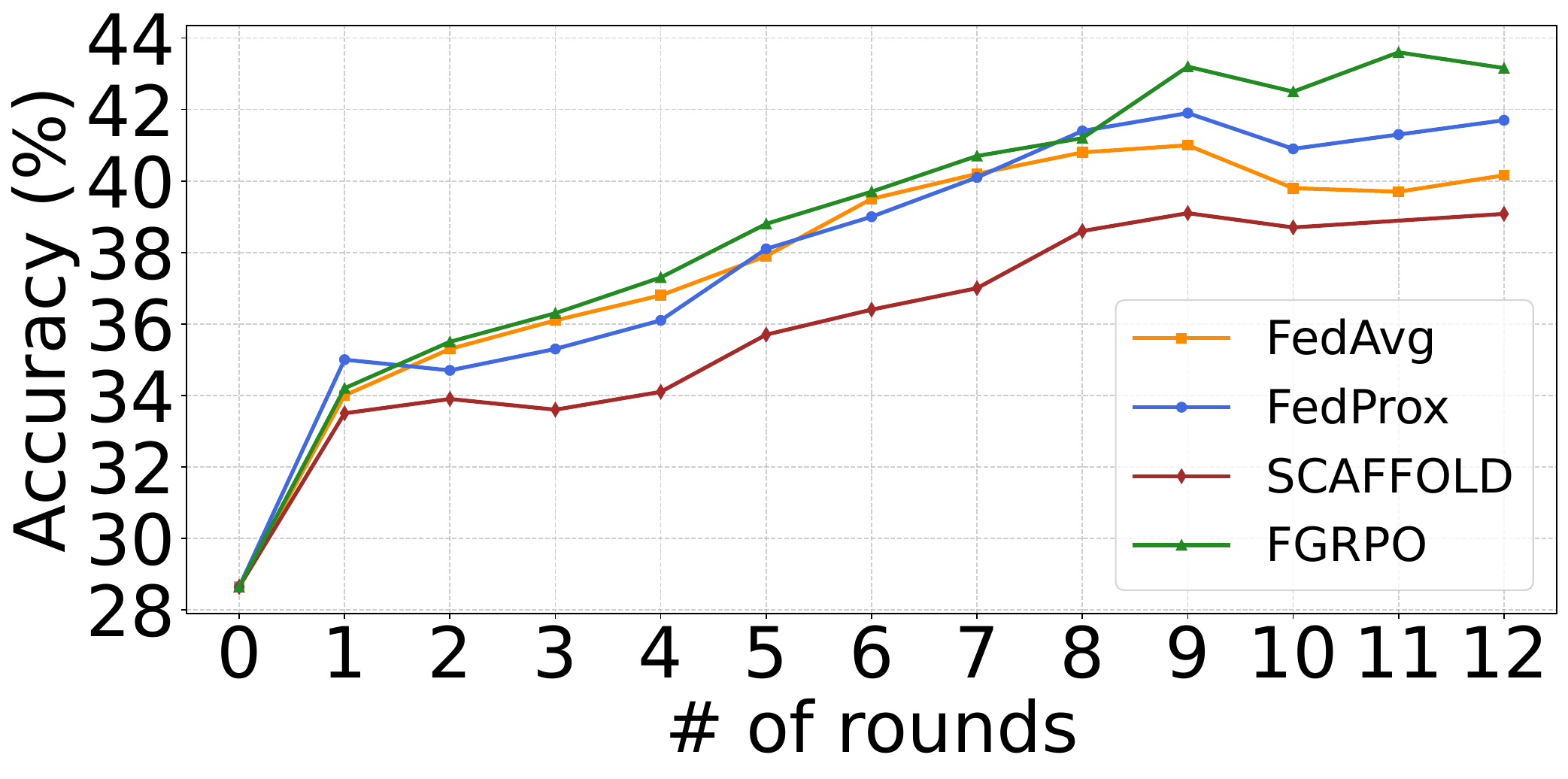}}
      \parbox{.245\textwidth}{\center\includegraphics[width=.245\textwidth]{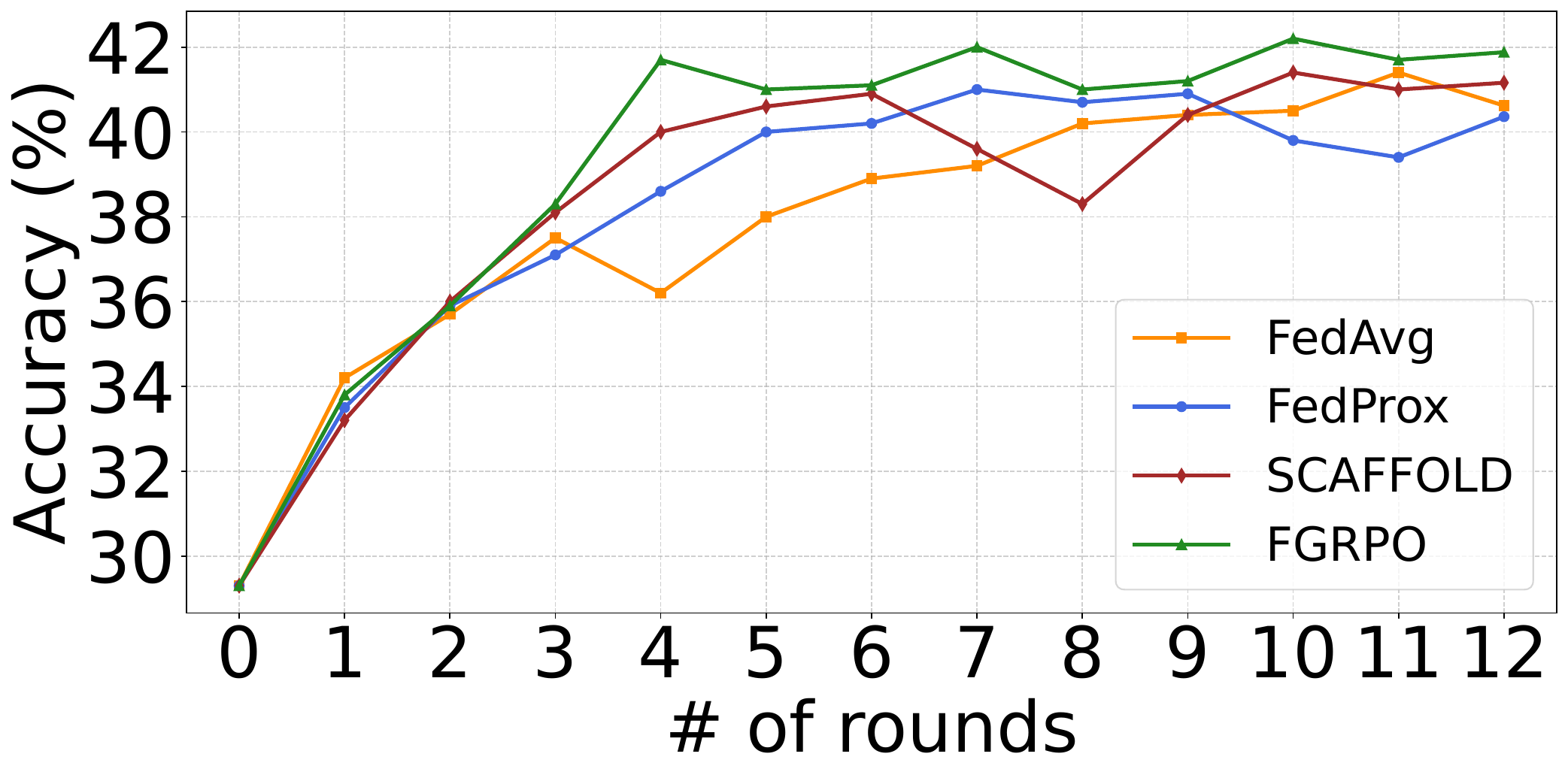}}
      \parbox{.245\textwidth}{\center\includegraphics[width=.245\textwidth]{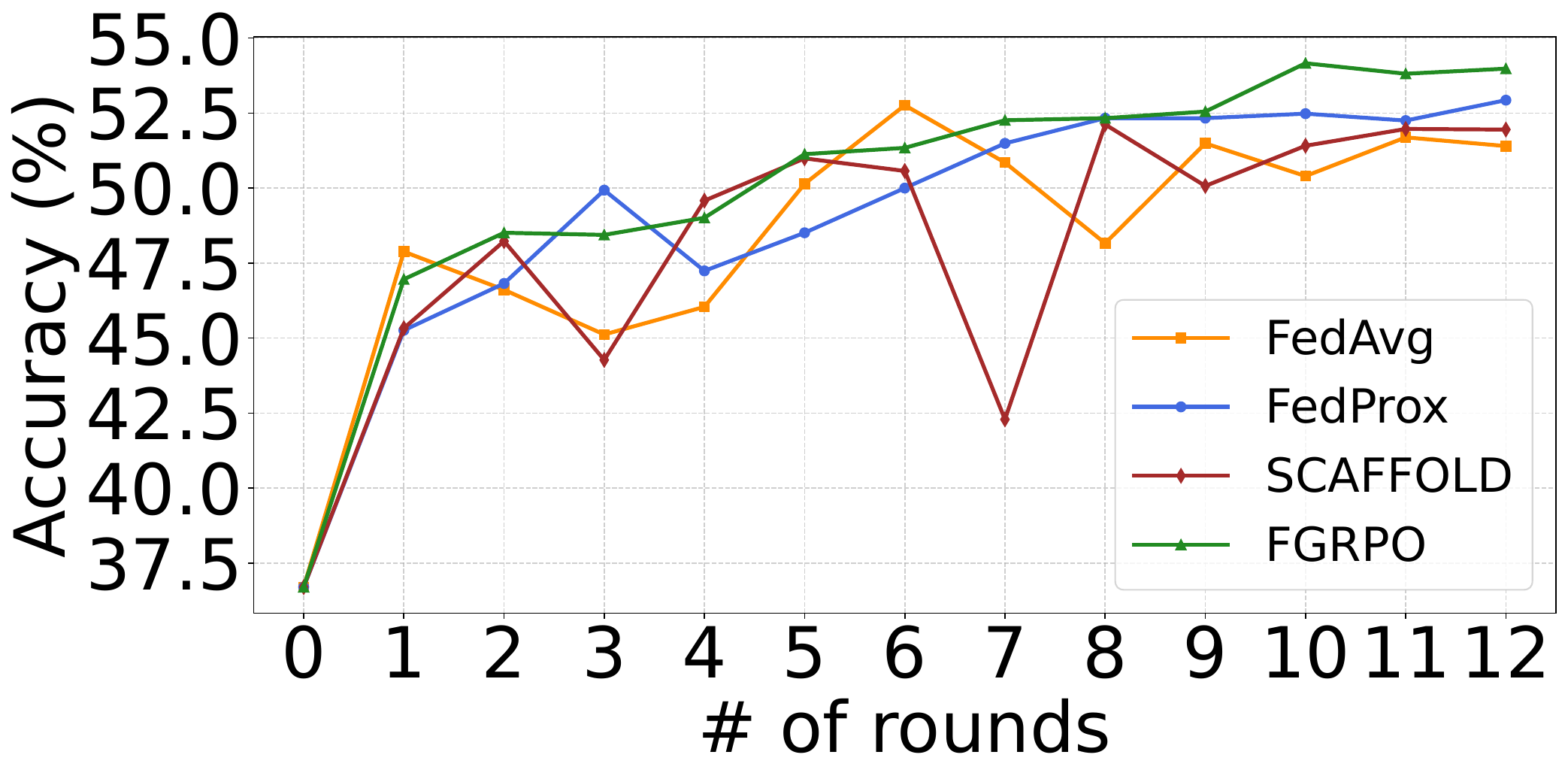}}
      \parbox{.245\textwidth}{\center\includegraphics[width=.245\textwidth]{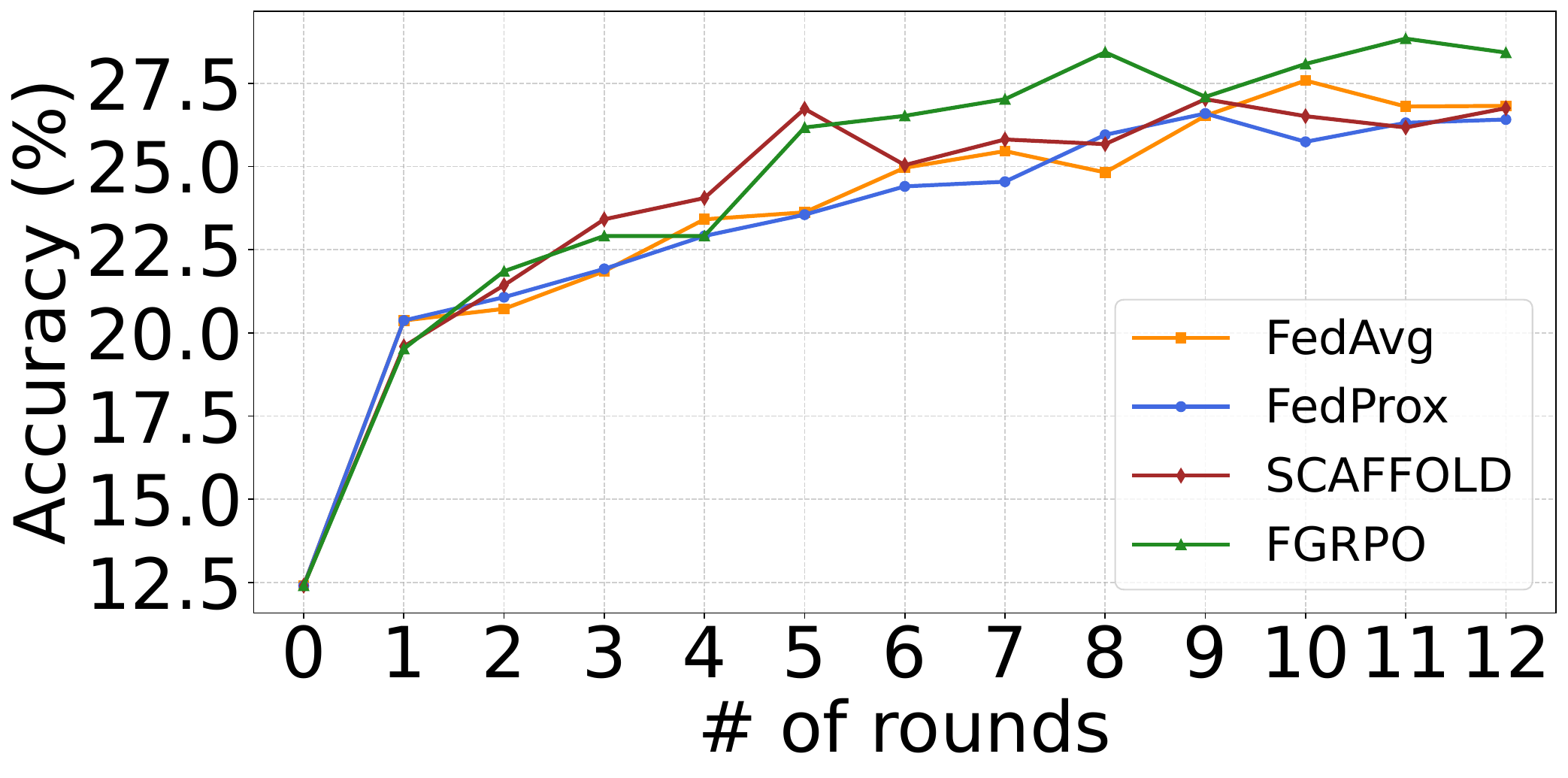}}
      \parbox{.245\textwidth}{\center\scriptsize(a) 4B model on OpenR1}
      \parbox{.245\textwidth}{\center\scriptsize(b) 11B model on OpenR1}
      \parbox{.245\textwidth}{\center\scriptsize(c) 4B model on GEOQA}
      \parbox{.245\textwidth}{\center\scriptsize(d) 11B model on GEOQA}
      \caption{Test accuracy convergence trajectories of the different models (Qwen3-4B and Llama-3.2-11B) on OpenR1 and GEOQA datasets.}
    \label{fig:appen_acc_convergence}
    \end{figure*}
    \begin{figure*}[t!]
    \centering
      \parbox{.245\textwidth}{\center\includegraphics[width=.245\textwidth]{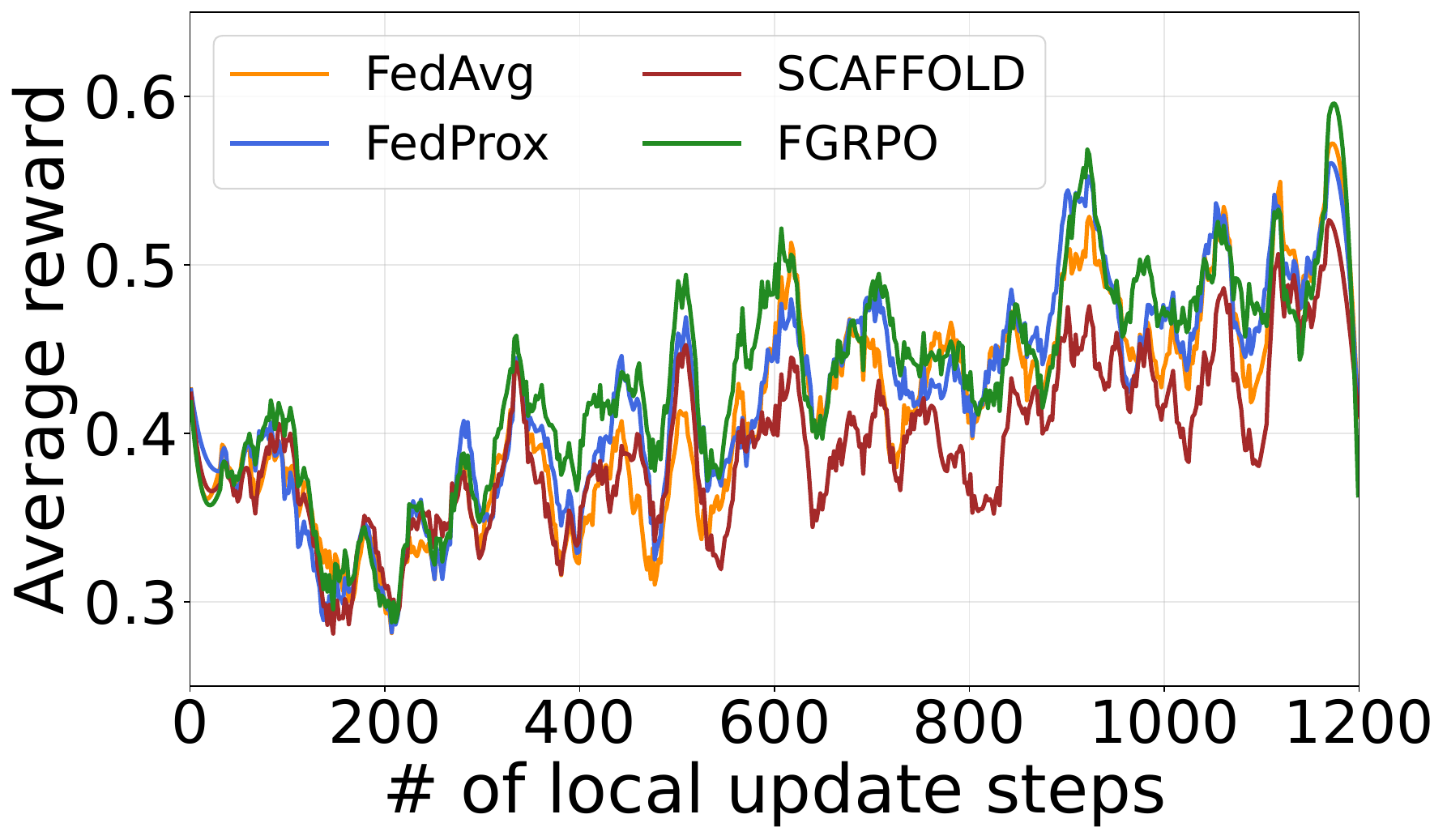}}
      \parbox{.245\textwidth}{\center\includegraphics[width=.245\textwidth]{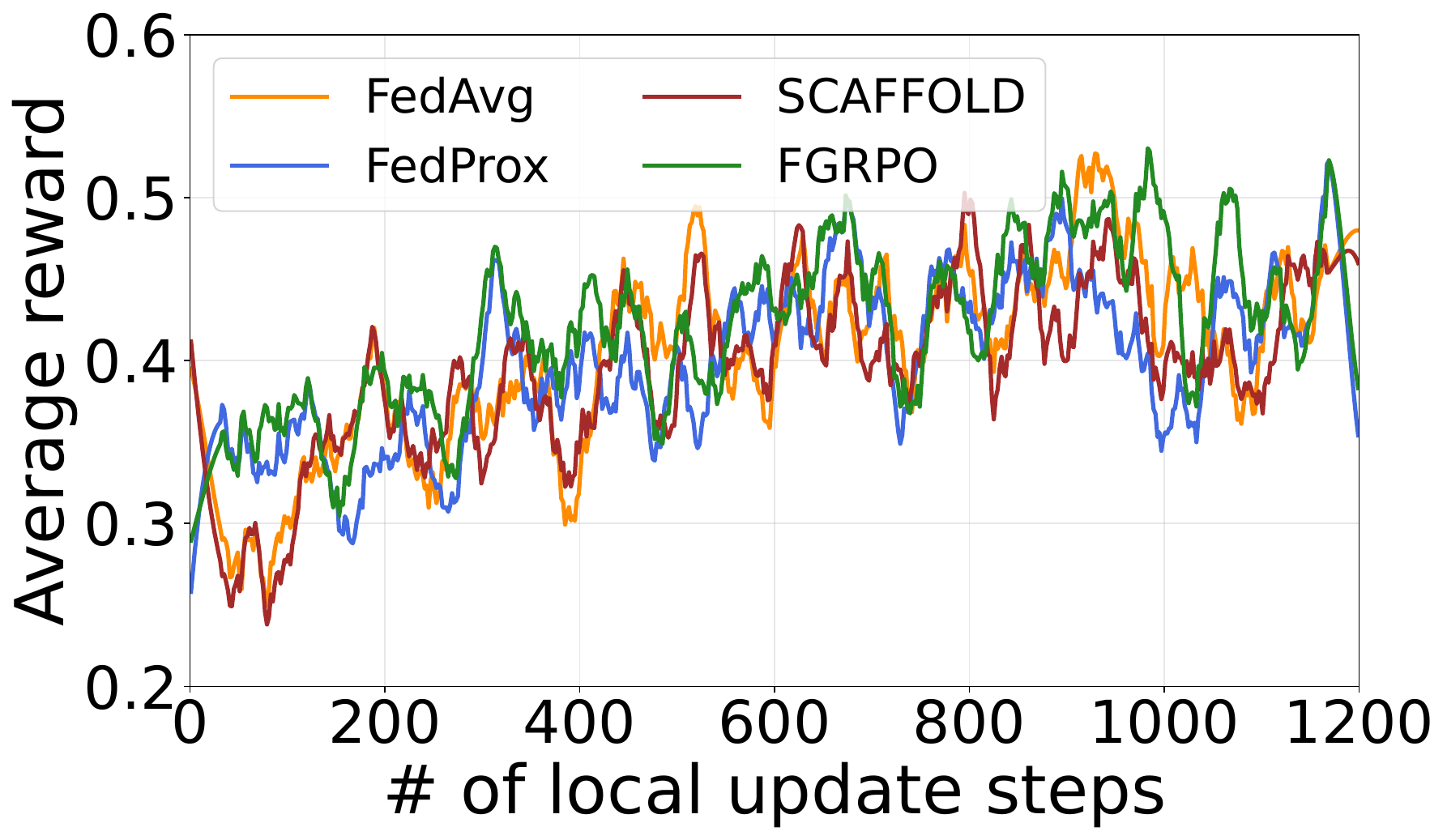}}
      \parbox{.245\textwidth}{\center\includegraphics[width=.245\textwidth]{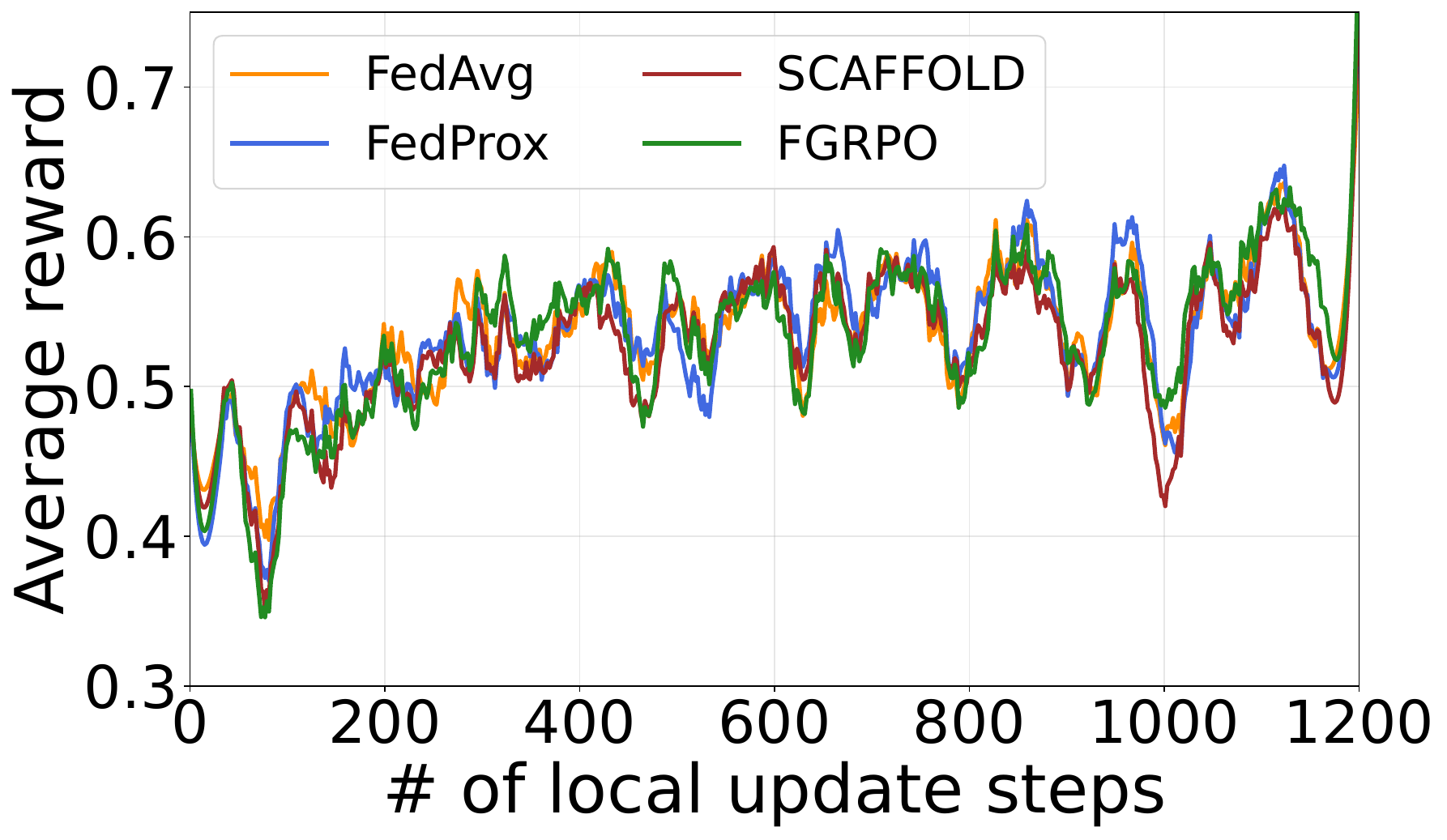}}
      \parbox{.245\textwidth}{\center\includegraphics[width=.245\textwidth]{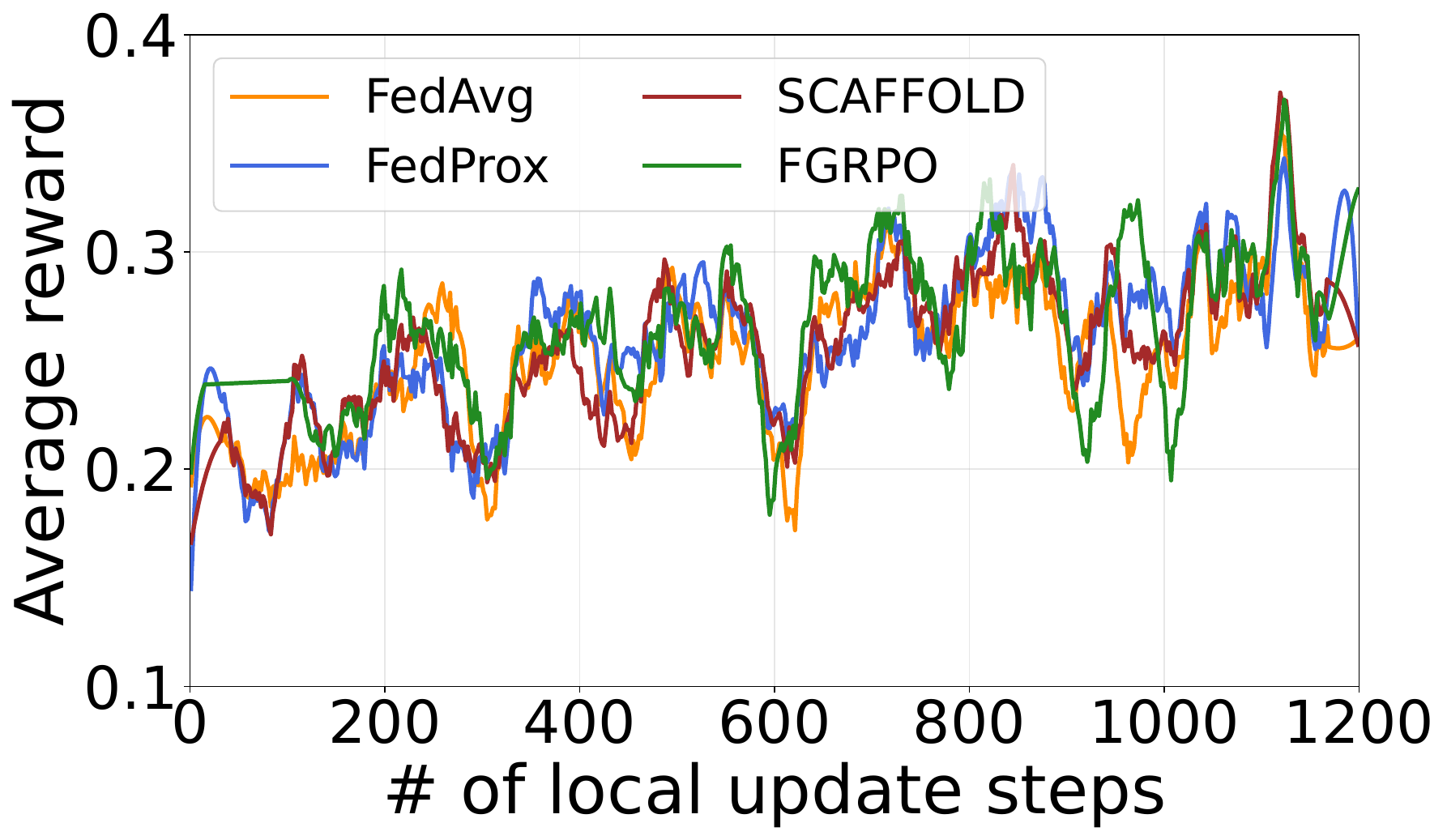}}
      \parbox{.245\textwidth}{\center\scriptsize(a) 4B model on OpenR1}
      \parbox{.245\textwidth}{\center\scriptsize(b) 11B model on OpenR1}
      \parbox{.245\textwidth}{\center\scriptsize(c) 4B model on GEOQA}
      \parbox{.245\textwidth}{\center\scriptsize(d) 11B model on GEOQA}
      \caption{Average reward trajectories of the different models (Qwen3-4B and Llama-3.2-11B) on OpenR1 and GEOQA datasets.}
      \label{fig:Appen_reward_convergence_comparison_with_baseline}
    \end{figure*}

    In this section, we provide additional comparison results on Qwen3-4B and Llama-3.2-11B models to further complement the results shown in Figs.~\ref{fig:acc_convergence}--\ref{fig:reward_convergence_comparison_with_baseline}. %
    As demonstrated in Fig.~\ref{fig:appen_acc_convergence}, for Qwen3-4B model, FGRPO exhibits a clear upward trajectory and reaches the highest final accuracy on both Open-R1 and GEOQA. In contrast, the baseline methods either plateau earlier or show more noticeable fluctuations in later communication rounds. For Llama-3.2-11B model, FGRPO also maintains a stable advantage in the later stages of training, especially on GEOQA, where it consistently stays above the other federated baselines after the middle communication rounds. These observations indicate that RPG-based aggregation can provide stable optimization benefits across different models.

    The reward trajectories in Fig.~\ref{fig:Appen_reward_convergence_comparison_with_baseline} are consistent with the accuracy curves. On both Qwen3-4B and Llama-3.2-11B models, FGRPO generally achieves more favorable average reward trajectories across local update steps. This shows that FGRPO not only improves final evaluation accuracy, but also leads to more effective reinforcement learning dynamics during federated training. By aggregating clients according to relative performance gains rather than absolute reward magnitudes, FGRPO better captures meaningful local progress under heterogeneous reward distributions.

  \subsection{Impact of the Number of Clients}  \label{ssec:client}
    We evaluate the test accuracy of different methods under varying numbers of clients ($N \in \{3, 5, 10\}$) using the Qwen2.5-3B model, as shown in Fig.~\ref{fig:number_of_clients}. The results show that FGRPO consistently outperforms all baselines, maintaining an accuracy above 42.12\% on OpenR1 even at $N=10$, substantially higher than the 37.44\% achieved by FedProx, demonstrating strong scalability to larger client populations. Moreover, the performance gap between FGRPO and the baselines widens as the number of clients increases. On GEOQA, the accuracy gap between FGRPO and SCAFFOLD grows from 2.42\% at $N=3$ to 6.79\% at $N=10$, suggesting that FGRPO remains effective when aggregating updates from more clients. A similar trend is observed against FedProx, where the margin improves from 4.51\% to 4.6\%. Overall, these results demonstrate that RPG-based aggregation enhances learning performance across varying numbers of clients.
    \begin{figure*}[t!]
    \centering
      \parbox{.4\textwidth}{\center\includegraphics[width=.38\textwidth]{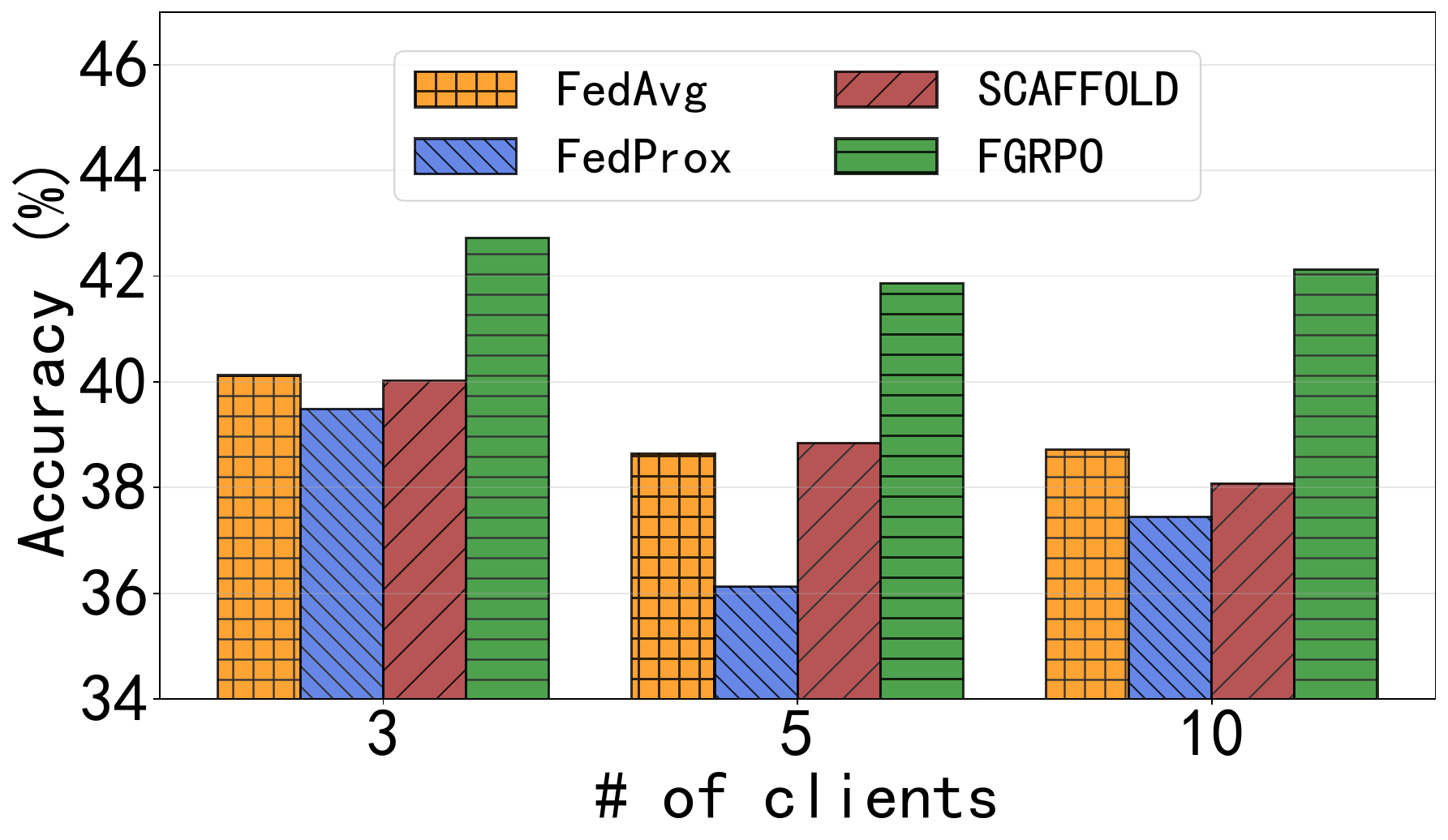}}
      \parbox{.4\textwidth}{\center\includegraphics[width=.38\textwidth]{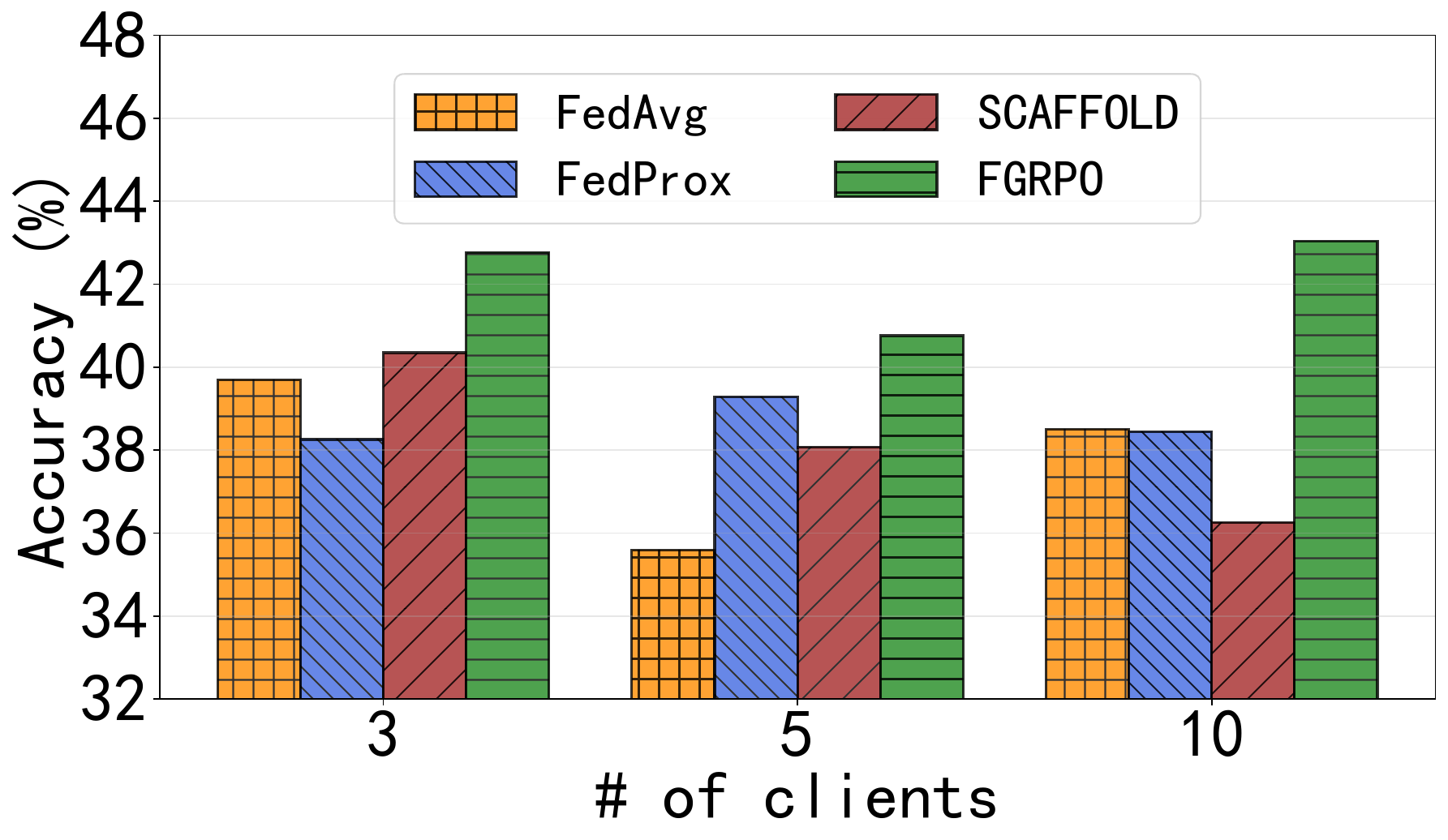}}
      \parbox{.4\textwidth}{\center\scriptsize(a) OpenR1}
      \parbox{.4\textwidth}{\center\scriptsize(b) GEOQA}
      \caption{Test accuracy under different numbers of clients.}
    \label{fig:number_of_clients}
    \end{figure*}

    To further analyze the effect of client population size, we report the average reward trajectories in Fig.~\ref{fig:reward_Client_OpenR1} and Fig.~\ref{fig:reward_client_GEOQA}. On both datasets, all methods show a generally increasing reward trend as local training proceeds, while the reward curves become relatively smoother when the number of clients increases from $N=3$ to $N=10$. This suggests that aggregating feedback from a larger client population can partially reduce the stochastic fluctuations caused by individual local updates. Compared with FedAvg, FedProx, and SCAFFOLD, FGRPO consistently maintains higher average reward trajectories across different client population sizes. When the federation is small ($N=3$), the training dynamics are more volatile because each client has a larger influence on the global update. Even in this setting, FGRPO achieves higher reward levels and shows stronger robustness to local reward noise. As the number of clients increases to $N=5$ and $N=10$, FGRPO further exhibits a more stable upward trend and sustains a clear reward advantage over the baselines on both Open-R1 and GEOQA. These results indicate that the benefit of FGRPO is not limited to a specific federation size. While increasing the number of clients can improve the stability of federated aggregation, the performance gap between FGRPO and the baselines shows that client population size alone is insufficient. The proposed RPG-based aggregation remains important for identifying clients with meaningful relative reward improvements and for mitigating the negative effect of heterogeneous reward scales during the federated training.
    \begin{figure*}[t!]
    \centering
      \parbox{.325\textwidth}{\center\includegraphics[width=.325\textwidth]{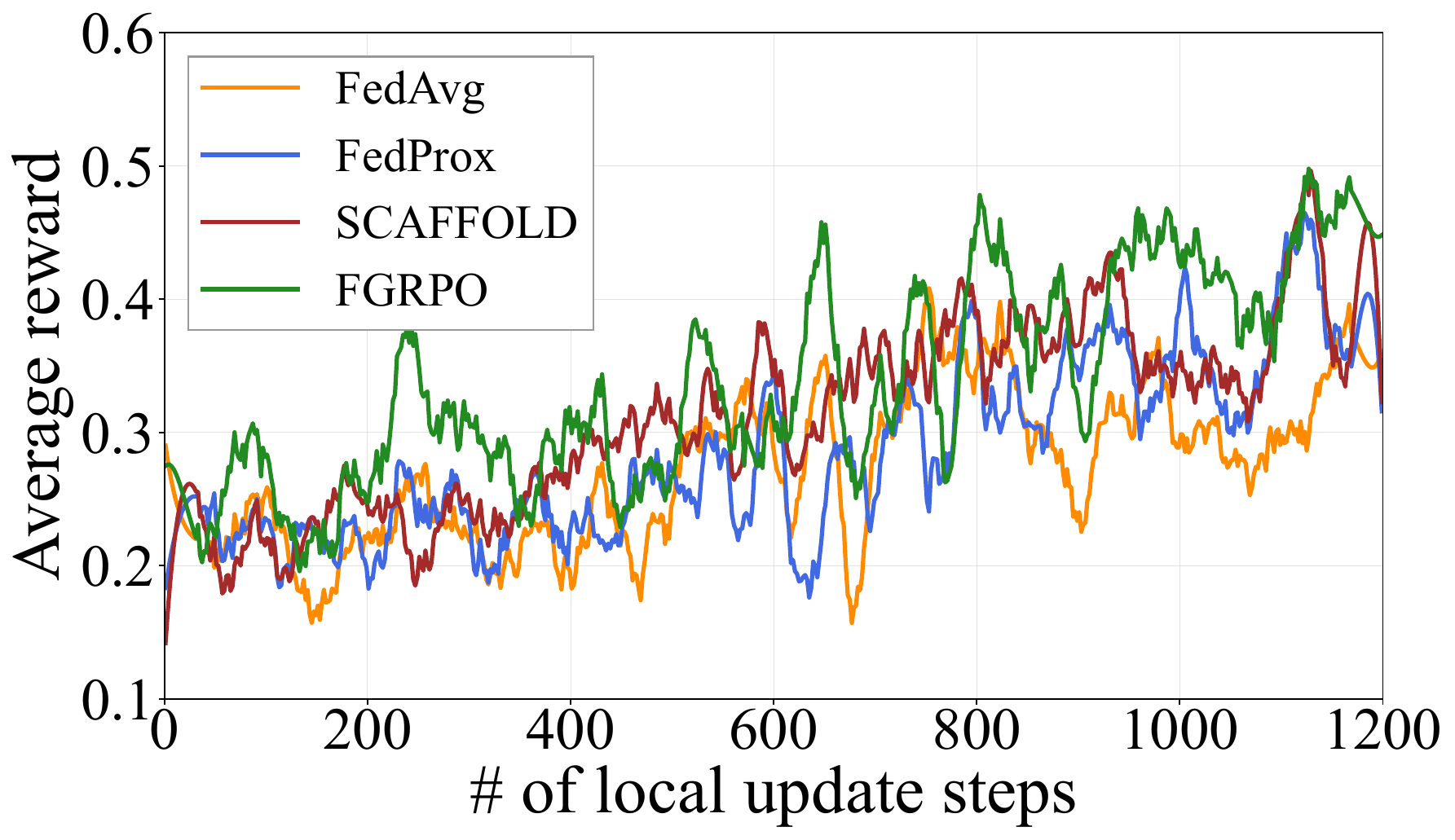}}
      \parbox{.325\textwidth}{\center\includegraphics[width=.325\textwidth]{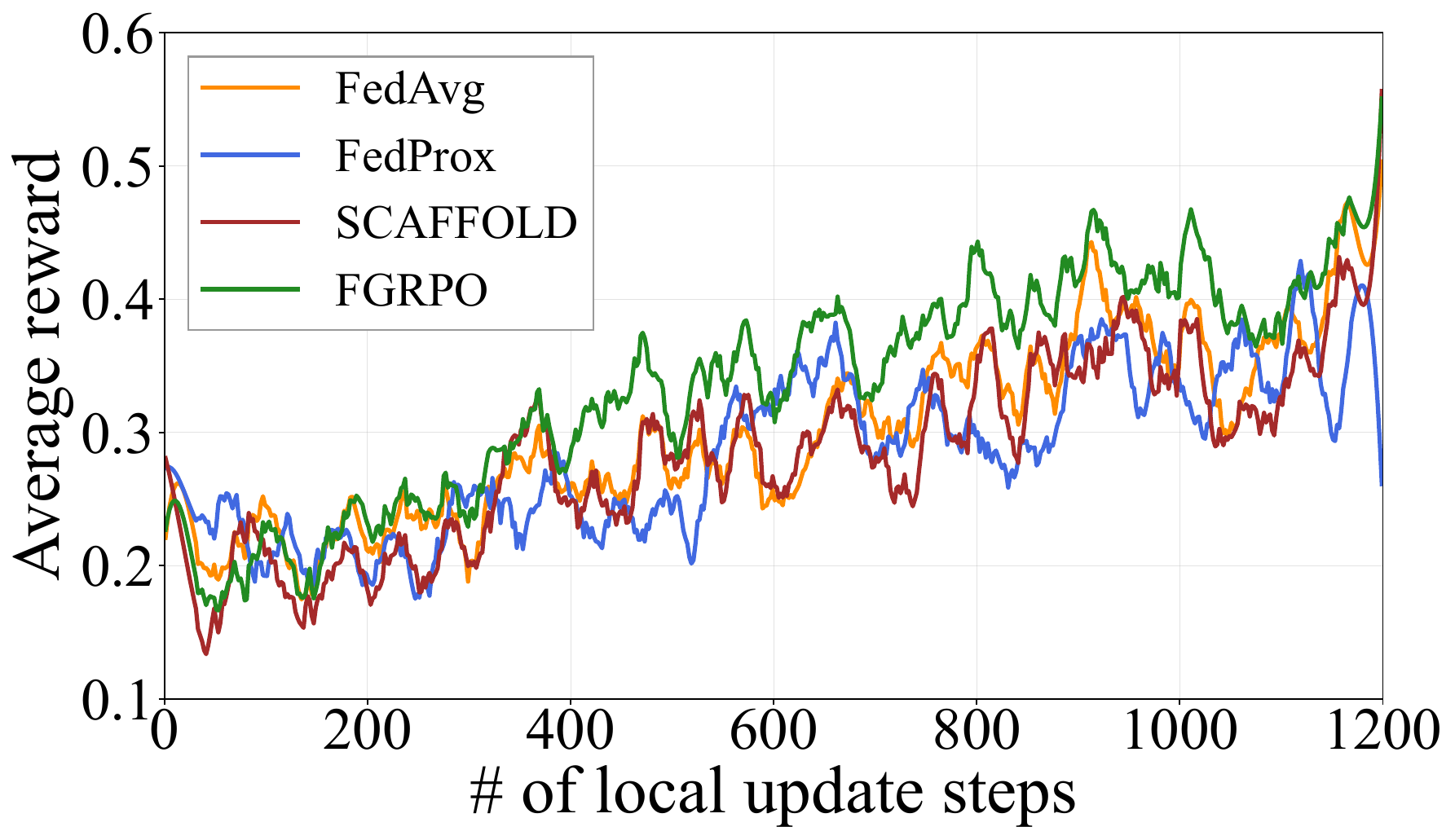}}
      \parbox{.325\textwidth}{\center\includegraphics[width=.325\textwidth]{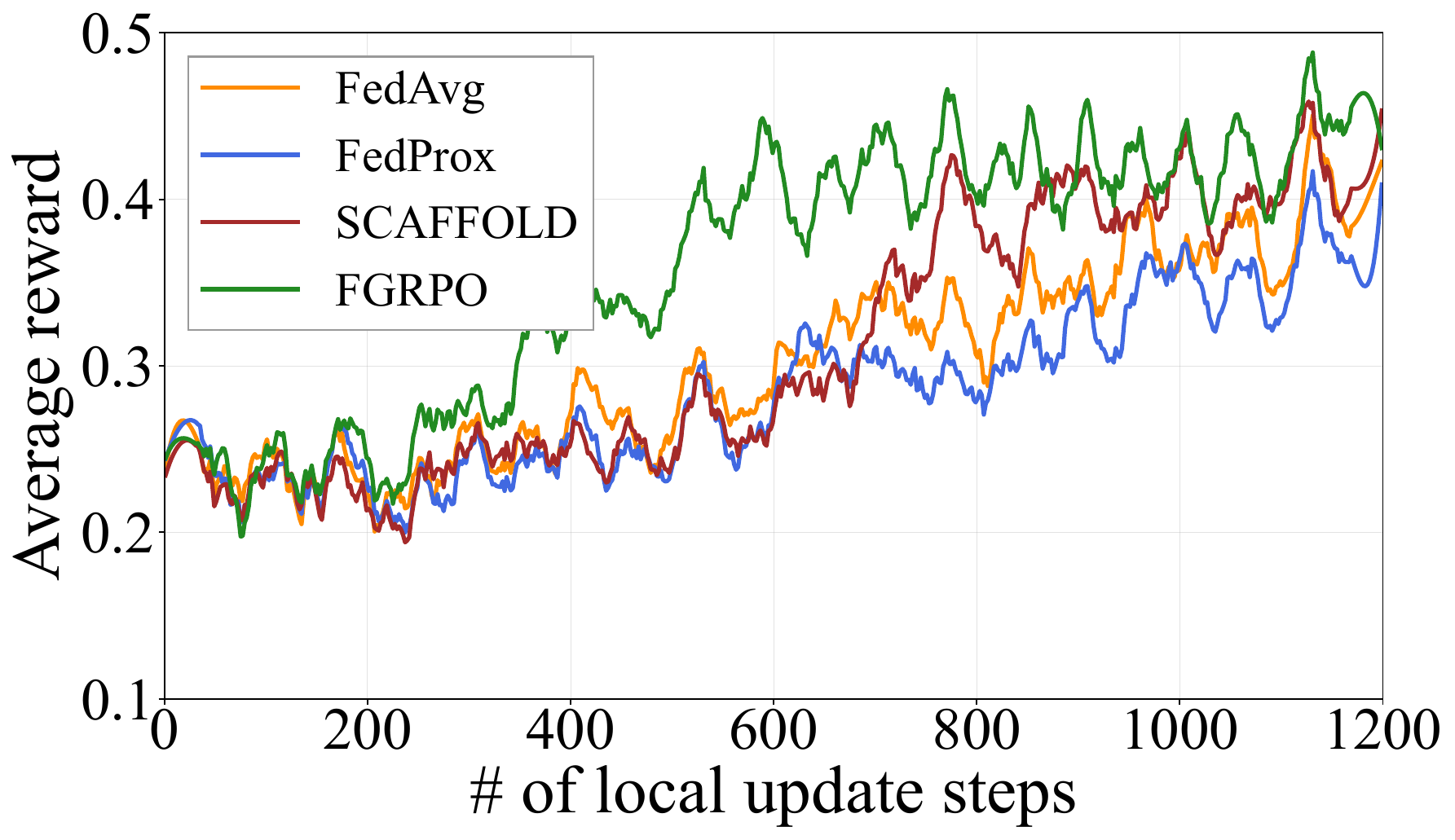}}
      \parbox{.325\textwidth}{\center\scriptsize(a) 3 clients}
      \parbox{.325\textwidth}{\center\scriptsize(b) 5 clients}
      \parbox{.325\textwidth}{\center\scriptsize(c) 10 clients}
      \caption{Average reward trajectories under varying numbers of clients on OpenR1 dataset.}
    \label{fig:reward_Client_OpenR1}
    \end{figure*}
    \begin{figure*}[t!]
    \centering
      \parbox{.325\textwidth}{\center\includegraphics[width=.325\textwidth]{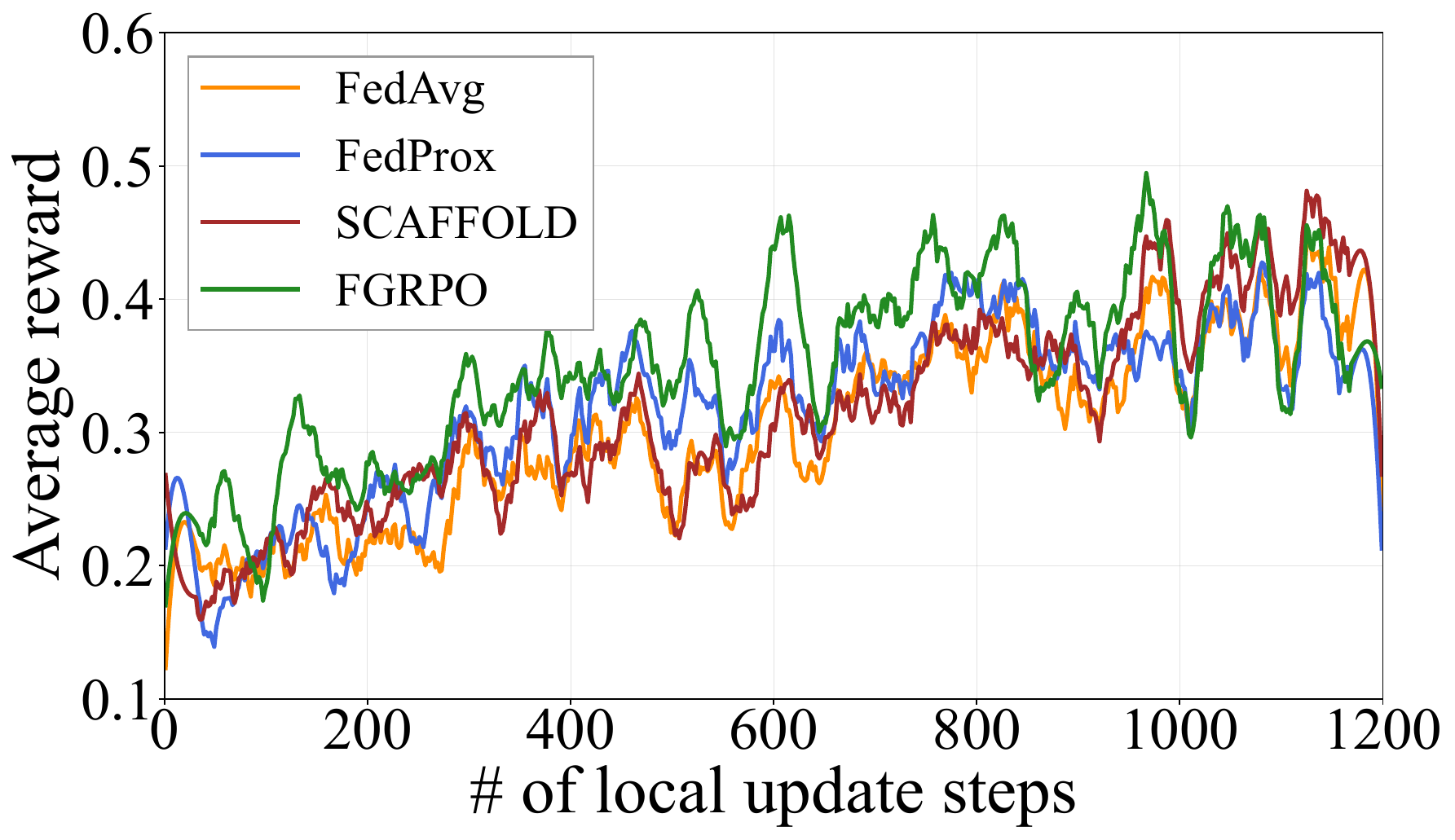}}
      \parbox{.325\textwidth}{\center\includegraphics[width=.325\textwidth]{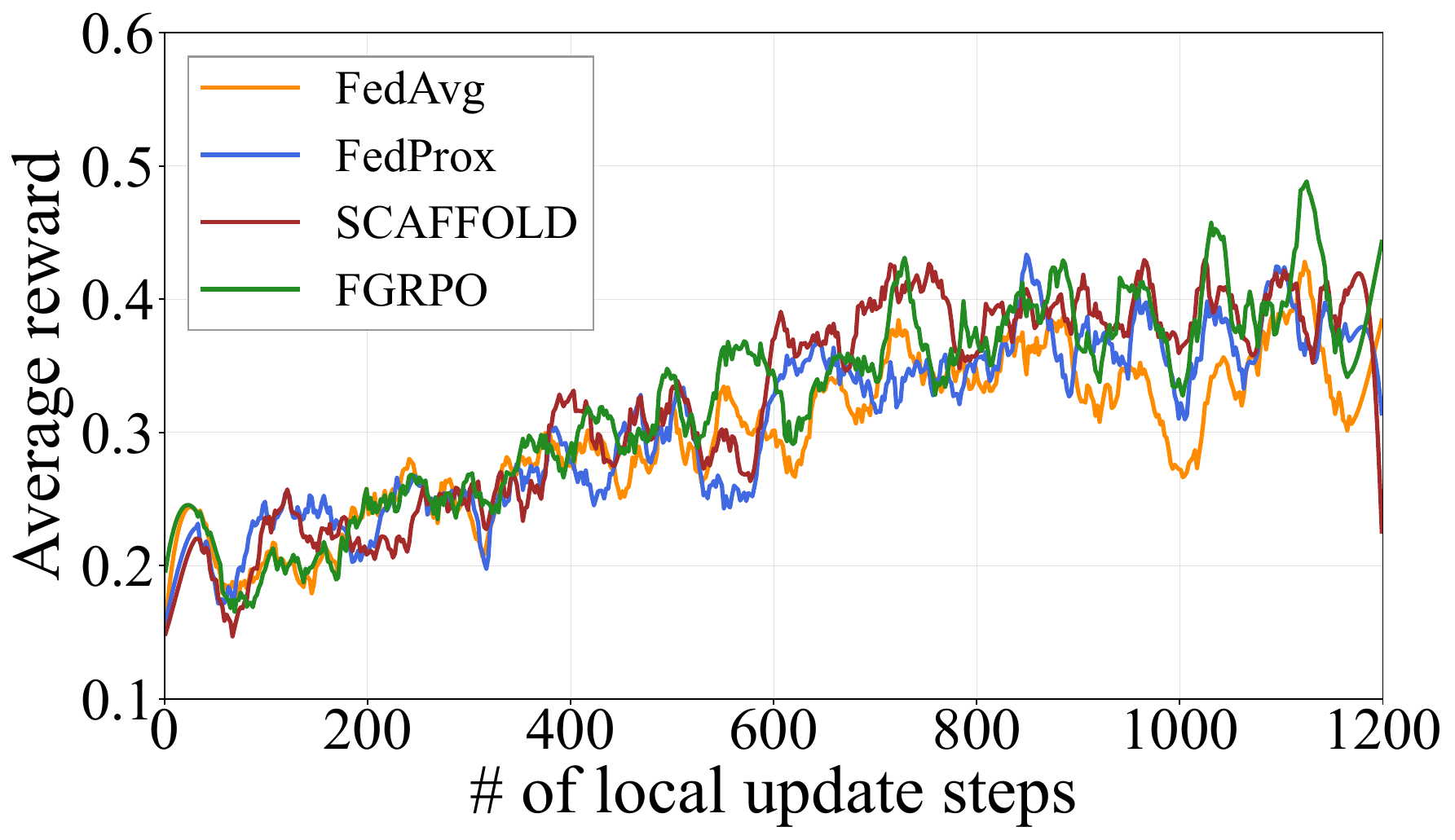}}
      \parbox{.325\textwidth}{\center\includegraphics[width=.325\textwidth]{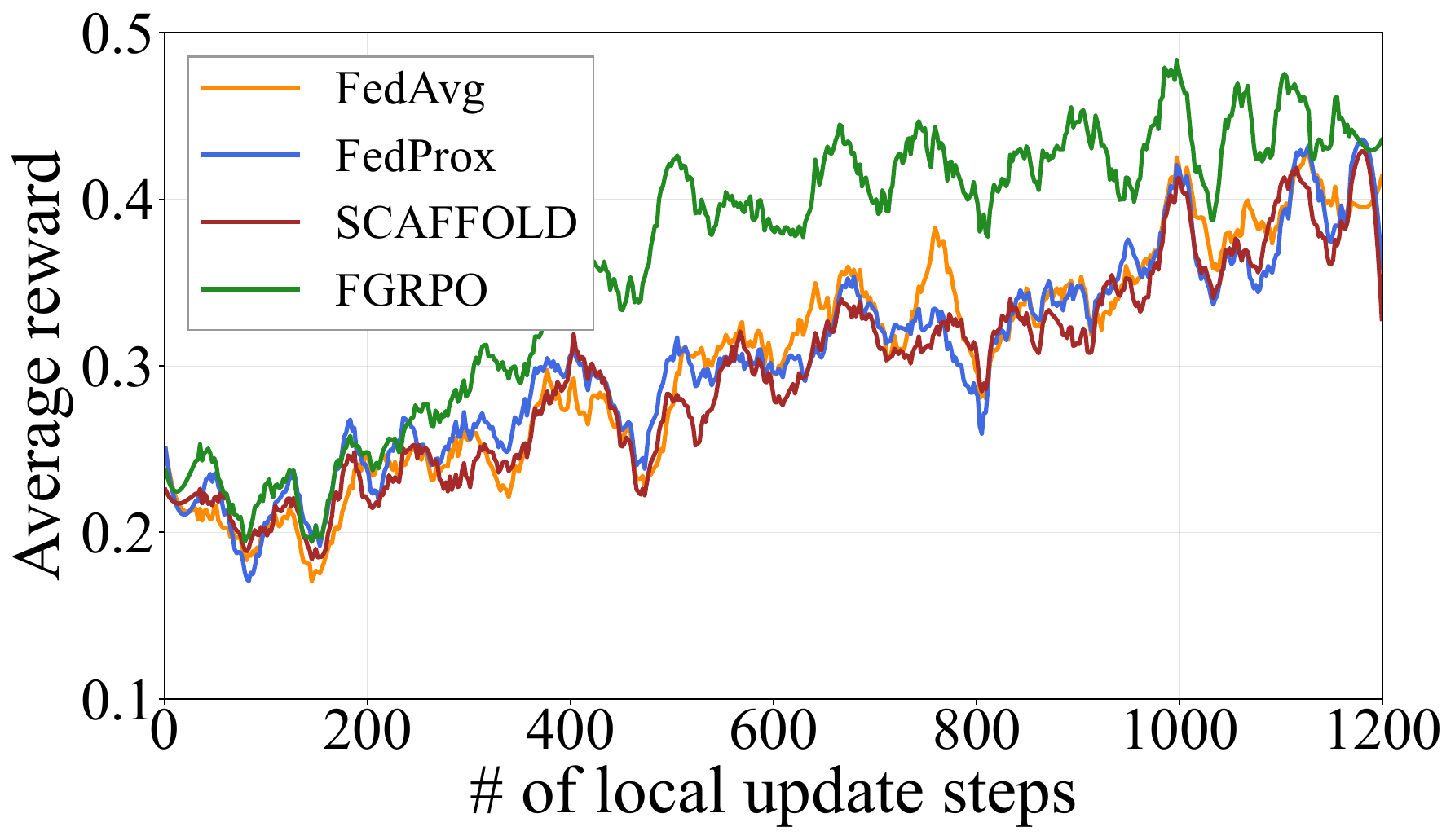}}
      \parbox{.325\textwidth}{\center\scriptsize(a) 3 clients}
      \parbox{.325\textwidth}{\center\scriptsize(b) 5 clients}
      \parbox{.325\textwidth}{\center\scriptsize(c) 10 clients}
      \caption{Average reward trajectories under varying numbers of clients on GEOQA datasets.}
    \label{fig:reward_client_GEOQA}
    \end{figure*}

  \subsection{Impact of Data Heterogeneity Levels} 
  \label{ssec:datahet}
    We further evaluate the performance of different algorithms under varying levels of data heterogeneity using the Qwen2.5-3B model. Specifically, we simulate non-IID data distributions across five clients by adjusting the Dirichlet parameter $\mu$, where smaller $\mu$ leads to more skewed client partitions, and $\mu \rightarrow \infty$ corresponds to a uniform distribution. Fig.~\ref{fig:exp_data_het} presents the performance of different methods under three representative settings: \textit{high} heterogeneity ($\mu=0.05$), \textit{moderate} heterogeneity ($\mu=1.0$), and \textit{low} heterogeneity ($\mu \rightarrow \infty$).
    The results show that FGRPO consistently achieves the best overall performance on both OpenR1 and GEOQA under all heterogeneity levels. Under the most challenging highly non-IID setting ($\mu=0.05$), FGRPO reaches 41.86\% accuracy on OpenR1, outperforming FedAvg, FedProx, and SCAFFOLD by 3.22\%, 5.74\%, and 3.02\%, respectively. A similar advantage can be observed on GEOQA, where FGRPO achieves 40.76\% accuracy and surpasses the strongest baseline by 1.48\%. These results demonstrate that FGRPO remains robust even when local clients are exposed to substantially different reasoning difficulties or visual concept distributions.
    As the data distribution becomes less heterogeneous, FGRPO continues to maintain a clear advantage over the baselines. In the medium-heterogeneity setting ($\mu=1.0$), FGRPO improves the best competing baseline by 3.22\% on OpenR1 and 2.18\% on GEOQA. In the uniform setting ($\mu=\infty$), FGRPO still achieves the highest accuracy, with gains of 2.65\% and 4.14\% over the best baseline on OpenR1 and GEOQA, respectively. This indicates that the benefit of FGRPO is not limited to extreme non-IID scenarios; instead, its adaptive aggregation mechanism also improves generalization when client distributions are relatively balanced.
    Overall, these results confirm that conventional federated optimization methods such as FedAvg, FedProx, and SCAFFOLD are insufficient to address the heterogeneous reward dynamics in the federated environment. By weighting client updates according to RPG, FGRPO can better identify effective local learning trajectories and suppress less reliable updates, leading to more stable and robust performance across different degrees of data heterogeneity.  
    \begin{figure*}[t!]
    \centering
      \parbox{.4\textwidth}{\center\includegraphics[width=.38\textwidth]{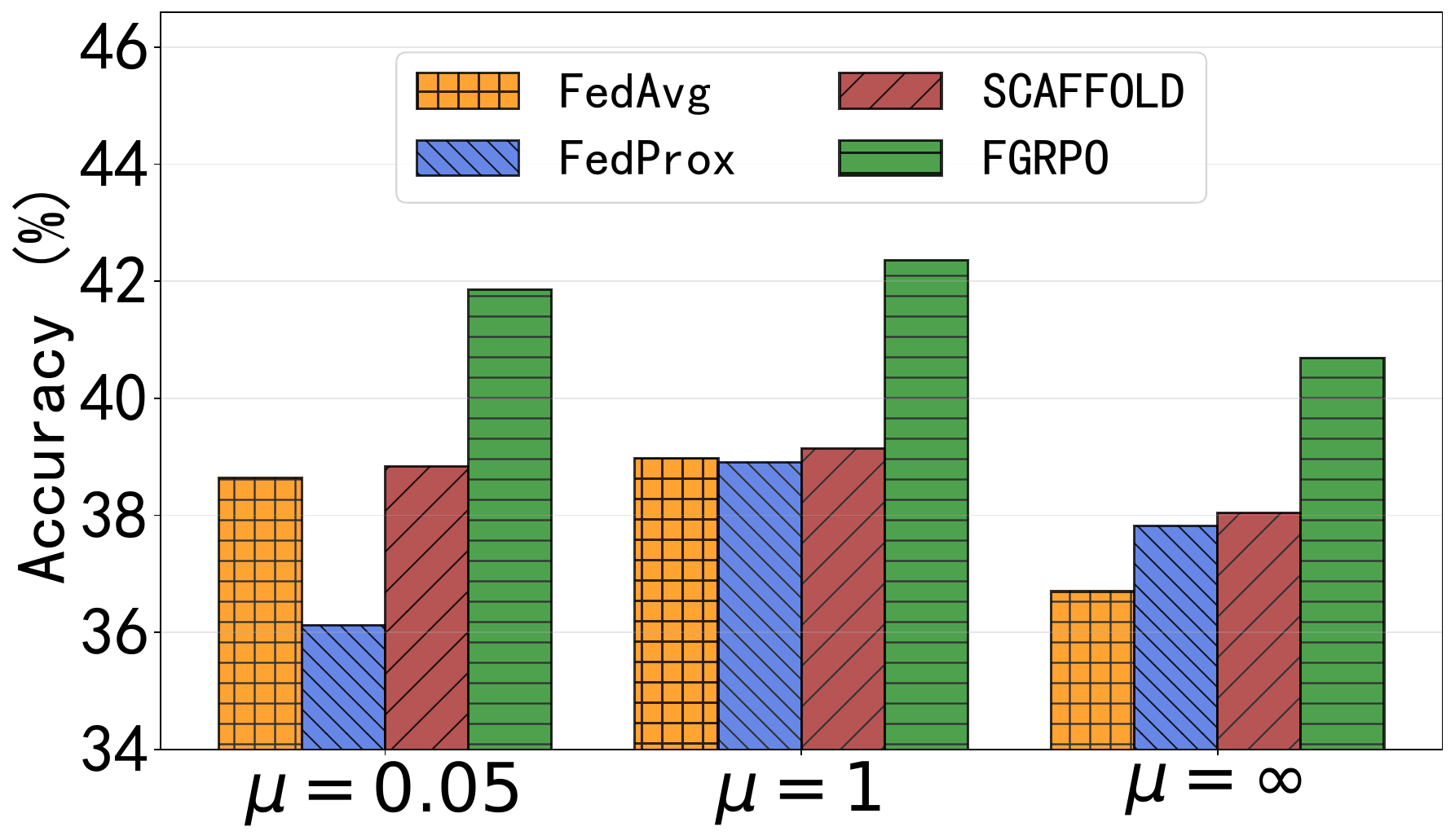}}
      \parbox{.4\textwidth}{\center\includegraphics[width=.38\textwidth]{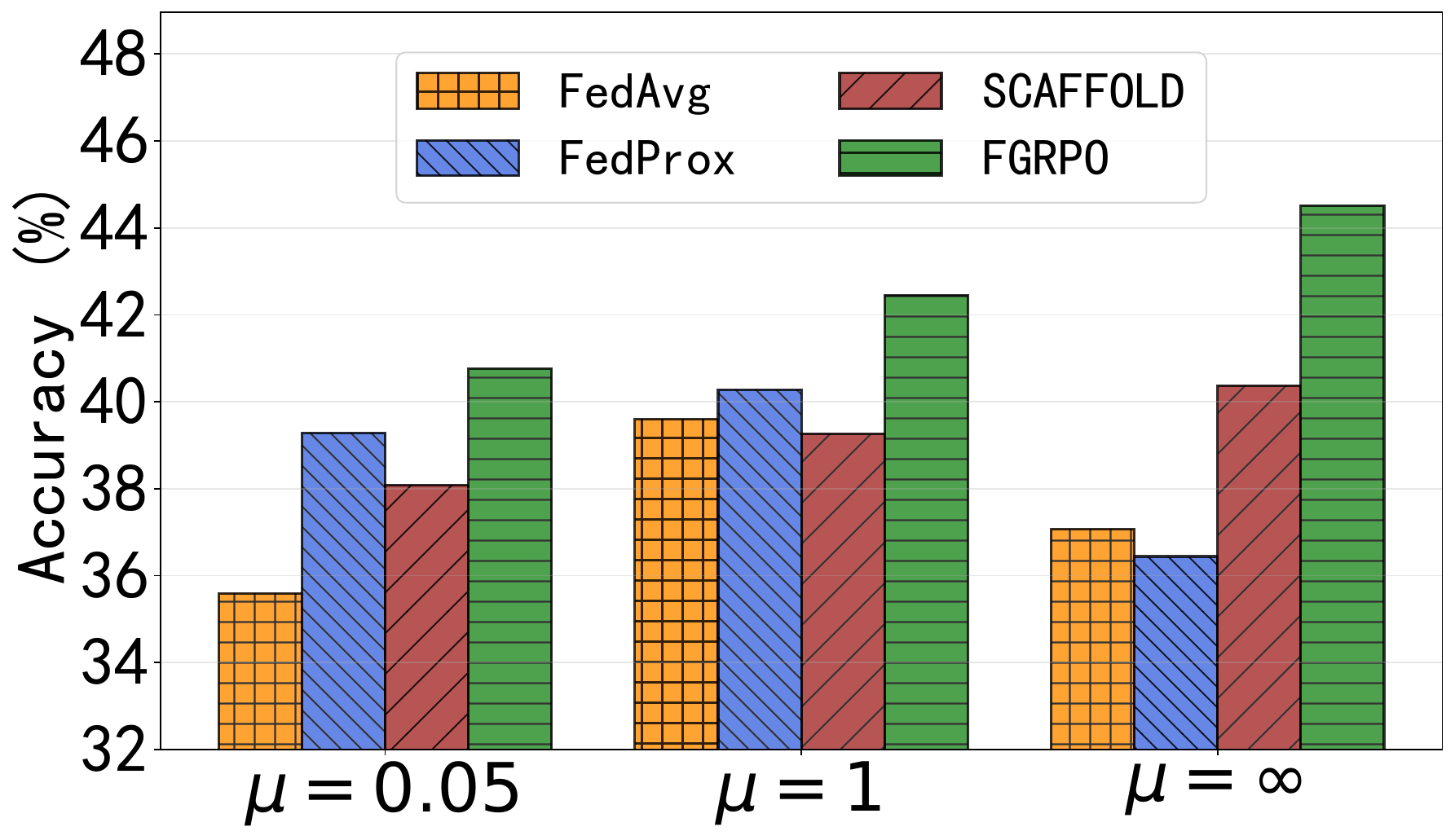}}
      \parbox{.4\textwidth}{\center\scriptsize(a) OpenR1}
      \parbox{.4\textwidth}{\center\scriptsize(b) GEOQA}
    \caption{Test accuracy under varying data heterogeneity on OpenR1 and GEOQA datasets.}
    \label{fig:exp_data_het}
    \end{figure*}

    Fig.~\ref{fig:reward_convergence_noniid_openr1} and Fig.~\ref{fig:reward_convergence_noniid_geoqa} show the average reward trajectories under different data heterogeneity settings on OpenR1 and GEOQA datasets, respectively. FGRPO consistently demonstrates superior performance, maintaining higher average reward levels. This advantage is most pronounced in the highly non-IID scenario (Fig.~\ref{fig:reward_convergence_noniid_openr1}~(a), \ref{fig:reward_convergence_noniid_geoqa}~(a)), where the highly skewed data distributions cause severe reward fluctuations for the baselines. In this challenging regime, FGRPO's trajectory remains higher than the baselines, confirming that the RPG-based aggregation effectively mitigates the optimization difficulties arising from gradient conflicts and reward-scale discrepancies. 
    %
    %
    \begin{figure*}[t!]
    \centering
      \parbox{.325\textwidth}{\center\includegraphics[width=.325\textwidth]{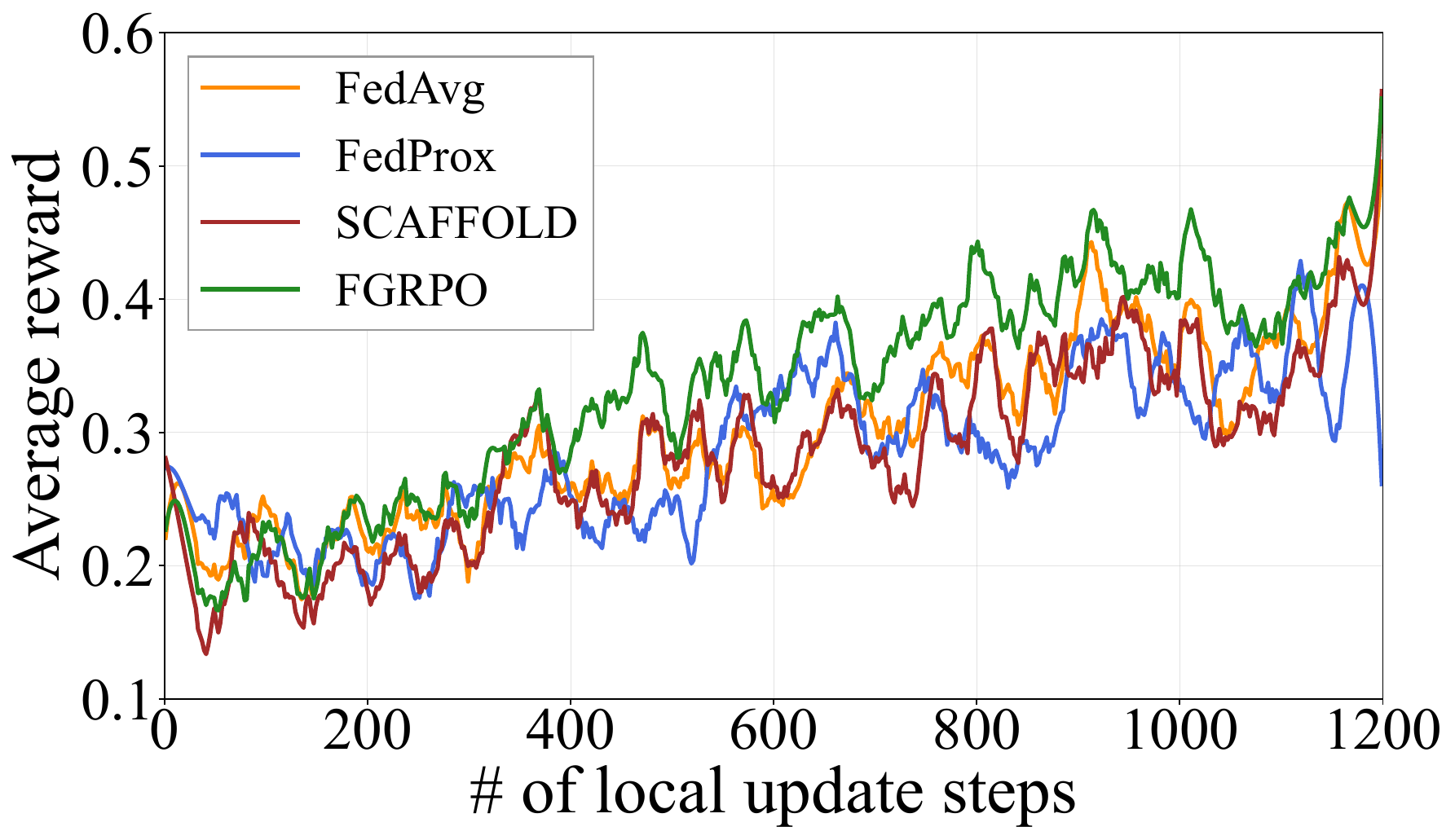}}
      \parbox{.325\textwidth}{\center\includegraphics[width=.325\textwidth]{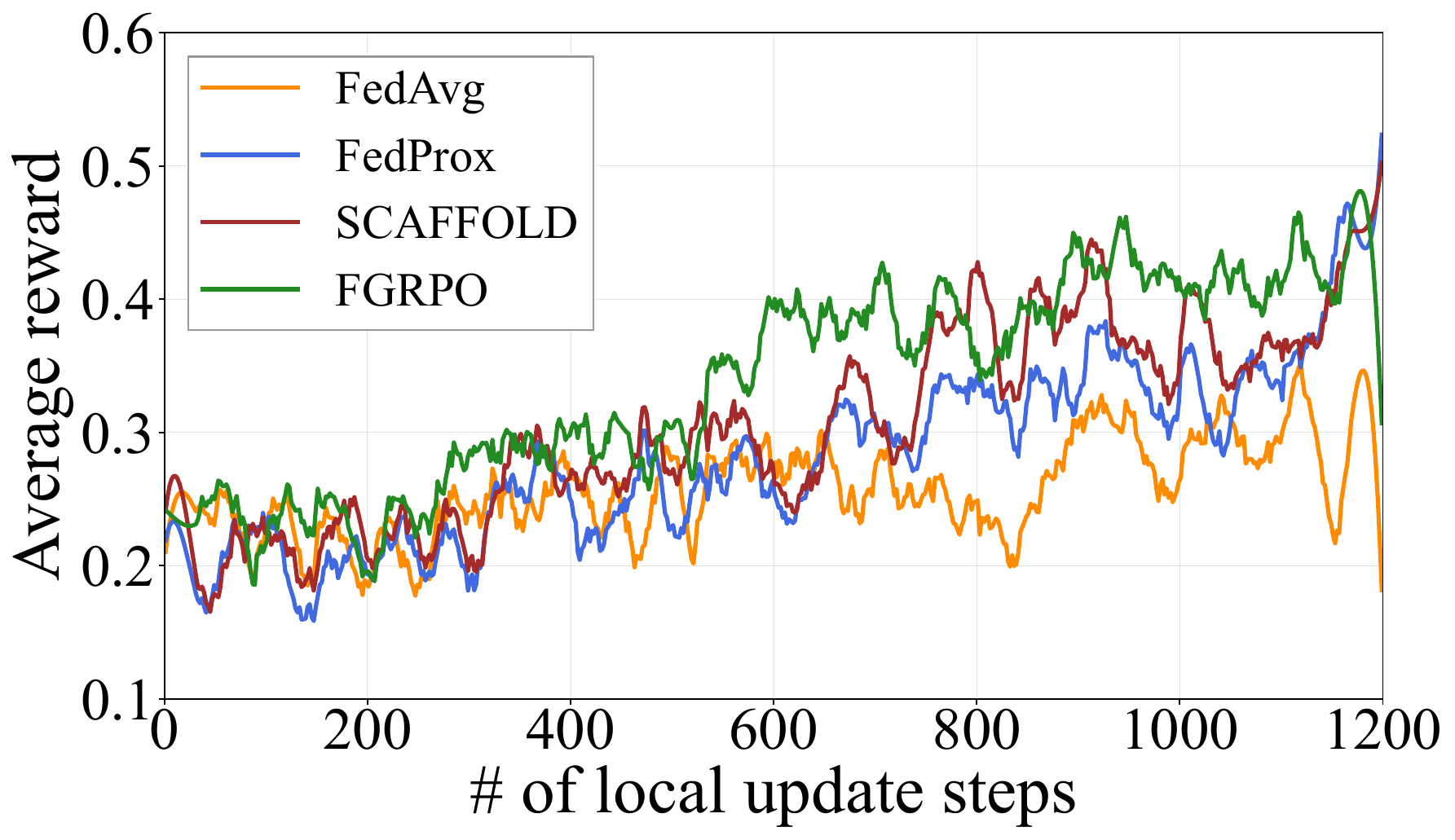}}
      \parbox{.325\textwidth}{\center\includegraphics[width=.325\textwidth]{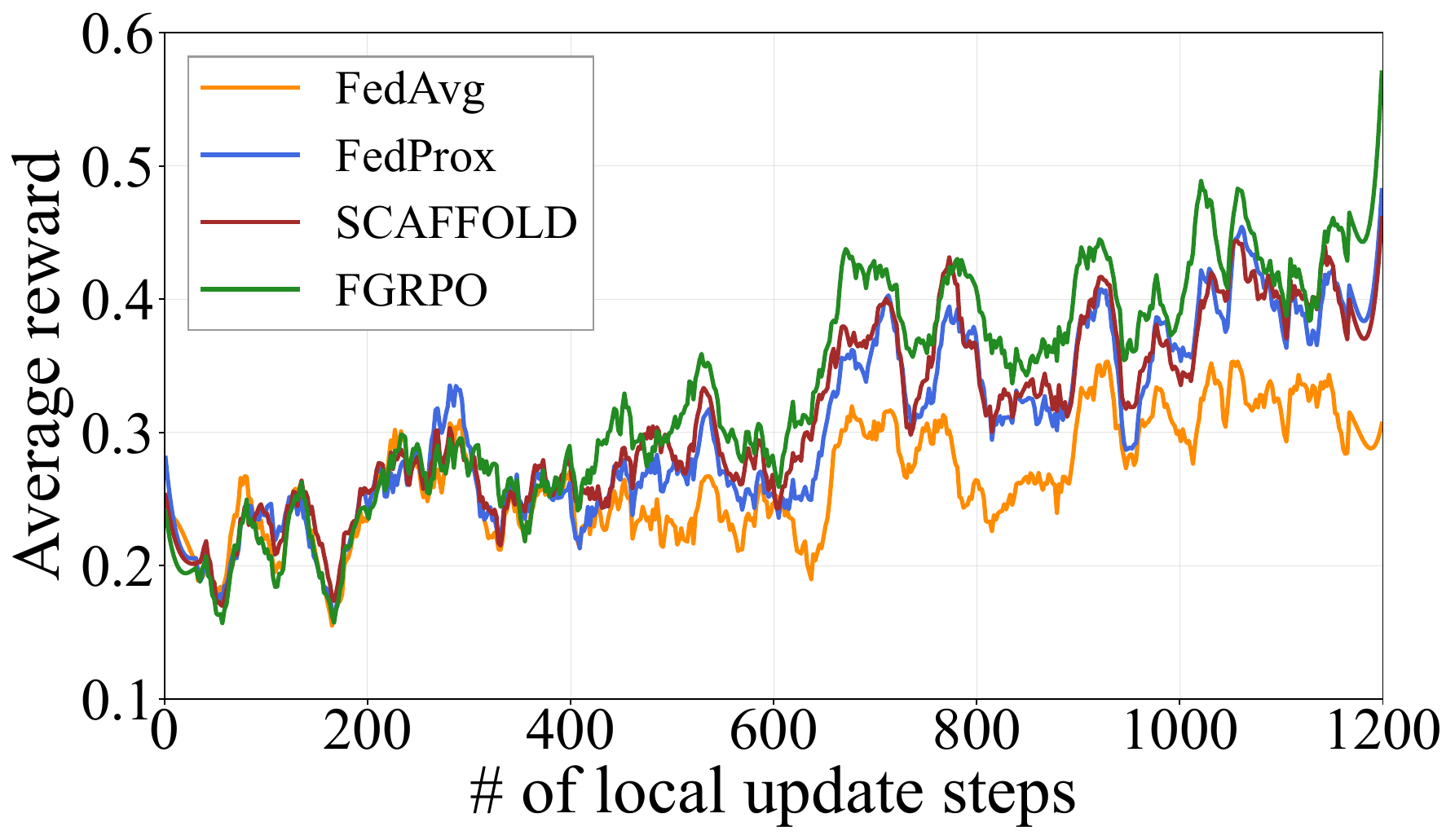}}
      \parbox{.325\textwidth}{\center\scriptsize(a) $\mu = 0.05$}
      \parbox{.325\textwidth}{\center\scriptsize(b) $\mu = 1$}
      \parbox{.325\textwidth}{\center\scriptsize(c) $\mu= \infty$}
      \caption{Average reward trajectories of different algorithms under varying non-IID data settings on OpenR1 dataset.}
    \label{fig:reward_convergence_noniid_openr1}
    \end{figure*}
    \begin{figure*}[t!]
    \centering
      \parbox{.325\textwidth}{\center\includegraphics[width=.325\textwidth]{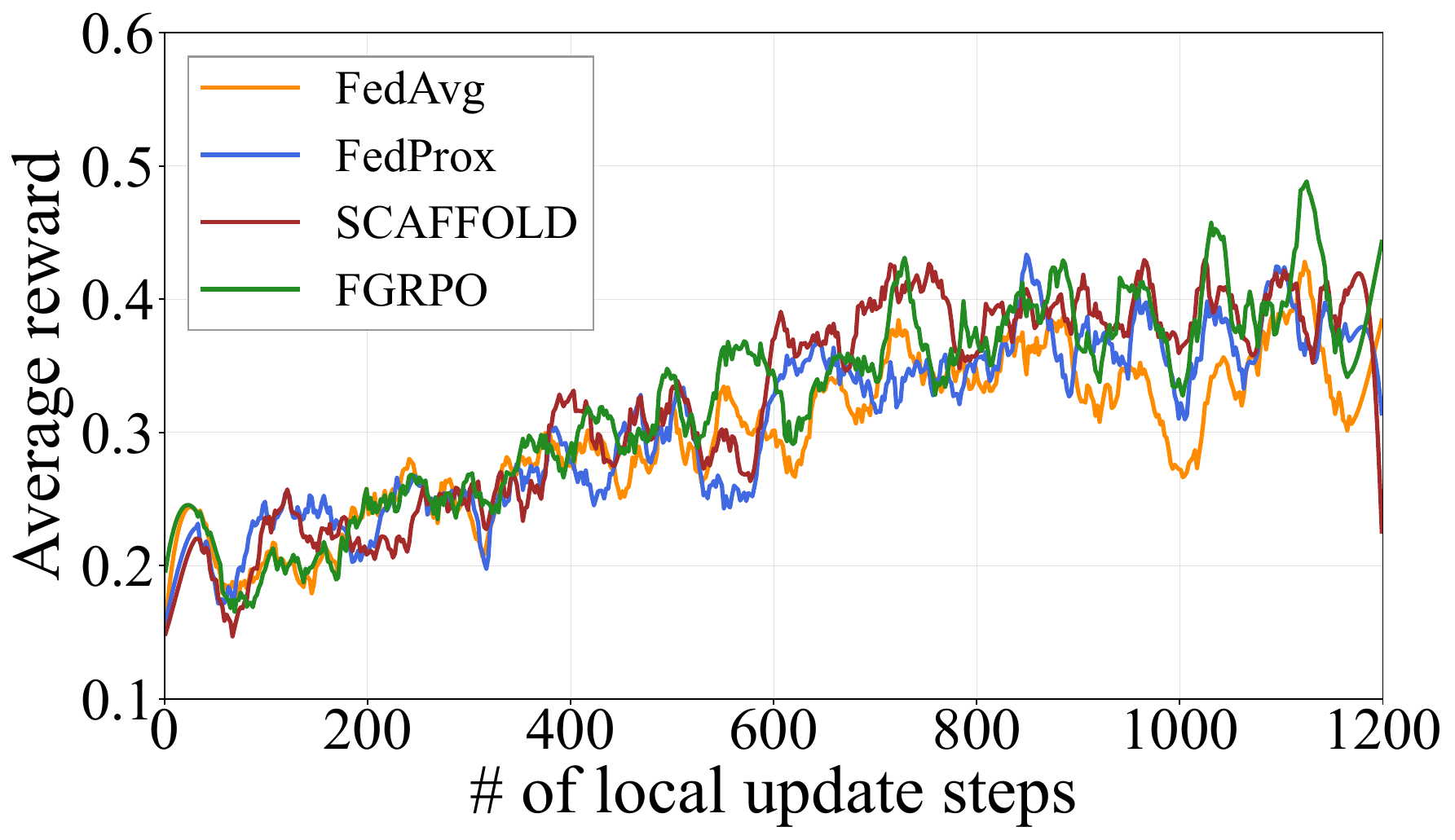}}
      \parbox{.325\textwidth}{\center\includegraphics[width=.325\textwidth]{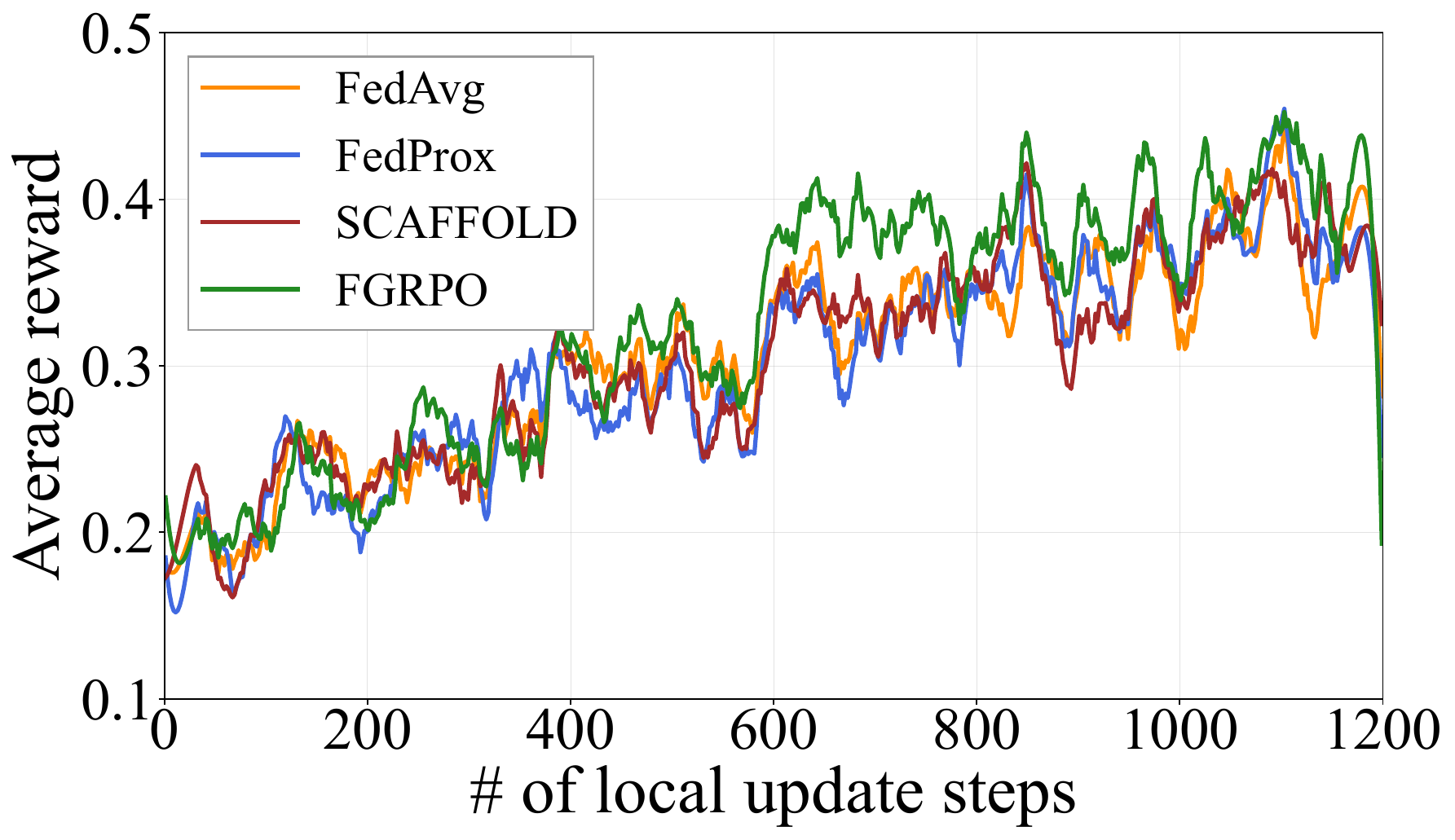}}
      \parbox{.325\textwidth}{\center\includegraphics[width=.325\textwidth]{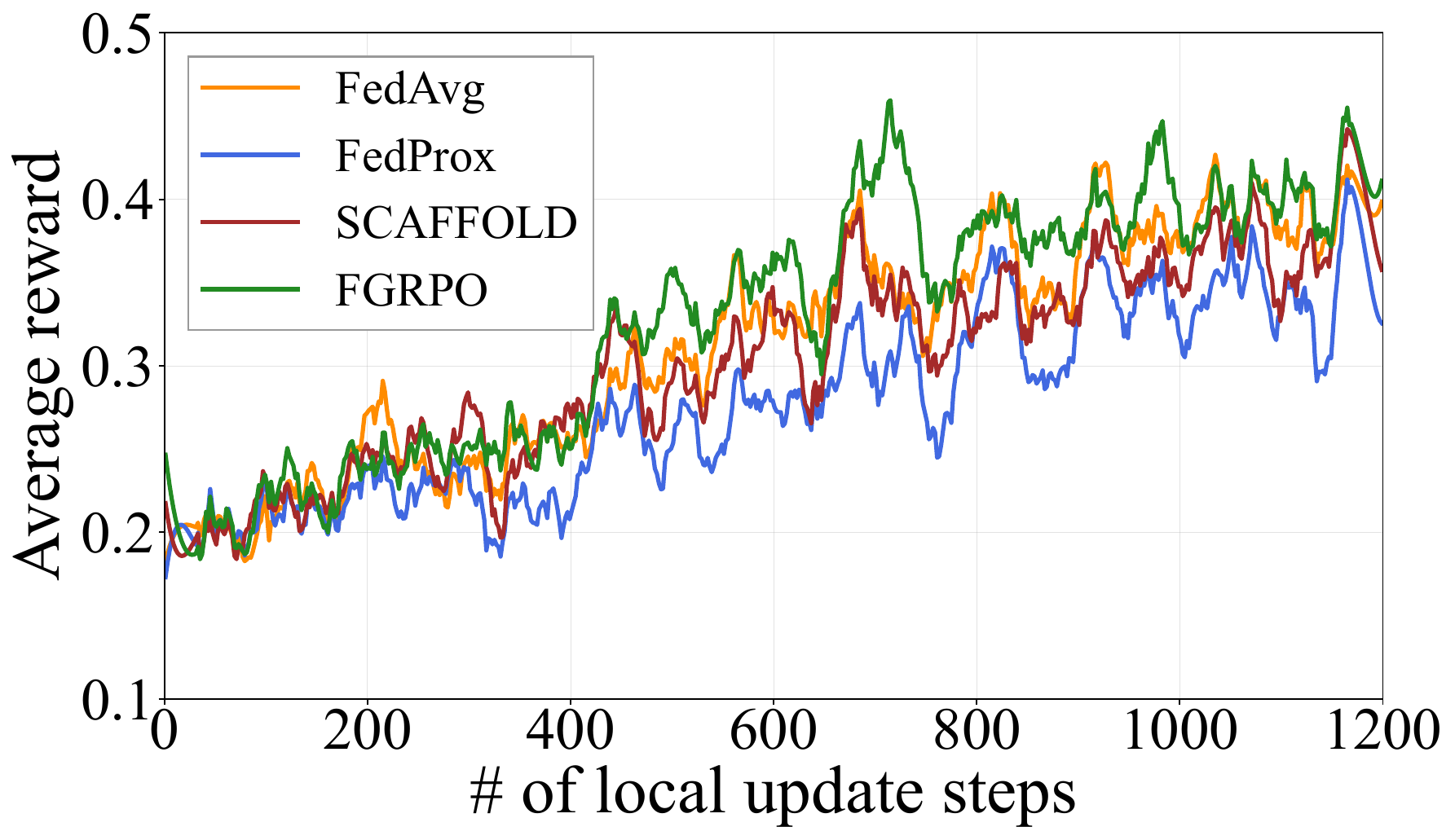}}
      \parbox{.325\textwidth}{\center\scriptsize(a) $\mu=0.05$}
      \parbox{.325\textwidth}{\center\scriptsize(b) $\mu=1$}
      \parbox{.325\textwidth}{\center\scriptsize(c) $\mu=\infty$}
      \caption{Average reward trajectories of different algorithms under varying non-IID data settings on GEOQA dataset.}
    \label{fig:reward_convergence_noniid_geoqa}
    \end{figure*}

  \subsection{Extensions to GRPO-based Variants with RPG Ablations} \label{ssec:extension}
    FGRPO is a general federated framework for GRPO-family methods under the paradigm of \emph{reinforcement learning with verifiable rewards} (RLVR), rather than being tied to a specific algorithmic variant. On the client side, it requires only a local GRPO-style update together with a round-level reward statistic, while the server performs RPG-based adaptive aggregation. This design enables recent GRPO extensions, such as DAPO~\cite{YuZZYZYDFLL-arXiv25} and GSPO~\cite{ZhengLLCYGDLMY-arXiv25}, to be seamlessly incorporated, yielding their federated counterparts, \textbf{FDAPO} and \textbf{FGSPO}. To evaluate the effect of the RPG-based adaptive weighting mechanism, we compare these federated variants (FGRPO, FDAPO, and FGSPO) with their counterparts that use standard data-volume-based aggregation instead of RPG. Specifically, we fine-tune a Qwen2.5-3B model on the Open-R1 and GEOQA datasets using each method, and report the resulting test accuracy in Table~\ref{tab:rlvr_algorithm}.
    
    The results show that federated RLVR algorithms equipped with the RPG-based adaptive weighting mechanism consistently outperform their ablated counterparts in terms of total test accuracy. For FGRPO on Open-R1, RPG-based weighting improves test accuracy from 40.66\% to 41.86\%, a gain of 1.20\%. Similarly, FDAPO achieves a 1.34\% improvement, increasing accuracy from 40.84\% to 42.18\%. FGSPO also benefits from RPG, improving from 40.00\% to 40.76\%. The improvement is more significant on GEOQA. FGRPO improves over its non-RPG variant by 2.10\%, while FDAPO achieves a larger improvement of 5.05\%, increasing the accuracy from 39.45\% to 44.50\%. FGSPO also obtains a substantial gain of 2.75\%. These results indicate that the RPG-based weighting mechanism provides consistent performance gains across different RLVR algorithms, demonstrating its general effectiveness.


    %
    \begin{table*}[t]
    \centering
    \caption{Test accuracy (\%) of applying the proposed federated framework to different RLVR algorithms on the Open-R1 and GEOQA benchmarks. Results are reported as mean$\pm$std.}
    \label{tab:rlvr_algorithm}
    \small
    \resizebox{\textwidth}{!}{
    \begin{tabular}{llccccccccc}
    \toprule
    \multicolumn{2}{c}{\textbf{Setup}} 
    & \multicolumn{4}{c}{\textbf{Open-R1}} 
    & \multicolumn{5}{c}{\textbf{GEOQA}} \\
    \cmidrule(lr){1-2} \cmidrule(lr){3-6} \cmidrule(lr){7-11}
    \textbf{Algorithm} & \textbf{Variant} 
    & \textbf{Simple} & \textbf{Medium} & \textbf{Hard} & \textbf{Total}
    & \textbf{Points} & \textbf{Lines} & \textbf{Circles} & \textbf{Polygons} & \textbf{Total} \\
    \midrule
    \multirow{2}{*}{FGRPO}
    & w/ RPG   
    & 48.11$\pm$0.96 & 44.14$\pm$3.87 & 33.35$\pm$2.19 & 41.86$\pm$1.21
    & 47.62$\pm$4.12 & 45.34$\pm$2.64 & 46.65$\pm$1.38 & 35.88$\pm$1.40 & 40.76$\pm$0.74 \\
    & w/o RPG  
    & 49.01$\pm$2.53 & 42.04$\pm$3.44 & 30.96$\pm$2.12 & 40.66$\pm$1.21
    & 43.33$\pm$5.16 & 39.25$\pm$1.98 & 47.22$\pm$2.49 & 33.13$\pm$2.06 & 38.66$\pm$1.88 \\
    \midrule
    \multirow{2}{*}{FGSPO}
    & w/ RPG   
    & 47.21$\pm$0.83 & 42.76$\pm$2.03 & 32.34$\pm$2.19 & 40.76$\pm$0.92
    & 35.24$\pm$5.43 & 38.51$\pm$2.88 & 39.38$\pm$0.61 & 34.50$\pm$1.19 & 36.55$\pm$0.84 \\
    & w/o RPG  
    & 47.21$\pm$1.25 & 41.50$\pm$3.39 & 31.32$\pm$3.44 & 40.00$\pm$1.26
    & 34.29$\pm$6.86 & 37.02$\pm$2.87 & 37.09$\pm$2.40 & 31.12$\pm$1.22 & 33.80$\pm$0.54 \\
    \midrule
    \multirow{2}{*}{FDAPO}
    & w/ RPG   
    & 44.68$\pm$1.33 & 43.12$\pm$1.63 & 38.74$\pm$1.58 & 42.18$\pm$0.24
    & 46.67$\pm$8.68 & 52.67$\pm$4.70 & 50.35$\pm$1.09 & 39.13$\pm$0.92 & 44.50$\pm$0.96 \\
    & w/o RPG  
    & 43.84$\pm$1.20 & 42.10$\pm$1.89 & 36.59$\pm$2.25 & 40.84$\pm$0.99
    & 36.19$\pm$6.82 & 42.73$\pm$6.88 & 47.14$\pm$1.73 & 34.32$\pm$1.50 & 39.45$\pm$1.82 \\
    \bottomrule
    \end{tabular}
    }
    \end{table*}

    Fig.\ref{fig:reward_convergence_rlvr_OPENR1} and Fig.\ref{fig:reward_convergence_rlvr_geoqa} illustrate the reward trajectories of federated RLVR algorithms with and without our RPG-based weighting mechanism. Across GRPO, GSPO, and DAPO, the RPG-based variants consistently exhibit more stable and favorable reward progression, aligning with the final accuracy improvements reported in Table~\ref{tab:rlvr_algorithm}. These results imply that leveraging the concept of RPG can more effectively capture meaningful client-side progress and mitigate the impact of heterogeneous reward scales. These results demonstrate the generality of RPG across different RLVR algorithms. Although the underlying objectives lead to different absolute performance levels, incorporating RPG consistently improves accuracy on both Open-R1 and GEOQA, highlighting its role as a general reward-aware aggregation mechanism for federated RLVR.
    \begin{figure*}[t!]
    \centering
      \parbox{.325\textwidth}{\center\includegraphics[width=.325\textwidth]{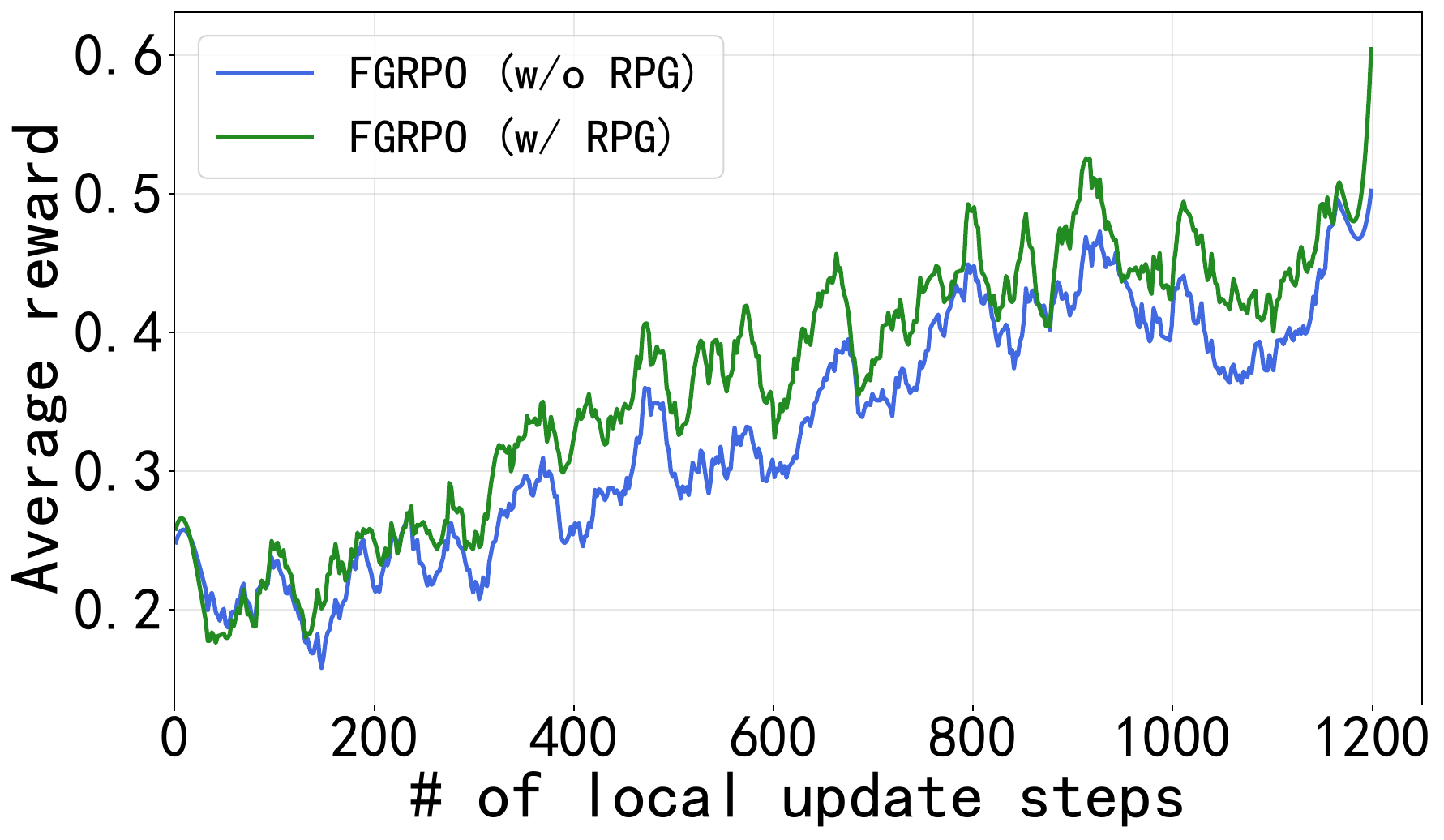}}
      \parbox{.325\textwidth}{\center\includegraphics[width=.325\textwidth]{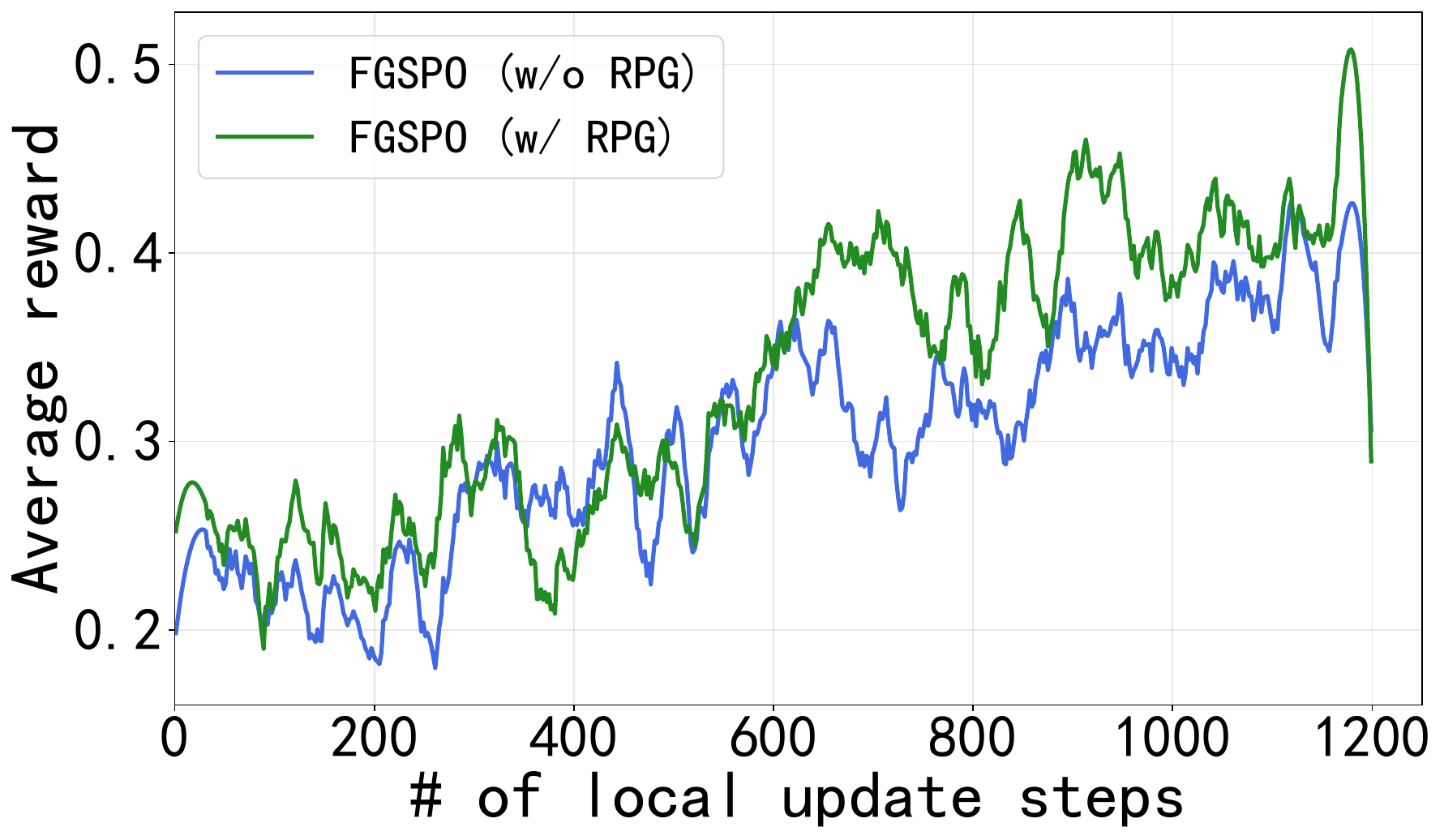}}
      \parbox{.325\textwidth}{\center\includegraphics[width=.325\textwidth]{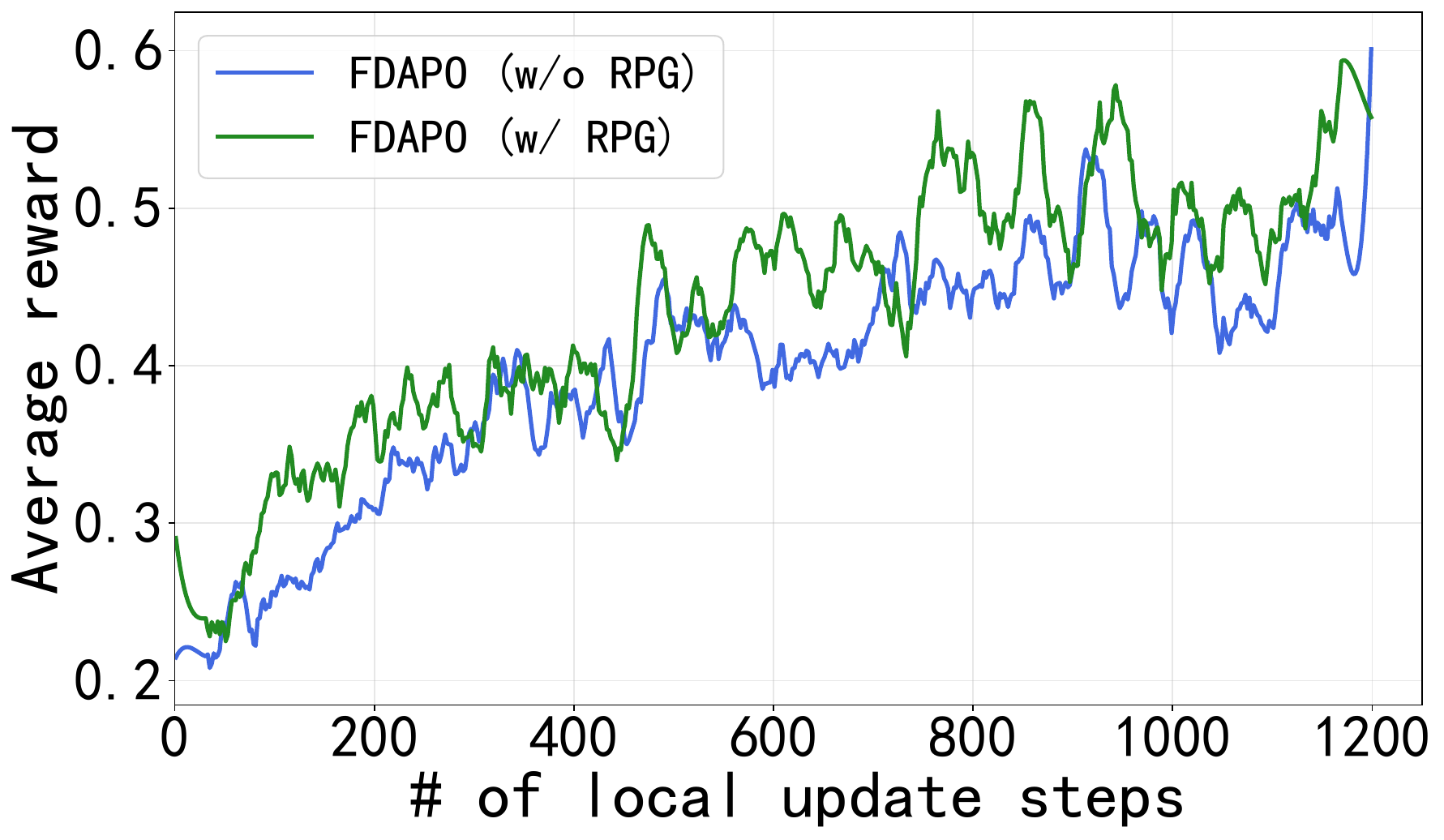}}
      \parbox{.325\textwidth}{\center\scriptsize(a) FGRPO}
      \parbox{.325\textwidth}{\center\scriptsize(b) FGSPO}
      \parbox{.325\textwidth}{\center\scriptsize(c) FDAPO}
      \caption{Average reward trajectories of different federated RLVR algorithms on OpenR1 dataset.}
    \label{fig:reward_convergence_rlvr_OPENR1}
    \end{figure*}
    \begin{figure*}[t!]
    \centering
      \parbox{.325\textwidth}{\center\includegraphics[width=.325\textwidth]{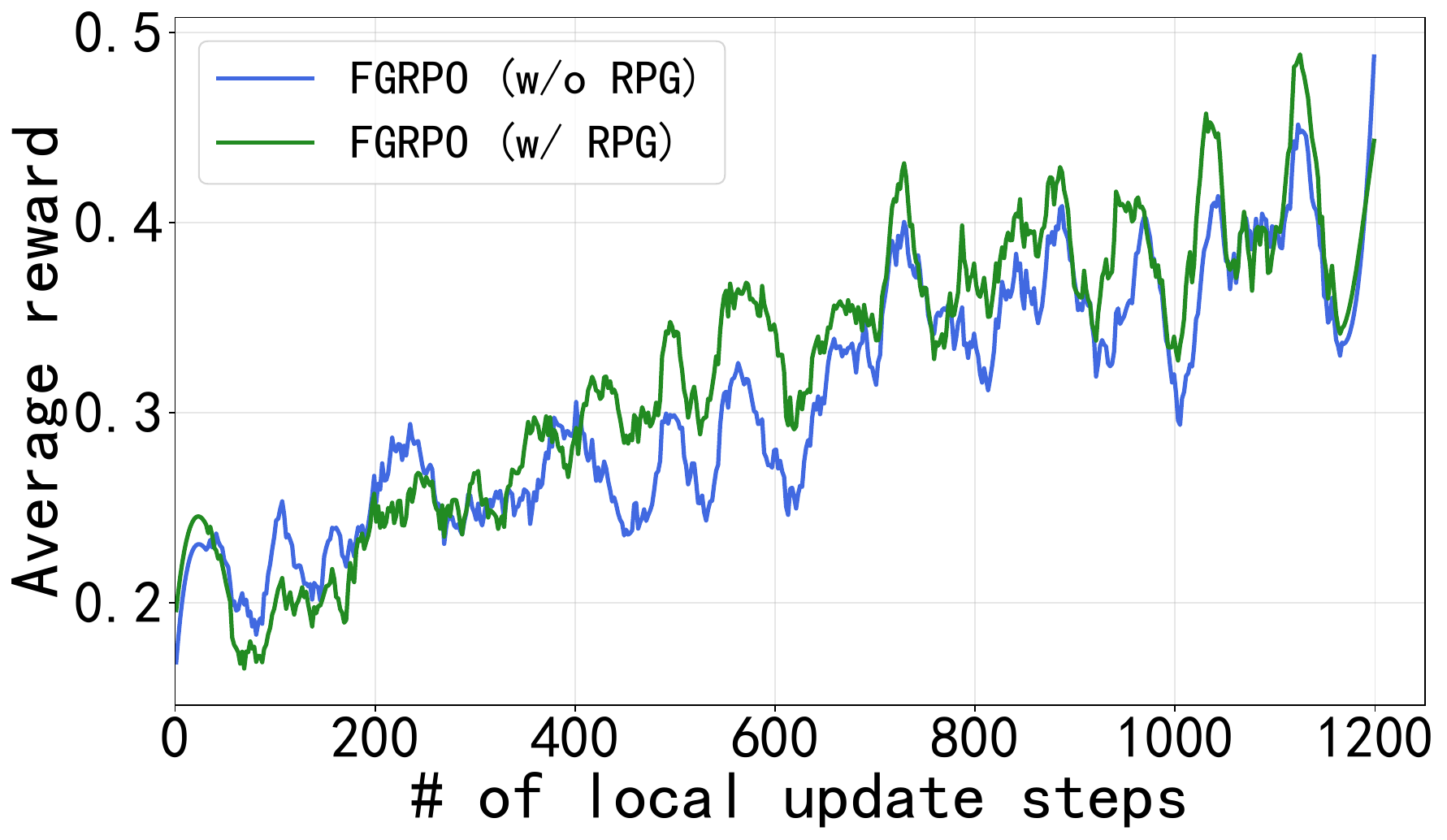}}
      \parbox{.325\textwidth}{\center\includegraphics[width=.325\textwidth]{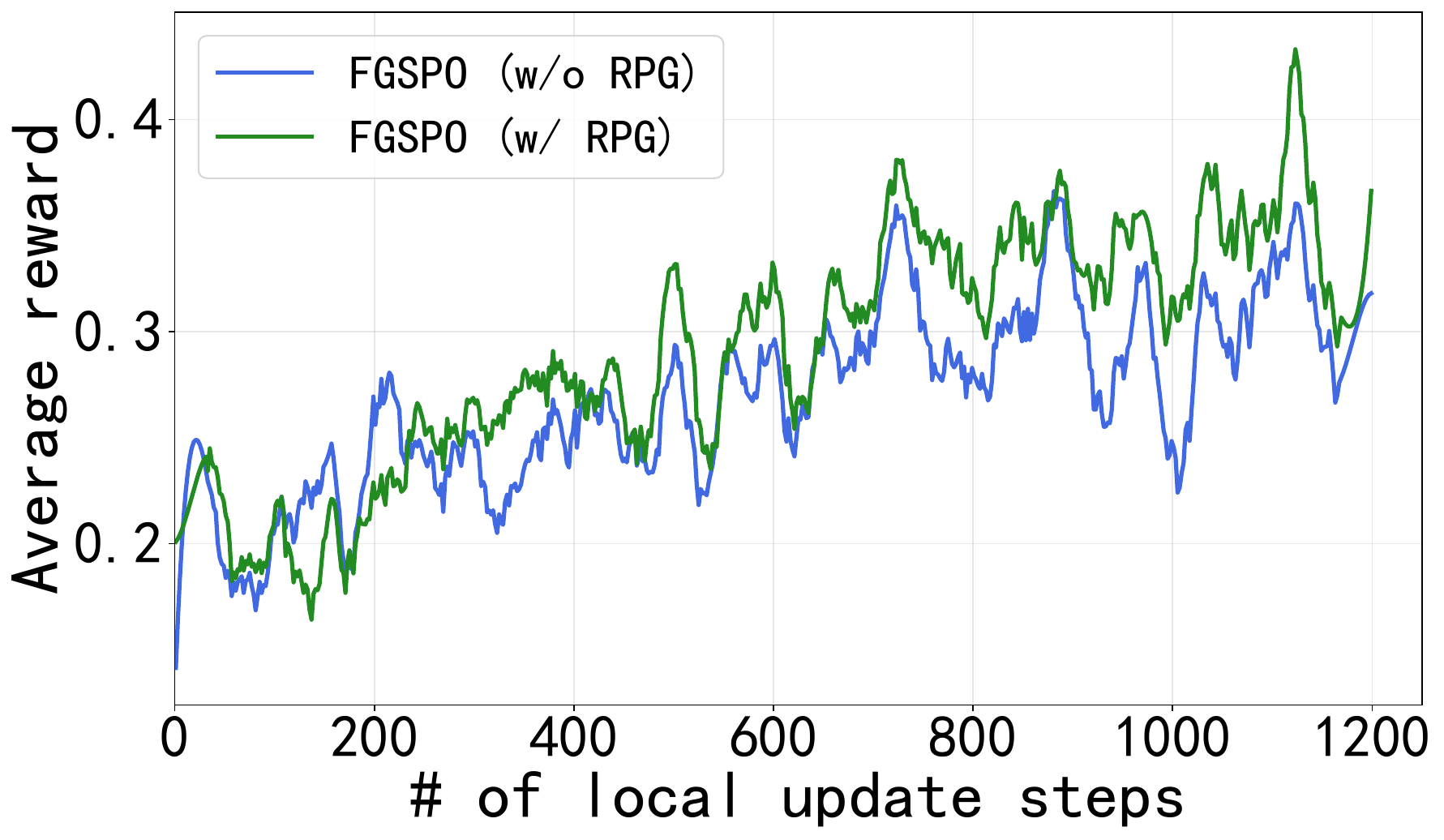}}
      \parbox{.325\textwidth}{\center\includegraphics[width=.325\textwidth]{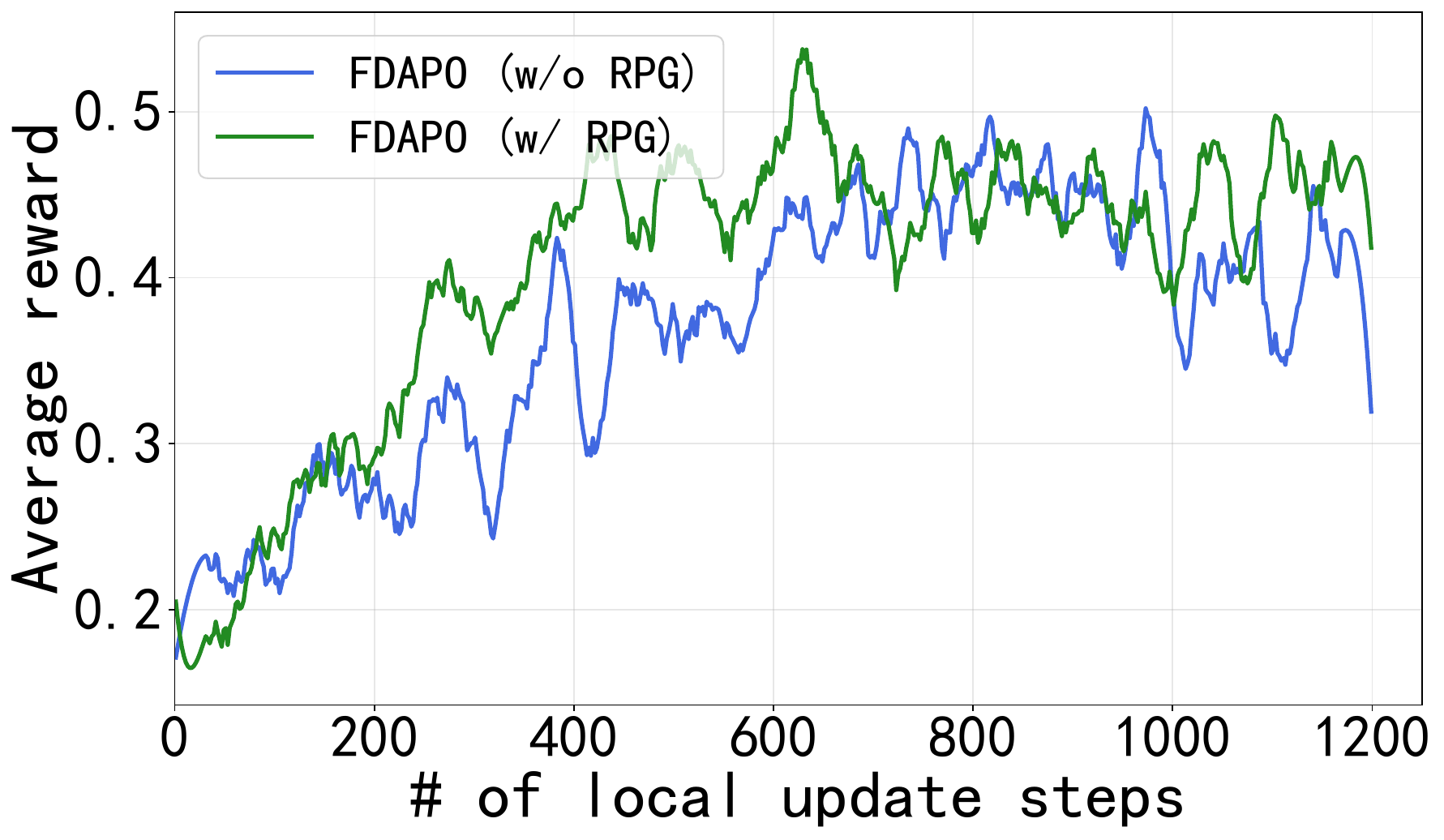}}
      \parbox{.325\textwidth}{\center\scriptsize(a) FGRPO}
      \parbox{.325\textwidth}{\center\scriptsize(b) FGSPO}
      \parbox{.325\textwidth}{\center\scriptsize(c) FDAPO}
      \caption{Average reward trajectories of different federated RLVR algorithms on GEOQA dataset.}
    \label{fig:reward_convergence_rlvr_geoqa}
    \end{figure*}

  \subsection{Hyperparameter Sensitivity Analysis}
  \label{ssec:sensitivity}
    We conduct a hyperparameter sensitivity analysis on the Qwen2.5-3B model using the Open-R1 dataset. Our study focuses on six key hyperparameters in FGRPO: the EMA coefficient $\lambda_{\rm base}$, the temperature bounds $(\tau_{\min}, \tau_{\max})$, the volatility bounds $(\sigma_{\min}, \sigma_{\max})$, and the annealing coefficient $\lambda_{\rm anneal}$. Unless otherwise specified, we adopt the default setting $\lambda_{\rm base}=0.8$, $\tau_{\min}=1.5$, $\tau_{\max}=2.5$, $\sigma_{\min}=0.05$, $\sigma_{\max}=0.2$, and $\lambda_{\rm anneal}=0.1$.

    The coefficient $\lambda_{\rm base}$ controls the exponential moving average used to smooth the estimation of client-side progress. A larger \(\lambda_{\rm base}\) makes the baseline more sensitive to recent reward changes, while a smaller \(\lambda_{\rm base}\) places more emphasis on historical estimates. As shown in Table~\ref{tab:lambda_sensitivity}, the default value $\lambda_{\rm base}=0.8$ achieves the best total accuracy of 41.86\%. When $\lambda_{\rm base}$ is reduced to 0.6 or increased to 0.9, the total accuracy becomes 40.98\% and 40.82\%, respectively. Although $\lambda_{\rm base}=0.6$ slightly improves the Hard split, it reduces the Simple and Total performance. This indicates that a moderate EMA coefficient provides a better balance between responsiveness and stability.

    The temperature bounds $(\tau_{\min}, \tau_{\max})$ control the range of adaptive aggregation sharpness. A lower temperature makes the aggregation more concentrated on clients with larger relative gains, while a higher temperature produces smoother and more uniform weighting.  Table~\ref{tab:tau_sensitivity} shows that FGRPO is relatively robust to different temperature ranges. The default setting $(\tau_{\min}, \tau_{\max})=(1.5,2.5)$ obtains the best total accuracy of 41.86\%.  Changing the upper bound to 2.0 or 3.0 only slightly decreases the total accuracy to 41.48\% and 41.58\%, respectively. Interestingly, increasing $\tau_{\max}$ to 3.0 improves the Hard split to 35.33\%, but reduces Simple and Medium accuracy, suggesting a trade-off between difficult-example optimization and overall balanced performance. Similarly, varying $\tau_{\min}$ also leads to only moderate performance changes, with total accuracy remaining above 41.12\%.

    The volatility bounds $(\sigma_{\min}, \sigma_{\max})$ define the clipping range for reward-progress volatility estimation. These bounds prevent the aggregation weights from becoming overly sensitive to unstable reward fluctuations. As shown in Table~\ref{tab:sigma_sensitivity}, the default setting $(\sigma_{\min}, \sigma_{\max})=(0.05,0.2)$ achieves the best total accuracy of 41.86\%. Reducing $\sigma_{\min}$ to 0.03 or increasing it to 0.10 decreases the total accuracy to 40.86\% and 40.54\%, respectively. Similarly, changing $\sigma_{\max}$ to 0.15 or 0.25 results in total accuracies of 41.04\% and 40.90\%. These results suggest that properly bounding the volatility estimate is important for stabilizing RPG-based aggregation, especially under non-IID data.

    Finally, $\lambda_{\rm anneal}$ controls the annealing strength of the adaptive aggregation process. It determines how quickly the aggregation behavior changes during training. As shown in Table~\ref{tab:lambda_anneal_sensitivity}, the default value $\lambda_{\rm anneal}=0.1$ achieves the best total accuracy of 41.86\%. Reducing it to 0.05 leads to a total accuracy of 41.24\%, while increasing it to 0.15 decreases the total accuracy to 40.76\%. Although $\lambda_{\rm anneal}=0.15$ slightly improves the Medium split, it significantly hurts the Hard split, indicating that overly aggressive annealing may destabilize optimization for difficult reasoning examples.

    Overall, the sensitivity results show that FGRPO is reasonably robust within a practical range of hyperparameter choices. The total accuracy remains around 40.5\%--41.9\% across different settings, and the default configuration consistently achieves the best overall performance. These results also suggest that the proposed RPG-based aggregation does not rely on a narrowly tuned hyperparameter configuration, while moderate smoothing, bounded volatility estimation, and stable annealing are beneficial for federated reasoning optimization.
 
    \begin{table*}[t]
    \centering
    \caption{Sensitivity analysis of $\lambda_{\rm base}$ with other hyperparameters fixed to $\tau_{\min}=1.5$, $\tau_{\max} =2.5$, $\sigma_{\min}=0.05$, $\sigma_{\max}=0.2$, and $\lambda_{\rm anneal}=0.1$.}
    \label{tab:lambda_sensitivity}
    \begin{tabular}{lcccc}
    \toprule
    Hyperparameter & Simple & Medium & Hard & Total \\
    \midrule
    $\lambda_{\rm base} = 0.8$ & 48.11 & 44.14 & 33.35 & 41.86 \\
    $\lambda_{\rm base} = 0.6$ &  45.05 & 43.84 & 34.07& 40.98 \\
    $\lambda_{\rm base} = 0.9$ & 46.79  & 42.88 & 32.81 & 40.82 \\
    \bottomrule
    \end{tabular}
    \end{table*}
    \begin{table*}[t]
    \centering
    \caption{Sensitivity analysis of $\tau_{\min}$ and $\tau_{\max}$ with other hyperparameters fixed to $\lambda_{\rm base}=0.8$, $\sigma_{\min}=0.05$, $\sigma_{\max}=0.2$, and $\lambda_{\rm anneal}=0.1$.}
    \label{tab:tau_sensitivity}
    \begin{tabular}{lcccc}
    \toprule
    Hyperparameter & Simple & Medium & Hard & Total \\
    \midrule
    $\tau_{\min}=1.5,\ \tau_{\max}=2.5$ & 48.11 & 44.14 & 33.35 & 41.86  \\
    $\tau_{\min}=1.5,\ \tau_{\max}=2.0$ & 49.85 & 43.55 & 31.06 & 41.48 \\
    $\tau_{\min}=1.5,\ \tau_{\max}=3.0$ &  46.85 & 42.58 & 35.33 & 41.58 \\
    $\tau_{\min}=1.0,\ \tau_{\max}=2.5$ & 48.23 & 41.08 & 34.43  & 41.24 \\
    $\tau_{\min}=2.0,\ \tau_{\max}=2.5$ & 48.41 & 43.12 & 31.86 & 41.12  \\
    \bottomrule
    \end{tabular}
    \end{table*}
    \begin{table*}[t!]
    \centering
    \caption{Sensitivity analysis of $\sigma_{\min}$ and $\sigma_{\max}$ with other hyperparameters fixed to $\lambda_{\rm base}=0.8$, $\tau_{\min}=1.5$, $\tau_{\max} =2.5$, and $\lambda_{\rm anneal}=0.1$.}
    \label{tab:sigma_sensitivity}
    \begin{tabular}{lcccc}
    \toprule
    Hyperparameter & Simple & Medium & Hard & Total \\
    \midrule
    $\sigma_{\min}=0.05,\ \sigma_{\max}=0.2$ & 48.11 & 44.14 & 33.35 & 41.86  \\
    $\sigma_{\min}=0.03,\ \sigma_{\max}=0.2$ & 49.55 & 42.70 & 30.36 & 40.86 \\
    $\sigma_{\min}=0.10,\ \sigma_{\max}=0.2$ & 48.35 &  41.86 &  31.44 & 40.54 \\
    $\sigma_{\min}=0.05,\ \sigma_{\max}=0.15$ &  46.97 & 43.84 &  32.34 &  41.04  \\
    $\sigma_{\min}=0.05,\ \sigma_{\max}=0.25$ & 48.11 &  42.94 &  31.68 &  40.90  \\
    \bottomrule
    \end{tabular}
    \end{table*}
    \begin{table*}[t]
    \centering
    \caption{Sensitivity analysis of $\lambda_{\rm anneal}$ with other hyperparameters fixed to  $\lambda_{\rm base}=0.8$, $\tau_{\min}=1.5$, $\tau_{\max} =2.5$, $\sigma_{\min}=0.05$, and $\sigma_{\max}=0.2$.}
    \label{tab:lambda_anneal_sensitivity}
    \begin{tabular}{lcccc}
    \toprule
    Hyperparameter & Simple & Medium & Hard & Total \\
    \midrule
    $\lambda_{\rm anneal} = 0.1$ & 48.11 & 44.14 & 33.35 & 41.86  \\
    $\lambda_{\rm anneal} = 0.05$ & 49.73 & 40.96  & 33.05 & 41.24 \\
    $\lambda_{\rm anneal} = 0.15$ & 48.77 & 44.68 &  28.86 & 40.76 \\
    \bottomrule
    \end{tabular}
    \end{table*}

  \subsection{Resource Consumption} \label{ssec:resconsumption}
    As discussed in Appendix~\ref{sec:app_lora}, the communication overhead per round is mainly determined by the size of the LoRA low-rank matrices. Since all methods use the same LoRA configuration, they incur comparable communication costs. Therefore, we focus on evaluating computational cost by analyzing GPU utilization and GPU memory utilization. Fig.~\ref{fig:cost_analysis} reports these metrics, averaged across the five clients per communication round. 
    %
    %
    With Qwen2.5-3B model (Fig.~\ref{fig:cost_analysis}~(a)), all algorithms exhibit comparable resource consumption. FedAvg achieves an averaged GPU utilization of 46.91\% and memory utilization of 14.45\%, while FedProx and SCAFFOLD consume 49.74\%/17.65\% and 47.52\%/14.83\%, respectively. FGRPO records 47.94\% GPU utilization and 15.10\% memory utilization, remaining very close to FedAvg and SCAFFOLD and lower than FedProx in both metrics. This indicates that the proposed RPG-based aggregation mechanism introduces negligible additional overhead for the 3B model.

    A similar trend is observed under the Qwen3-4B model (Fig.~\ref{fig:cost_analysis}~(b)), where overall resource utilization increases due to the larger model size. FedAvg and SCAFFOLD show relatively higher averaged GPU utilization of 61.30\% and 62.43\%, with memory utilization of 27.22\% and 27.57\%, respectively. FedProx consumes 57.62\% GPU and 24.82\% memory. In comparison, FGRPO achieves 58.37\% GPU utilization and 25.53\% memory utilization, remaining within the same cost envelope as the baselines and even reducing resource usage compared with FedAvg and SCAFFOLD.

    Under the Qwen2.5-7B model (Fig.~\ref{fig:cost_analysis}~(c)), the resource demand further increases for all methods. FedAvg, FedProx, and SCAFFOLD achieve averaged GPU utilization of 64.38\%, 65.14\%, and 62.94\%, respectively, with corresponding memory utilization of 35.99\%, 36.69\%, and 35.32\%. The full FGRPO consumes 66.21\% GPU and 37.33\% memory. Although FGRPO is slightly higher than the baselines in this setting, the difference remains marginal, suggesting that the adaptive aggregation strategy does not introduce significant additional per-round computational or memory cost.

    Finally, under the Llama-3.2-11B model (Fig.~\ref{fig:cost_analysis}~(d)), all methods reach higher GPU and memory utilization due to the substantially larger model scale. FedAvg, FedProx, and SCAFFOLD consume 79.81\%/51.60\%, 79.34\%/51.76\%, and 78.05\%/52.25\% GPU/memory utilization, respectively. In contrast, FGRPO achieves an averaged GPU utilization of 80.78\%, which is slightly higher than the baselines, while maintaining the lowest memory utilization at 50.72\%. This demonstrates that FGRPO can scale to larger models without introducing additional memory pressure.

    Overall, these results demonstrate that the performance gains of FGRPO are achieved without incurring significant additional computational or memory overhead. Across all model scales, FGRPO remains within the same resource-consumption range as conventional federated baselines, while providing stronger robustness and accuracy under heterogeneous data distributions.
    \begin{figure}[t!]
    \centering
      \parbox{.4\textwidth}{\center\includegraphics[width=.33\textwidth]{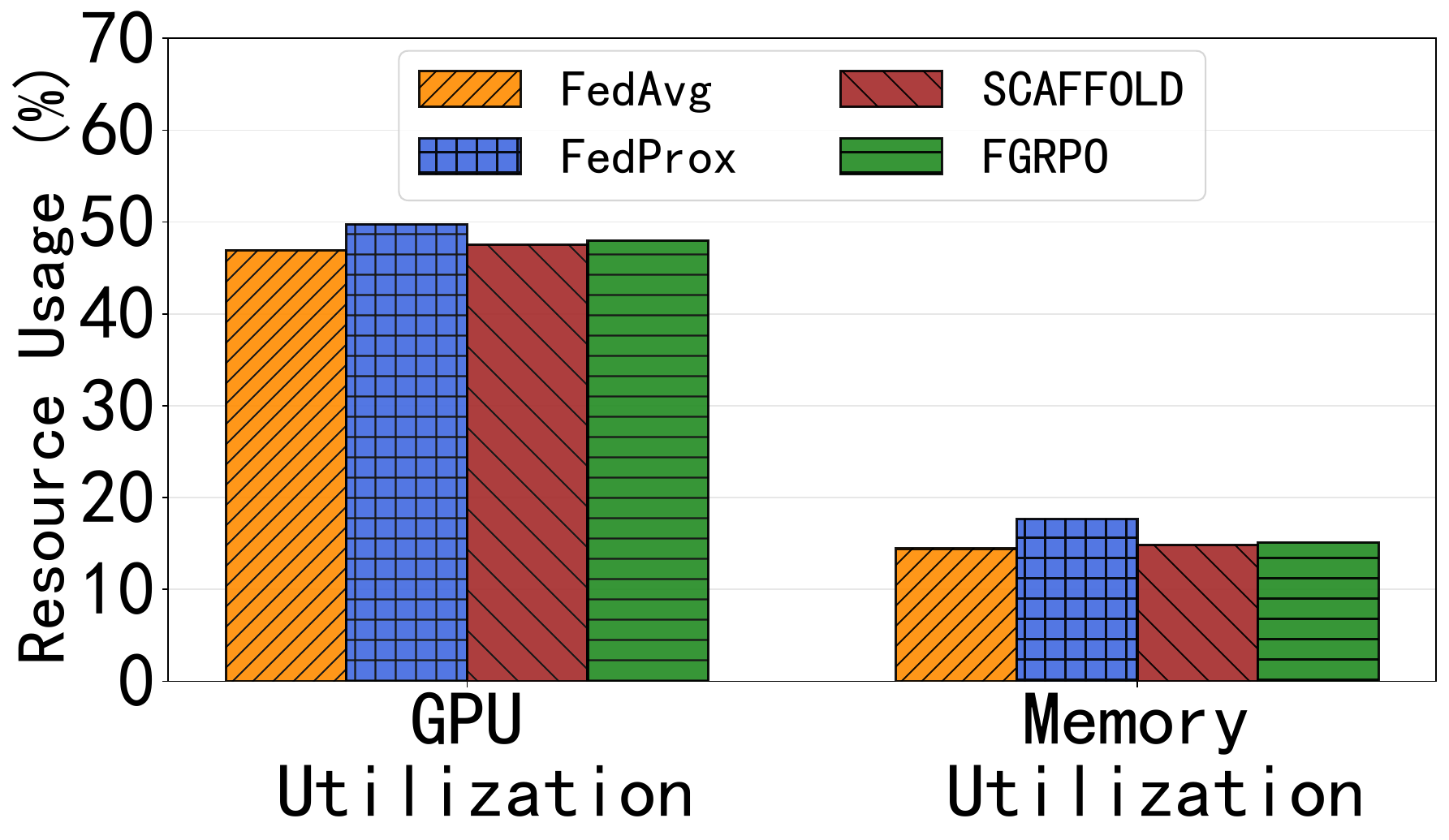}}
      \parbox{.4\textwidth}{\center\includegraphics[width=.33\textwidth]{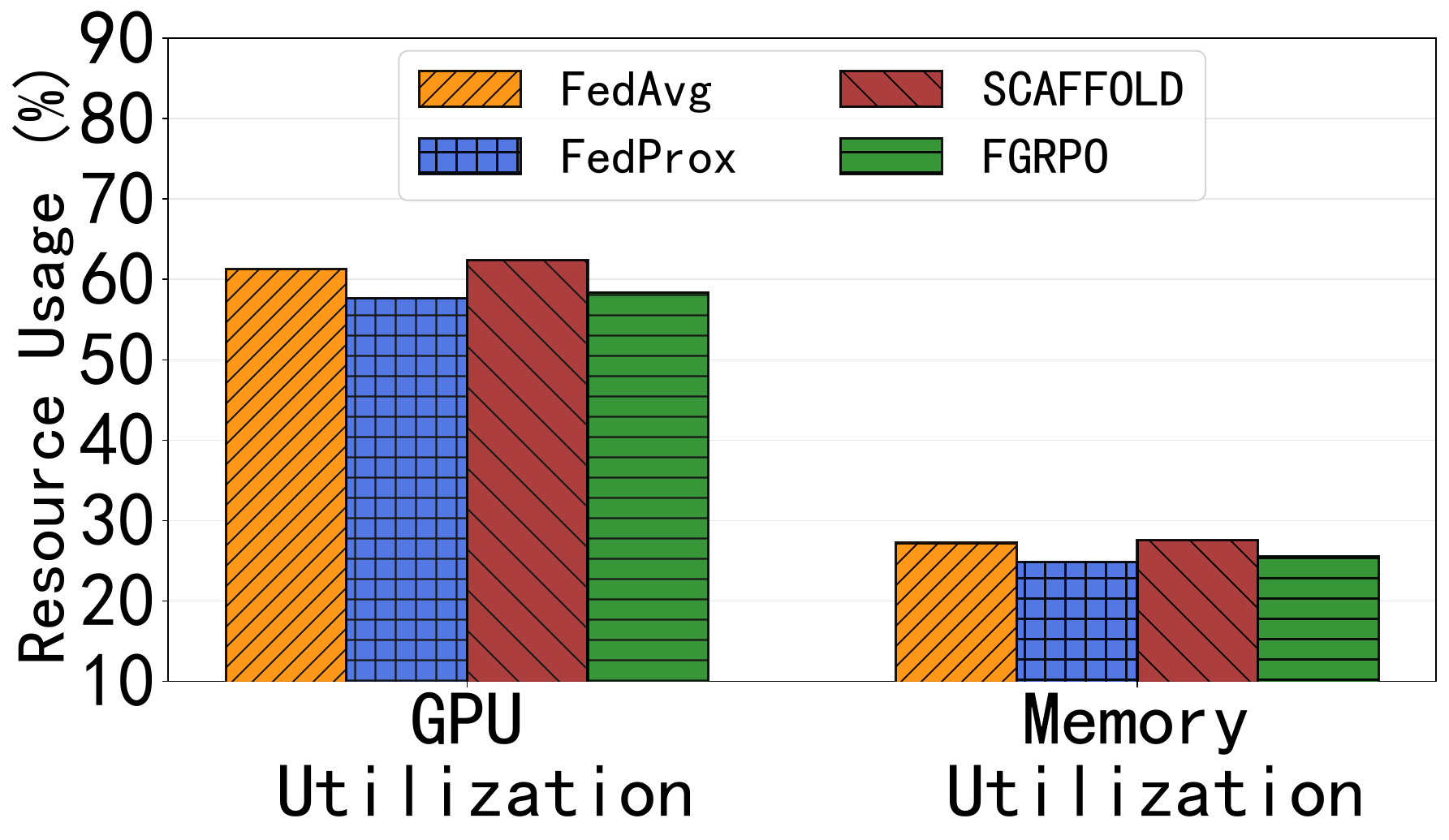}}
      \parbox{.4\textwidth}{\center\scriptsize(a) Qwen2.5-3B}
      \parbox{.4\textwidth}{\center\scriptsize(b) Qwen3-4B}
      \parbox{.4\textwidth}
      {\center\includegraphics[width=.33\textwidth]{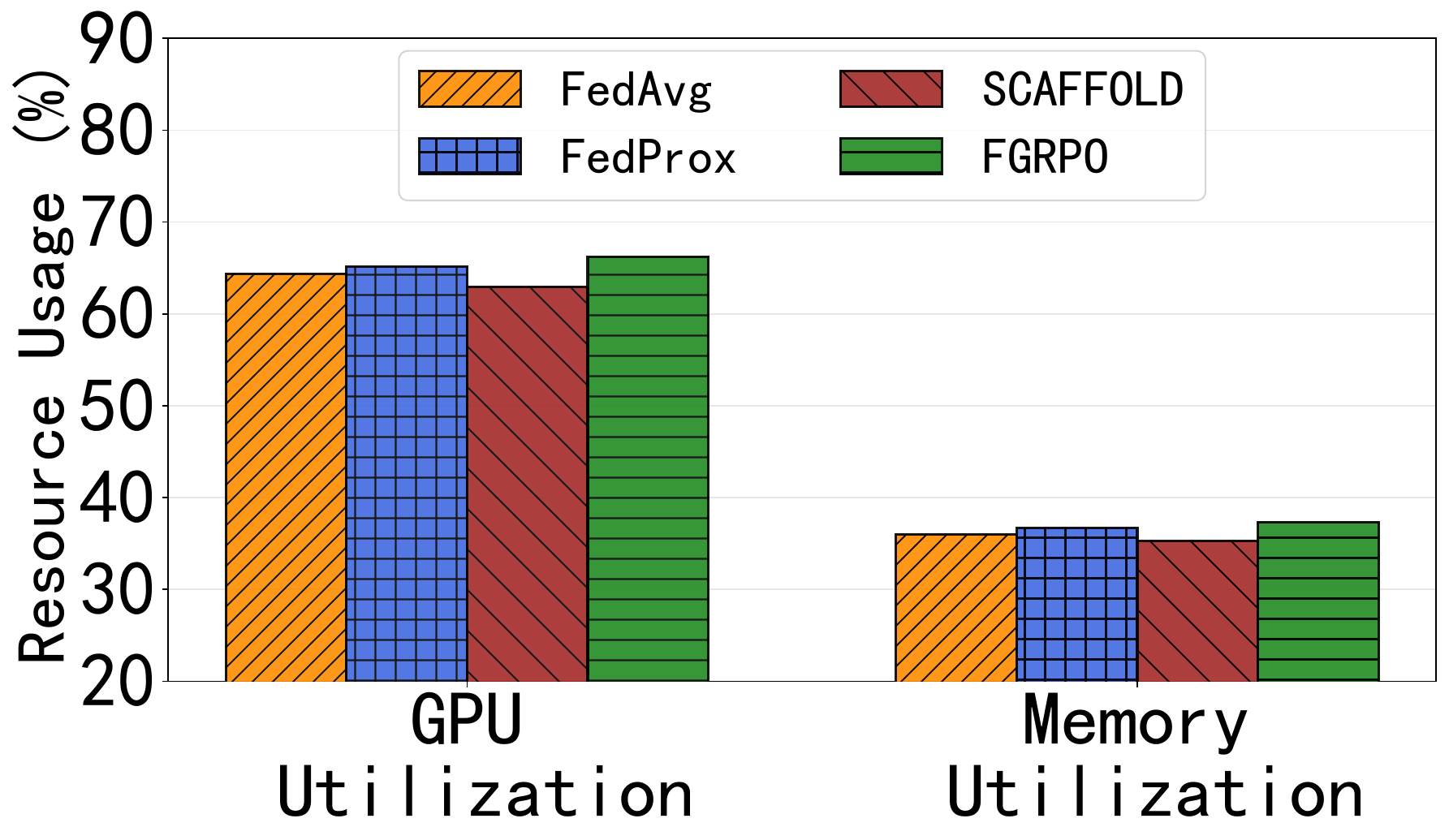}}
      \parbox{.4\textwidth}{\center\includegraphics[width=.33\textwidth]{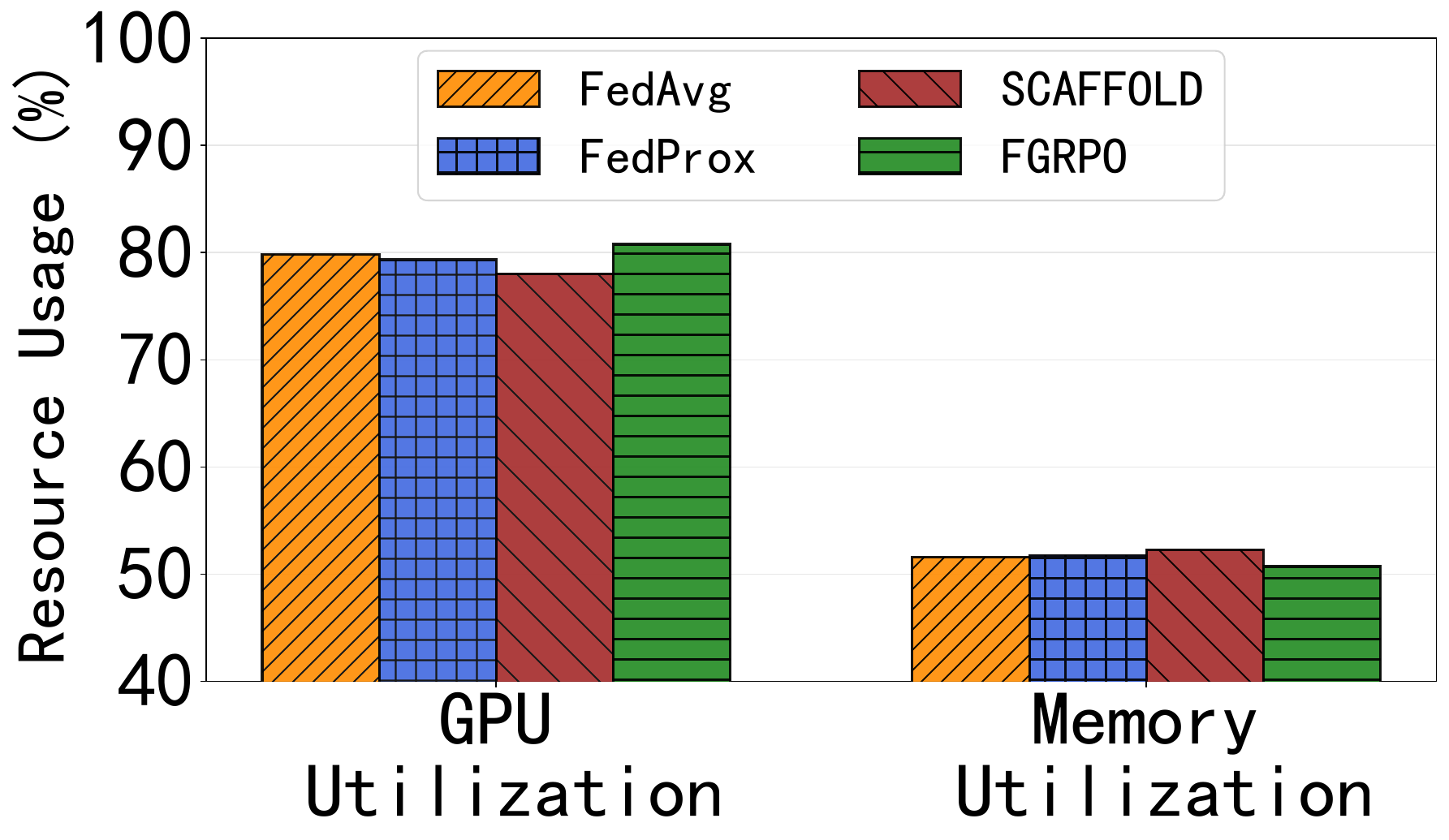}}
      \parbox{.4\textwidth}{\center\scriptsize(c) Qwen2.5-7B}
      \parbox{.4\textwidth}{\center\scriptsize(d) Llama-3.2-11B}
      \caption{Average GPU and memory utilization across different models.}
    \label{fig:cost_analysis}
    \end{figure}

\section{Comprehensive Literature Survey} \label{sec:app_survey}

  \subsection{Federated Reinforcement Learning}
  \label{ssec:appsurvey_fedrl}
    \emph{Federated reinforcement learning} (FedRL) aims to leverage the principles of collaborative learning \cite{McMahanMRHA-AISTATS17} across diverse clients to enhance sample efficiency without compromising raw data or trajectory privacy. Recent theoretical research has focused on establishing rigorous convergence guarantees under the unique constraints of sequential decision-making. \cite{ZhangWMA-ICLR24} provides a fundamental finite-time analysis of on-policy FedRL under data heterogeneity, while \cite{KhodadadianSJM-ICML22} demonstrates that linear speedup is achievable even under the complexities of Markovian sampling. In the offline setting, it is proved in \cite{WooSJC-ICML24} that a ``collaborative single-policy coverage'' condition, where the union of client data covers the optimal policy, is sufficient for global optimality. Furthermore, \cite{XiongWJL-ICLR25} highlights that shared representation learning can further accelerate convergence by extracting collaborative features across diverse tasks, while \cite{YangCWCC-NIPS24} proposes federated natural policy gradient methods to exploit common task.

    Beyond theoretical convergence, recent works address the practical challenges of system reliability and environmental diversity. Robustness against failures and adversaries has been explored by \cite{FanMDJTL-NIPS21} and \cite{FangWG-WWW25}, who provide formal certifications of policy performance under Byzantine disruptions and state perturbations. To handle system-level constraints, \cite{LanHHAB-ICLR25} introduces an asynchronous framework that ensures robustness to varying computational speeds, while \cite{JordanGFW-AAMAS24} proposes a fully decentralized policy gradient algorithm for peer-to-peer topologies. Managing heterogeneous dynamics and architectures also remains a central theme; while \cite{WangHZMA-ICML24} utilizes momentum-based aggregation to stabilize updates across diverse environments, \cite{JiangWZBTF-AAMAS25} introduces FedHPD, utilizing policy distillation to enable collaborative learning between clients with different model structures. Finally, \cite{WooJC-ICML23} proves that diverse local distributions can actually reduce global coverage requirements, while \cite{RengarajanRKS-NIPS24} utilizes ensemble-directed models to quantify uncertainty in offline settings.
    %

  \subsection{GRPO} \label{ssec:appsurvey_grpo}
    \emph{Group relative policy optimization} (GRPO) \cite{ShaoWZXSBZZLW-arXiv24} has established a new paradigm for fine-tuning reasoning models by replacing the traditional critic network with group-level reward normalization \cite{GuoYZSWZXZMB-Nature25}. While this approach stabilizes updates by computing relative advantages within response sets, its practical application is often hindered by high computational costs and a lack of sensitivity to prompt difficulty. To mitigate these efficiency bottlenecks, CPPO \cite{LinLXJ-arXiv23} introduces completion pruning to discard low-advantage trajectories, while GSPO \cite{ZhengLLCYGDLMY-arXiv25} shifts from token-level to sequence-level importance ratios to stabilize training in Mixture-of-Experts (MoE) architectures. Furthermore, SEED-GRPO \cite{ChenCWY-arXiv25} integrates semantic entropy into the optimization process, enabling the model to explicitly differentiate between certain and uncertain knowledge boundaries instead of treating all prompts as equally informative.

    Building upon these architectural efficiencies, recent research has focused on enhancing the precision of the learning signal through advanced reward shaping and sampling strategies. GRPO-LEAD \cite{ZhangZ-arXiv25} addresses the issues of verbosity and sparsity by integrating length-regularized rewards and difficulty-aware advantage reweighting, which ensures robust generalization on challenging problems. Similarly, DAPO \cite{YuZZYZYDFLL-arXiv25} introduces a decoupled and dynamic sampling system designed to stabilize long Chain-of-Thought (CoT) reasoning. By employing asymmetric clipping to prevent entropy collapse and overlong reward shaping, these methods collectively evolve GRPO into a more resilient framework capable of handling complex, multi-step logical synthesis.

    While the aforementioned studies focus on algorithmic refinements such as reward shaping or sampling efficiency, they implicitly assume a centralized data architecture where global statistics are readily available. In contrast, our proposed FGRPO framework breaks this assumption by enabling collaborative reasoning training across distributed private datasets.

\end{document}